\documentclass[10pt,journal,compsoc]{IEEEtran}
\usepackage{cite}
\usepackage{amsmath,amssymb,amsfonts,amsthm}
\usepackage{algorithm}
\usepackage{algorithmic}
\usepackage{graphicx}
\usepackage{textcomp}
\usepackage{enumitem}
\usepackage{nicefrac}
\usepackage{multirow}
\usepackage{tabularx}
\usepackage{soul}
\usepackage{xcolor}
\usepackage{subfigure}
\usepackage{url}
\usepackage{bm}
\usepackage{booktabs}
\usepackage[normalem]{ulem}

\newcolumntype{C}[1]{>{\centering\arraybackslash}p{#1}}

\def\BibTeX{{\rm B\kern-.05em{\sc i\kern-.025em b}\kern-.08em
		T\kern-.1667em\lower.7ex\hbox{E}\kern-.125emX}}

\newtheorem{prop}{Proposition}
\newtheorem{definition}{Definition}

\newtheorem{lemma}{Lemma}

\newcommand{\diag}[1]{ \text{diag}\!\left(#1\right)}
\newcommand{\R}{\mathbb{R}}

\newcommand{\parabegin}[1]{\vspace{3px}\noindent\textbf{#1}}

\title{Bilinear Scoring Function Search \\for Knowledge Graph Learning}

\author{Yongqi Zhang,~\IEEEmembership{Member,~IEEE}
	Quanming Yao,~\IEEEmembership{Member,~IEEE}
	James T. Kwok,~\IEEEmembership{Fellow,~IEEE}
\IEEEcompsocitemizethanks{\IEEEcompsocthanksitem Y. Zhang was with 4Paradigm Inc. Beijing, China.
	E-mail: zhangyongqi@4paradigm.com
	\IEEEcompsocthanksitem Q. Yao was with Department of Electronic Engineering,
	Tsinghua University and 4Paradigm Inc. Beijing, China.
	E-mail: qyaoaa@connect.ust.hk
	\IEEEcompsocthanksitem 
	J.T. Kwok was with Department of Computer Science,
	Hong Kong University of Science and Technology.
	Hong Kong, China.
	E-mail: jamesk@cse.ust.hk
}% <-this % stops an unwanted space
\thanks{The code is public available at https://github.com/AutoML-Research/AutoSF. 
	Corresponding author: Quanming Yao.}
}% <-this % stops a space

\begin{document}	
	
\IEEEtitleabstractindextext{
\begin{abstract}
Learning embeddings
for entities and relations in knowledge graph (KG)
have benefited many downstream tasks.
In recent years,
scoring functions, 
the crux of KG learning,
have been human designed 
to 
measure the plausibility of triples
and
capture different kinds of relations in KGs.
However, 
as relations exhibit intricate patterns that are hard to infer before training,
none of them consistently perform the best on benchmark tasks. 
In this paper,
inspired by the recent success of automated machine learning (AutoML), 
we search bilinear scoring functions
for different KG tasks 
through the AutoML techniques.
However,
it is non-trivial to explore domain-specific information here.
We first set up a search space for AutoBLM
by analyzing existing scoring functions.
Then, we propose a progressive algorithm (AutoBLM) 
and an evolutionary algorithm (AutoBLM+),
which are further accelerated by filter and predictor 
to deal with the domain-specific properties for KG learning.
Finally,
we perform extensive experiments on benchmarks in 
KG completion, multi-hop query, and entity classification tasks.
Empirical results
show that the searched scoring functions
are KG dependent, new to the literature, 
and outperform the existing scoring functions.
AutoBLM+ is better than  AutoBLM 
as the evolutionary algorithm can
 flexibly explore better structures in the same budget.
\end{abstract}
	
\begin{IEEEkeywords}
Automated machine learning,
Knowledge graph,
Neural architecture search,
Graph embedding
\end{IEEEkeywords}
}

\maketitle

\section{Introduction}
\label{sec:intro}

\IEEEPARstart{T}{he} knowledge
 graph (KG)
\cite{singhal2012introducing,nickel2016review,wang2017knowledge} is a 
graph in which the nodes represent entities, 
the edges are the relations between
entities, and
the facts
are represented by triples of the form \textit{(head entity, relation, tail entity)} (or $(h,r,t)$ in short).
The KG has been found useful in a lot of 
data mining and machine learning
applications and tasks,
including question answering~\cite{ren2019query2box},
product
recommendation~\cite{zhang2016collaborative},
knowledge graph completion~\cite{nickel2011three,yang2014embedding},
multi-hop query~\cite{hamilton2018embedding,ren2019query2box},
and entity classification~\cite{kipf2016semi}.

In a KG, 
plausibility  of a fact
$(h,r,t)$ 
is given by $f(h,r,t)$, where $f$ is the
{\em scoring function\/}. 
Existing $f$'s are
custom-designed by human experts,
and can be categorized into the following three families:
(i) translational distance models (TDMs)
\cite{bordes2013translating,wang2014knowledge,fan2014transition,sun2019rotate},
which 
model the relation embeddings as translations from the head entity embedding
to the tail entity embedding;
(ii)
bilinear model (BLMs)
\cite{nickel2011three,yang2014embedding,trouillon2017knowledge,nickel2016holographic,liu2017analogical,kazemi2018simple,lacroix2018canonical,zhang2019quaternion},
which
model the interaction between entities and relations
by a bilinear product between the entity and relation embeddings; and (iii)
neural network models (NNMs) 
\cite{dong2014knowledge,
dettmers2017convolutional,guo2019learning,schlichtkrull2018modeling,vashishth2019composition},
which
use neural networks to capture the interaction.
The scoring function
can significantly impact KG learning performance \cite{nickel2016review,wang2017knowledge,lin2018knowledge}.
Most TDMs are less expressive and have poor empirical performance~\cite{wang2017knowledge,wang2018evaluating}.
NNMs are powerful but have
large numbers of parameters and may overfit the training triples.
In comparison, 
BLMs  
are more advantageous in that they are 
easily customized to be expressive,
have linear complexities w.r.t. the numbers of entities/relations/dimensions,
and have state-of-the-art performance~\cite{lacroix2018canonical}.
While a number of BLMs have been proposed,
the best BLM is often dataset-specific.

Recently, 
automated machine learning (AutoML)~\cite{quanming2018auto,automl_book}
has demonstrated 
its power in many machine learning tasks such as 
hyperparameter optimization (HPO)~\cite{feurer2015efficient}
and neural architecture search
(NAS)~\cite{zoph2017neural,liu2018darts,elsken2019neural}.
The models discovered have better performance than those
designed by
humans, and 
the amount of human effort required is
significantly reduced.
Inspired by its success,
we propose 
in this paper 
the use of AutoML for the design  of
KG-dependent scoring functions.
To achieve this, one has to 
pay careful consideration to the
three main components in 
an AutoML algorithm:
(i) search space, 
which identifies important properties of the learning models
to search;
(ii) search algorithm, which ensures that
finding a good model in this space
is efficient;
and 
(iii) evaluation method,
which offers feedbacks to the search algorithm.

In this paper,
we make the following contributions in achieving these goals:
\begin{itemize}[leftmargin = 10px,itemsep=1px]
\item 
We design a search space
of scoring functions, which includes all the
existing BLMs.
We further analyze properties of this search space,
and provide conditions for a candidate scoring function
to be expressive,
degenerate, and equivalent to another.

\item 
To explore the above search space properties 
and reduce the computation cost in evaluation,
we design 
a filter to remove degenerate and equivalent structures,
and a performance predictor with specifically-designed symmetry-related features (SRF) 
to select promising structures.

\item 
We customize a progressive algorithm (AutoBLM)
and an evolutionary algorithm (AutoBLM+)
that,
together with the filter and performance predictor,
allow flexible exploration of new BLMs.
\end{itemize}

Extensive experiments are 
performed on the 
tasks
of KG completion, multi-hop query and entity classification.
The results
demonstrate that the models obtained by AutoBLM and
AutoBLM+
outperform the start-of-the-art 
human-designed 
scoring functions.
In addition,
we show that the customized progressive and evolutionary algorithms are 
much less expensive than popular search algorithms
(random search, Bayesian optimization
and reinforcement learning)
in finding a good scoring function.

\vspace{3px}
\noindent
\textit{Differences with the Conference Version.}
Compared to 
the preliminary version published in ICDE 2020~\cite{zhang2020autosf},
we have made the following important extensions:
\begin{enumerate}[leftmargin=*]
\item \textit{Theory.}
We add new analysis to the designed 
search space based on bilinear models.
We theoretically prove
when the candidates in the search space can be expressive (Section~\ref{ssec:unifiedBLM}),
degenerate (Section~\ref{sssec:degenerate})
and equivalent structures (Section~\ref{sssec:equiv}).

\item 
\textit{Algorithm.}
We extend the search algorithm with the evolutionary algorithm (Section~\ref{ssec:evolution}),
i.e., AutoBLM+.
The evolutionary strategy in Algorithm~\ref{alg:evolution} 
can explore better in the search space,
and can also leverage the filter 
and predictor to deal with the domain-specific properties.

\item
\textit{Tasks.}
We extend AutoBLM and AutoBLM+ to two new tasks,
namely, multi-hop query (Section~\ref{sec:exp:hop})
and entity classification in (Section~\ref{sec:exp:cls}).
We show that the search problem can be well adopted to these new scenarios,
and achieve good empirical performance.

\item \textit{Ablation Study.}
We conduct more experiments on the performance 
(Section~\ref{exp:kgc:performance} and \ref{exp:kgc:ogb}) and analysis (Section~\ref{exp:alg:compare})
of the new search algorithm,
analysis on the influence of  $K$ 
(Section~\ref{sec:exp:varyK}),
and the problem of parameter sharing (Section~\ref{sec:exp:ps})
to analyze the design schemes in the search space and search algorithm.

\end{enumerate}

\parabegin{Notations.}
In this paper, vectors are denoted by lowercase boldface, and matrix by uppercase boldface.
The important notations are listed in Table~\ref{tab:notations}.

\begin{table}[ht]
	\centering
	\vspace{-7px}
	\caption{Notations used in the paper.}
	\label{tab:notations}
	\vspace{-10px}
	\renewcommand{\arraystretch}{1.1}
	\begin{tabular}{c|C{6.3cm}}
		\toprule
		%		Notations & Description \\ \midrule
		$\mathcal E, \mathcal R, \mathcal S$   &   set of entities, relations, triples \\
		$|\mathcal E|, |\mathcal R|, |\mathcal S|$ & number of entities, relations, triples \\ 
		$(h,r,t)$ & triple of head entity, relation and tail entity \\
		$\bm h, \bm r, \bm t$ & embeddings of $h$, $r$, and $t$ \\ 
		$f(h,r,t)$ & scoring function for triple $(h,r,t)$ \\ 
		$\mathbb R^d, \mathbb C^d, \mathbb H^d$ & $\!\!d$-dimensional real/complex/hypercomplex space \\
		$\bm R_{(\bm r)}\!\in\!\mathbb R^{d\times d}$ & square matrix based on relation embedding $\bm r$ \\
		$\langle\bm a, \bm b, \bm c\rangle$ & triple product $:=\sum_{i=1}^d a_i b_i c_i=\bm a^\top\diag{\bm b}\bm c$ \\
		$\|\bm v\|_1$ & $\ell_1$-norm of vector $\bm v$ \\
		$\text{Re}(\bm v)$ & real part of complex vector $\bm v\in\mathbb C^d$ \\
		$\bar{\bm v}$ & conjugate of complex vector $\bm v\in\mathbb C^d$  \\
		\bottomrule
	\end{tabular}
	\vspace{-10px}
\end{table}

\section{Background and Related Works}
\label{sec:relworks}

\subsection{Scoring Functions for Knowledge Graph (KG)}
\label{ssec:kg}

A knowledge graph 
(KG)
can be represented by a third-order tensor $\mathbf G\in \mathbb
R^{|\mathcal E|\times |\mathcal R| \times |\mathcal E|}$,
in which 
$G_{hrt}=1$ if
the corresponding triple $(h,r,t)$ 
exists
in the KG, and 0
otherwise.
The {\em scoring function} $f(h,r,t): \mathcal E\times \mathcal R \times \mathcal
E \rightarrow \mathbb R$ measures plausibility of the triple $(h,r,t)$.
%where $h, t$ are entities with embeddings $\bm h$ and $\bm t$, respectively, and $r$ is a relation with embedding $\bm r$.
As introduced in Section~\ref{sec:intro},
it is
desirable 
for a scoring function 
to be 
able to represent  any of the
symmetric, anti-symmetric, general asymmetric and inverse
KG relations  
in Table~\ref{tab:relations}.

\begin{table}[t]
	\centering
	\caption{Popular properties in KG relations.}
	\label{tab:relations}
	\setlength\tabcolsep{1.5pt}
	\renewcommand{\arraystretch}{1.08}
	\vspace{-10px}
	\begin{tabular}{c|c|c} \toprule
		property & examples in WN18/FB15k & constraint on $f$\\ \midrule
		symmetry
		&  \texttt{isSimilarTo}, \texttt{spouseOf} 
		&   $f(t, r, h) \! = \! f(h,r,t)$  
		\\  %\midrule 
		anti-symmetry
		&  \texttt{ancestorOf}, \texttt{isPartOf} 
		&  $f(t, r, h) \! = \! -f(h,r,t)$  
		\\%\midrule
		general asymmetry
		& \texttt{locatedIn}, \texttt{profession} 
		&  $f(t, r, h) \! \neq \! f(h,r,t)$   
		\\  %\midrule
		inverse 
		&  \texttt{hypernym}, \texttt{hyponym}  
		& $f(t, r, h) \! = \! f(h, r', t)$		
		\\ 
		\bottomrule
	\end{tabular} 
	\vspace{-12px}
\end{table}

\begin{definition}[Expressiveness~\cite{trouillon2017knowledge,wang2017multi,balavzevic2019tucker}] 
	\label{def:express}
	A scoring function 
   is {\em fully expressive\/} if
for any KG $G$ and the corresponding tensor $\mathbf G\in\mathbb R^{|\mathcal
E|\times |\mathcal R|\times |\mathcal E|}$, one can find an instantiation 
	$f$
of 
	the scoring function  such that
	$f(h,r,t)=G_{hrt}$, $\forall h,t\in\mathcal E, r\in \mathcal R$.
	%i.e., it can learn to fit the KG.
\end{definition}

Not all scoring functions are fully expressive.
For example, consider a KG with two people \textit{A}, \textit{B}, and 
a relation ``OlderThan''. Obviously, we can have
either \textit{(A, OlderThan, B)} or \textit{(B, OlderThan, A)}, but not both.
The scoring function $f(h,r,t) = \langle \bm h, \bm r, \bm t\rangle =\sum_{i=1}^dh_i r_i t_i$, where
$\bm h, \bm r, \bm t$ are $d$-dimensional embeddings of $h, r$ and $t$,
respectively, cannot be fully expressive
since
$f(h,r,t)=f(t,r,h)$.

On the other hand,
while expressiveness  indicates
the ability of $f$ to fit a given KG,
it may not 
generalize well when inference on different KGs.
As real-world KGs can be very sparse~\cite{singhal2012introducing,wang2017knowledge},
a scoring function with a large amount of trainable parameters may overfit the training triples.
Hence, it is also desirable that the scoring function has only a manageable number of parameters.

In the following, we review the three main types of scoring functions, namely,
translational distance model (TDM),
neural network model (NNM), and
biLinear model (BLM).
As will be seen,
many TDMs (such as TransE~\cite{bordes2013translating} and TransH~\cite{wang2014knowledge}) cannot model the symmetric relations well~\cite{wang2017knowledge,ji2020survey}.
Neural network models,
though fully expressive,
have large numbers of parameters.
This not only prevents the model from generalizing well on unobserved triples in a sparse KG,
but also increases the training and inference costs~\cite{dettmers2017convolutional,lacroix2018canonical,vashishth2019composition}.
In comparison,
BLMs 
(except DistMult)
can model all relation pattens in Table~\ref{tab:relations} and are fully expressive.
Besides,
these models 
(except RESCAL and TuckER)
have moderate complexities
(with the number of parameters 
linear in $|\mathcal E|, |\mathcal R|$ and $d$).
Therefore,
we consider BLM as a better choice, and it will be our focus in this paper.

\begin{table*}[ht]
	\centering
	\caption{The representative BLM scoring functions. 
		For each scoring function we show the definitions, 
		expressiveness in Definition~\ref{def:express},
		the ability to model all common relation patterns in Table~\ref{tab:relations} (``RP'' for short),
		and the number of parameters.}
	\vspace{-10px}
	\label{tab:scofun}
	\setlength\tabcolsep{8pt}
	\renewcommand{\arraystretch}{1.3}
	\begin{tabular}{c |  c | c | c | c}
		\toprule
		scoring function & definition & $\!$expressiveness$\!$    &  RP  & \# parameters                                                        \\ \midrule
		RESCAL~\cite{nickel2011three}                                                                               & $\bm h^\top \bm R_{(\bm r)}\bm t$    & $\surd$    &  $\surd$  & $O\left(|\mathcal E|d + |\mathcal R|d^2\right)$   \\
		              DistMult~\cite{yang2014embedding}      & $\langle \bm h, \bm r, \bm t\rangle$     & $\times $  & $\times$ & $O(|\mathcal E|d+|\mathcal R|d)$              \\
%		\midrule
		ComplEx~\cite{trouillon2017knowledge}/HolE
		\cite{nickel2016holographic} & Re$\left({\bm{h}} \otimes {\bm{r}} \otimes\bar{{\bm{t}}} \right)$         & $\surd$    & $\surd$ & $O\left(|\mathcal E|d+|\mathcal R|d\right)$    \\
		Analogy~\cite{liu2017analogical}  & $ \langle \hat{\bm{h}}, \hat{\bm{r}},	\hat{\bm{t}}\rangle$ + Re$\left(\underline {\bm{h}} \otimes \underline {\bm{r}} \otimes\bar{\underline {\bm{t}}} \right)$         & $\surd$    & $\surd$ & $O\left(|\mathcal E|d+|\mathcal R|d\right)$    \\
		{SimplE~\cite{kazemi2018simple}/CP~\cite{lacroix2018canonical}}                                           & $\langle \hat{\bm{h}}, \hat{\bm{r}}, \underline {\bm{t}} \rangle$ + $\langle \underline {\bm{h}}, \underline {\bm{r}}, \hat{\bm{t}} \rangle$                 & $\surd$  & $\surd$   & $O\left(|\mathcal E|d+|\mathcal R|d\right)$                           \\
		 QuatE~\cite{zhang2019quaternion}                                                                            & $\bm h\odot\bm r\odot\bm t$                                                                                                                                                        & $\surd$   & $\surd$  & $O\left(|\mathcal E|d+|\mathcal R|d\right)$                           \\
		TuckER~\cite{balavzevic2019tucker}                                                                          & $\mathcal W\times_1 \bm h\times_2\bm r\times_3\bm t$                                                                                                                               & $\surd$  & $\surd$   & $O\left(|\mathcal E|d+|\mathcal R|d+d^3\right)$                        \\ \bottomrule
	\end{tabular} 
	\vspace{-7px}
\end{table*}

%\footnote{*** seems that u dont need the symbols 
%$\bm E$ and 
%$\bm R$ in the whole paper. if so, remove this para}
%Let $\bm E\in\mathbb R^{d\times{|\mathcal E|}}$ be the matrix  containing the $d$-dimensional entity embeddings (with one column $\bm e$ for each entity $e$), and $\bm R\in\mathbb R^{d\times{|\mathcal R|}}$ be the matrix  containing the relation embeddings (with one column $\bm r$ for each relation $r$).

\parabegin{Translational Distance Model (TDM).}
Inspired by analogy results in word embeddings~\cite{bengio2013representation}, 
scoring functions in TDM 
take the relation $r$ as a translation from $h$ to $t$.
The most representative TDM is TransE~\cite{bordes2013translating}, with $f(h,r,t)=-\|\bm t - (\bm h+\bm r)\|_1$.
In order to handle one-to-many, many-to-one and many-to-many relations,
TransH~\cite{wang2014knowledge} and TransR~\cite{fan2014transition} introduce additional vectors/matrices to map
the entities to a relation-specific hyperplane.
The more recent
RotatE~\cite{sun2019rotate} treats the relations as rotations in a
complex-valued space:
$f(h,r,t) = -\|\bm t-\bm h\otimes\bm r\|_1$,
where $\bm h, \bm r, \bm t \in \mathbb C^d$
and $\otimes$ is the Hermitian product~\cite{trouillon2017knowledge}.
As discussed in~\cite{wang2017multi},
most TDMs are not fully expressive. For example,
TransE and TransH cannot model symmetric relations.

\parabegin{Neural Network Model (NNM)}.
NNMs take the 
entity and relation 
embeddings 
as input, and
output a probability for the triple $(h,r,t)$ using a neural network.
Earlier works are
based on multilayer perceptrons~\cite{dong2014knowledge} 
and neural tensor networks~\cite{socher2013reasoning}.
More recently,
ConvE~\cite{dettmers2017convolutional}
uses the convolutional network to 
capture interactions among embedding dimensions.
By sampling relational paths~\cite{guu2015traversing} from the KG,
RSN 
\cite{guo2019learning}
and Interstellar \cite{zhang2020interstellar}
use the recurrent network~\cite{hochreiter1997long} to recurrently combine the head entity and
relation 
with a step-wise scoring function.
As the KG is a graph,
R-GCN~\cite{schlichtkrull2018modeling} 
and CompGCN~\cite{vashishth2019composition} use the graph convolution network
\cite{kipf2016semi}
to aggregate entity-relation compositions layer by layer.
Representations 
at the final layer 
are then used to compute the scores.
Because of the use of an additional neural network,
NNM requires more parameters and has larger model complexity.

\parabegin{BiLinear Model (BLM).}
BLMs model the KG relation as a bilinear product between entity embeddings.
For example, RESCAL~\cite{nickel2011three} 
defines $f$ as:
$f(h,r,t) = \bm h^\top \bm R_{(\bm r)}\bm t$,
where $\bm h, \bm t\in\mathbb R^d$, and 
$\bm R_{(\bm r)}\in\mathbb R^{d\times d}$.
To avoid overfitting, DistMult~\cite{yang2014embedding} requires
$\bm{R}_{(\bm r)}$ to be 
diagonal, and
$f(h,r,t)$ reduces to a triple product:
$f(h,r,t) 
= \bm h^\top \diag{\bm r}\bm t 
= \left\langle \bm h, \bm r, \bm t\right\rangle$.
However, it can only model symmetric relations.
To capture anti-symmetric relations,
ComplEx~\cite{trouillon2017knowledge} uses complex-valued embeddings $\bm{h}, \bm r, \bm{t}\in\mathbb C^d$ 
with 
$f(h,r,t) 
=  \text{Re}\big(\bm h^\top\diag{\bm r}\bar{\bm t} \big)
=  \text{Re}\left(\bm h\otimes \bm r\otimes\bar{\bm t}\right)$,
where 
$\otimes$ is the Hermitian product in complex space~\cite{trouillon2017knowledge}.
HolE~\cite{nickel2016holographic} uses the circular correlation instead of the dot product,
but is shown to be equivalent to ComplEx~\cite{hayashi2017equivalence}.

Analogy~\cite{liu2017analogical} decomposes the head embedding $\bm h$ into a real part $ \hat{\bm{h}}\in\mathbb R^{\hat{d}}$
and a complex part
$\underline {\bm{h}}\in\mathbb C^{\underline {d}}$.
Relation embedding $\bm r$ (resp. tail embedding $\bm t$) is similarly decomposed into 
a real part $ \hat{\bm{r}}$ (resp. $\hat{\bm{t}}$)
and a complex part
$\underline {\bm{r}}$ (resp. $\underline {\bm{t}}$). $f$ is then written
as:
$f(h,r,t)
=\langle \hat{\bm h}, \hat{\bm r}, \hat{\bm t}\rangle 
+ \text{Re}\big(\underline {\bm h}\otimes \underline {\bm r}\otimes\bar{\underline {\bm t}}\big)$,
which can be regarded as a combination of DistMult and ComplEx.
To simultaneously model the forward triplet $(h,r,t)$ and its inverse $(t,r',h)$,
SimplE~\cite{kazemi2018simple} / CP~\cite{lacroix2018canonical} 
similarly
splits the embeddings to a forward part 
($\hat{\bm{h}},\hat{\bm{r}},\hat{\bm{t}}\in\mathbb R^{d}$)
and a backward part ($\underline {\bm{h}},\underline {\bm{r}},\underline
{\bm{t}}\in\mathbb R^{d}$):
$f(h,r,t)=\langle \hat{\bm{h}}, \hat{\bm{r}}, \underline {\bm{t}} \rangle+\langle \underline {\bm{t}}, \underline {\bm{r}}, \hat{\bm{h}} \rangle$.
To allow more interactions among embedding dimensions,
the recent QuatE~\cite{zhang2019quaternion} 
uses embeddings in the hypercomplex space ($\bm h,\bm r, \bm t\in\mathbb H^d$) to 
model
$f(h,r,t)=\bm h\odot \bm r \odot \bm t$
where
$\odot$ 
is the Hamilton product.
By using the Tucker decomposition~\cite{tucker1966some},
TuckER~\cite{balavzevic2019tucker} proposes a generalized bilinear model 
and introduces more parameters in the core tensor $\mathcal W\in\mathbb R^{d\times d \times d}$:
$f(h,r,t) =
\mathcal W\times_1 \bm h \times_2 \bm r\times_3 \bm t$,
where
$\times_i$ is the tensor product along the $i$th mode.
A summary of 
these BLMs is  
in
Table~\ref{tab:scofun}.

\subsection{Common Learning Tasks in KG}

\subsubsection{KG Completion}
\label{sec:app1}

KG
is naturally incomplete~\cite{singhal2012introducing}, and
KG completion is a representative task in KG learning \cite{wang2017knowledge,nickel2011three,bordes2013translating,kazemi2018simple,trouillon2017knowledge,yang2014embedding,dettmers2017convolutional}.  
Scores on the observed triples are maximized, while those on the non-observed  triplets are minimized.
After training,
new triples 
can be added
to the KG
by entity prediction with either a missing head
entity
$(?, r, t)$
or a missing tail entity
$(h, r, ?)$~\cite{wang2017knowledge}.
For each kind of query,
we enumerate all the entities $e\in\mathcal E$
and compute the corresponding scores $f(e, r, t)$ or
$f(h, r, e)$.
Entities with larger scores are more likely to be true facts.
Most of the models in Section~\ref{ssec:kg} can be directly used for KG completion.

\subsubsection{Multi-hop Query}
\label{sec:app2}

In KG completion, we predict queries $(h, r, ?)$ with length one, i.e., $1$-hop query.  In practice, there can be multi-hop queries with lengths larger than one~\cite{guu2015traversing,hamilton2018embedding,wang2017knowledge}.  For example, one may want to predict ``\textit{who is the sister of Tony's mother}".  
To solve this problem, we need to solve the length-2 query problem $(?, sister \circ mother, Tony)$ with the relation composition operator $\circ$.

Given the KG $\mathcal G=\{\mathcal E, \mathcal R, \mathcal S\}$,
let 
$\varphi_r$,
corresponding to the relation $r\in\mathcal R$,
be a binary function
$\mathcal E\times \mathcal E\mapsto\{True, False \}$.
The multi-hop query 
is defined
as follows.
\begin{definition}[Multi-hop query~\cite{hamilton2018embedding,ren2019query2box}]
	The multi-hop query $(e_0, r_1\circ r_2 \circ \dots \circ r_L, e_? )$ with length $L > 1$
	is defined as
	$
	\exists e_1 \dots e_{L-1}, e_?: \varphi_{\!r_1}(e_0, e_1)
	\! \wedge \! 
	\varphi_{\!r_2}(e_1, e_2) 
	\! \wedge \! 
	\dots 
	\! \wedge \! 
	\varphi_{\!r_L}(e_{L-1}, e_?)
	$
	where $\wedge$ is the conjunction operation, $e_0$ is the starting entity, $e_?$ is the entity to predict,
	and $e_1 \dots e_{L-1}$ are intermediate entities that connect the conjunctions.
\end{definition}

Similar to KG completion,
plausibility of a query $(e_0, r_1\circ r_2 \circ \dots \circ r_L,
e_L)$ 
is measured
by
a scoring function 
\cite{guu2015traversing,hamilton2018embedding}:
\begin{equation}
f(e_0, r_1\circ r_2 \circ \dots \circ r_L, e_L)  = \bm e_0^\top \bm R_{(r_1)} \cdots \bm R_{(r_L)} \bm e_L,
\label{eq:multihop}
\end{equation}
where $\bm R_{(r_i)}$ is a relation-specific matrix of the $i$th relation.
The key is on how to model the composition of relations in the embedding space.
Based on TransE~\cite{bordes2013translating},
TransE-Comp~\cite{guu2015traversing}
models the composition operator as addition, and 
defines
the scoring function 
as
$f(e_0, r_1\circ \cdots \circ r_L, e_L) = -\|\bm e_L - (\bm e_0 + \bm r_1 +
\cdots + \bm r_L)\|_1$.
Diag-Comp~\cite{guu2015traversing}
uses the multiplication operator in DistMult~\cite{yang2014embedding} to define
$f(e_0, r_1\circ \cdots \circ r_L, e_L) = \bm e_0^\top\bm D_{\bm r_1}\cdots\bm
D_{\bm r_L}\bm e_L$,
where $\bm D_{\bm r_i}=\diag{\bm r_i}$.
Following RESCAL~\cite{nickel2011three},
GQE~\cite{hamilton2018embedding} 
performs
the composition with a product of relational matrices
$\{\bm R_{(r_1)},\dots,\bm R_{(r_L)}\}$, as:
$f(e_0, r_1\circ \cdots \circ r_L, e_L) = \bm e_0^\top\bm R_{(r_1)}\cdots\bm
R_{(r_L)}\bm e_L$.
More recently,
Query2box~\cite{ren2019query2box} models the composition of relations as a projection of box embeddings
and defines an entity-to-box distance to measure the score.

\subsubsection{Entity Classification}
\label{sssec:entclass}

Entity classification aims at predicting the labels of the unlabeled entities.
Since the labeled entities are few, a common approach is to 
use a graph convolutional network
(GCN)~\cite{kipf2016semi,gilmer2017neural} 
to aggregate neighborhood information.
The GCN
operates on the local neighborhoods of each entity
and aggregates the representations 
layer-by-layer as:
\begin{equation*}
\bm e_i^{\ell+1} = \sigma 
\big(\bm W_0^\ell \bm e_i^\ell + \sum\nolimits_{j: (i,r,j)\in\mathcal S_{\text{tra}}}\bm W^\ell \bm e_j^{\ell}\big),
\end{equation*}
where 
$\mathcal{S}_{\text{tra}} $ contains all the training triples,
$\sigma$ is the activation function,
$\bm e_i^{\ell}, \bm e_j^{\ell}$
are 
the 
layer-$\ell$
representations of $i$ and the  neighboring entities $j$,
respectively,
and $\bm W_0^\ell, \bm W^\ell\in\mathbb R^{d\times d}$ are weighting matrices sharing across different entities
in the $\ell$th layer.

GCN does not encode relations in edges. To alleviate this problem,
R-GCN~\cite{schlichtkrull2018modeling} and CompGCN~\cite{vashishth2019composition}
encode relation $r$ 
and entity $j$ 
together by a
composition function $\phi$:
\begin{equation*}
\bm e_i^{\ell+1} = \sigma 
\big(\bm W_0^\ell \bm e_i^{\ell}+\sum\nolimits_{(r,j): (i,r,j) \in \mathcal{S}_{\text{tra}}}\bm W^\ell \phi(\bm e_j^\ell,  \bm r^\ell)\big),
\end{equation*}
where 
$\bm r^\ell$ is the representation of relation $r$ 
at the $\ell$th layer.
The composition function $\phi(\bm e_j^\ell,  \bm r^\ell)$ 
can significantly impact performance~\cite{vashishth2019composition}.
R-GCN 
uses the composition operator in RESCAL~\cite{nickel2011three}, and
defines
$\phi(\bm e_j^\ell,  \bm r^\ell) = \bm R_{(r)}^\ell\bm e_j^\ell$,
where $\bm R^\ell_{(r)}$ is a relation-specific weighting matrix in the $\ell$th layer.
CompGCN,
following TransE~\cite{bordes2013translating}, DistMult~\cite{yang2014embedding}
and HolE~\cite{nickel2016holographic}, 
defines
three operators:
subtraction $\phi(\bm e_j^\ell,  \bm r^\ell)=\bm e_j^\ell-\bm r^\ell$, 
multiplication
$\phi(\bm e_j^\ell,  \bm r^\ell)=\bm e_j^\ell\cdot\bm r^\ell$ where
$\cdot$ is the 
element-wise product,
and circular correlation $\phi(\bm e_j^\ell,  \bm r^\ell)=\bm e_j^\ell\star\bm
r^\ell$ where $[\bm a\star \bm b]_k = \sum_{i=1}^d a_i b_{k+i-1 \text{\,mod\,} d}$.

\subsection{Automated Machine Learning (AutoML)}
\label{sec:automl}

Recently,
automated machine learning (AutoML)~\cite{automl_book,quanming2018auto}
has 
demonstrated its advantages
in the design of better machine learning models.
AutoML is often formulated as a bi-level optimization
problem~\cite{colson2007overview}, in which
model parameters 
are updated
from the training data in the inner loop, while
hyper-parameters are tuned from the validation data in the outer loop.
There are three important components
in AutoML~\cite{automl_book,quanming2018auto,bender2018understanding}:
\begin{enumerate}[leftmargin=10.9px]
	\item
	\textit{Search space:}
This identifies important properties of the learning models to search.
The search space should be large enough to cover most manually-designed models,
while specific enough to ensure that the search will not be too expensive.
	
\item \textit{Search algorithm}: 
A search algorithm is used to search for good solutions in the designed space.  
Unlike convex optimization problems, there is no universally efficient optimization tool.

\item \textit{Evaluation}: 
Since the search aims at improving performance, evaluation is needed 
to offer feedbacks to the search algorithm.
The evaluation procedure should be fast and the signal should be accurate.
\end{enumerate}

\subsubsection{Neural Architecture Search (NAS)}
\label{ssec:nas}

Recently, 
a variety of NAS algorithms have been developed to facilitate efficient search
of deep networks
\cite{elsken2019neural,automl_book,zoph2017neural}.
They can generally be divided into 
model-based approach and sample-based approach~\cite{quanming2018auto}.
The model-based approach builds a surrogate model for all candidates in the
search space,
and selects candidates with promising performance
using methods such as 
Bayesian optimization~\cite{feurer2015efficient},
reinforcement learning~\cite{zoph2017neural,pham2018efficient},
and gradient descent~\cite{liu2018darts,yao2019differentiable}.
It requires evaluating a large number of architectures 
for training the surrogate model
or requires a differentiable objective w.r.t. the architecture.
The sample-based approach
is more flexible and
explores new structures
in the search space 
by using heuristics
such as 
progressive algorithm~\cite{liu2018progressive} 
and evolutionary algorithm~\cite{real2019regularized}.
%The zeroth-order optimization approach~\cite{conn2009introduction} 
%obtains search directions
%by perturbation in the continuous space.

As for evaluation,
parameter-sharing~\cite{liu2018darts,pham2018efficient,yao2019differentiable} 
allows faster architecture evaluation by
combining architectures in the whole search space
with the same set of parameters.
However,
the obtained results can
be sensitive to initialization, which hinders reproducibility.
On the other hand,
 stand-alone methods 
\cite{zoph2017neural,liu2018progressive,real2019regularized}
train 
and evaluate 
the different models 
separately.
They are slower but more reliable.
To improve its efficiency,
a predictor can be used to select promising architectures~\cite{liu2018progressive}
before it is fully trained.

\begin{figure*}[ht]
	\centering
	\small 
	\subfigure[DistMult.] {
	\includegraphics[height=3.1cm]{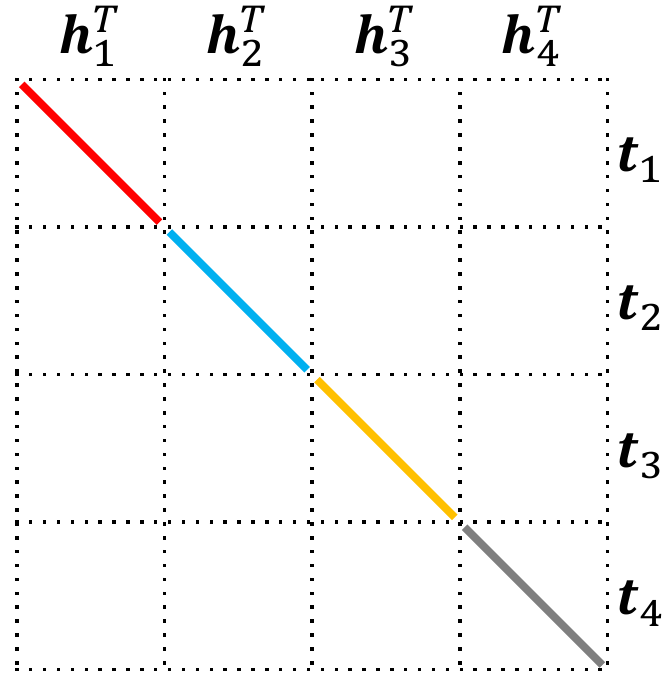}}
	\quad
	\subfigure[SimplE.] {
		\includegraphics[height=3.1cm]{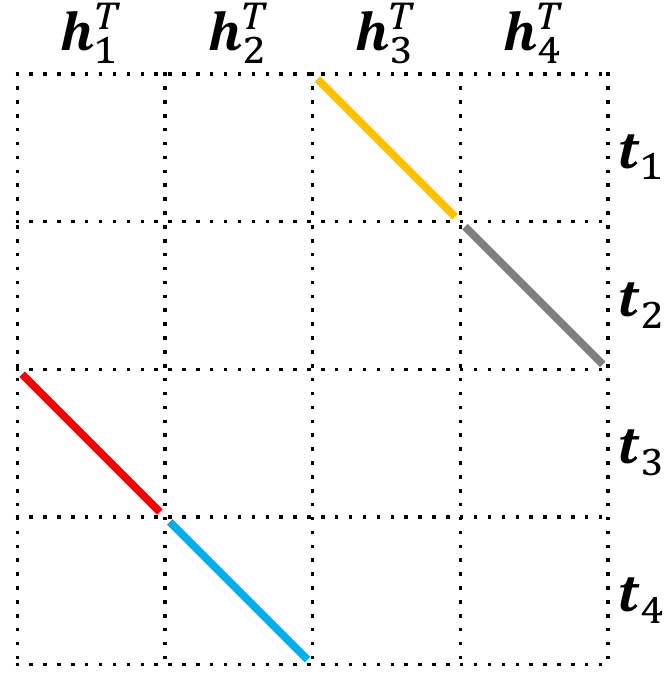}}
	\quad
	\subfigure[ComplEx.] {
		\includegraphics[height=3.1cm]{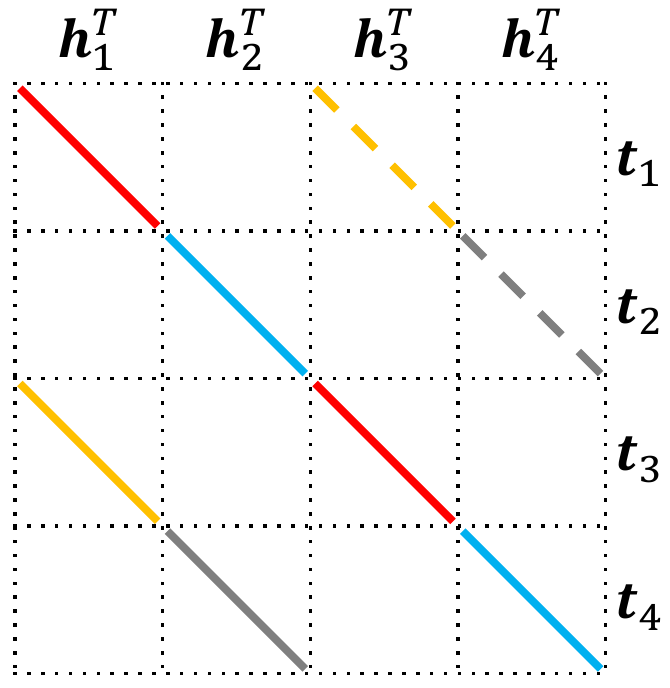}}
	\quad
	\subfigure[Analogy.] {
		\includegraphics[height=3.1cm]{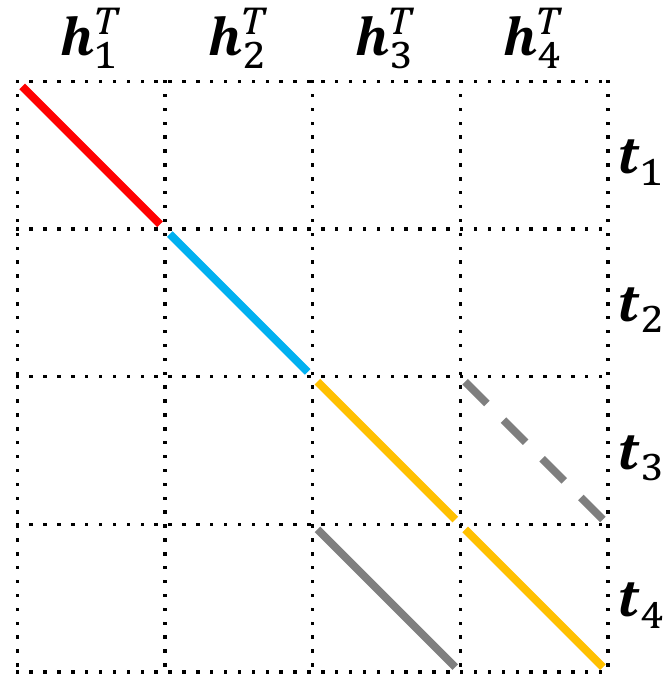}}
	\quad
	\subfigure[QuatE.] {
		\includegraphics[height=3.1cm]{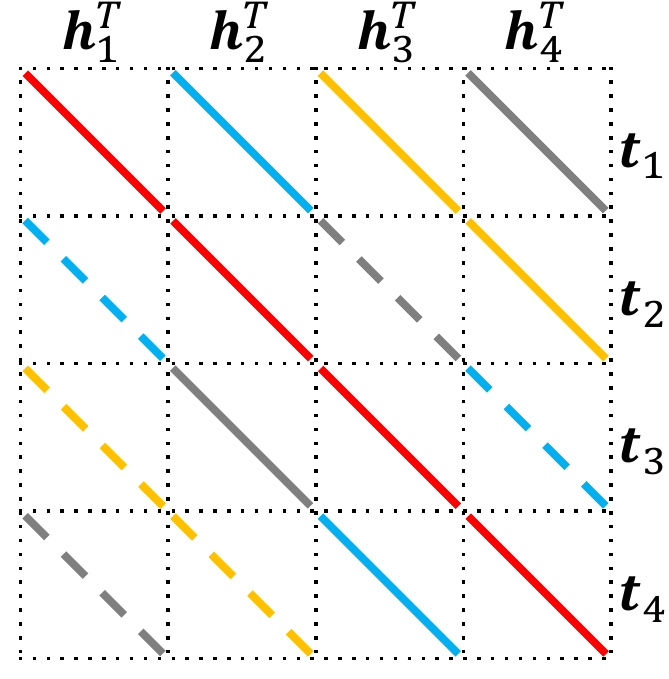}}
	\vspace{-10px}
\caption{The forms of $\bm R_{(\bm r)}$ for representative
BLMs (best viewed in color).  Different colors correspond to 
different parts of $[\bm r_1, \bm r_2, \bm r_3, \bm r_4]$
	(red for $\bm r_1$, blue for $\bm r_2$, yellow for $\bm r_3$, gray for $\bm
	r_4$).
Solid lines mean positive values, while 
dashed lines mean negative values. 
The empty parts have value zero.
}
\label{fig:graphsf}
\vspace{-10px}
\end{figure*}

\section{Automated Bilinear Model}
\label{sec:search}

In the last decade, KG learning has been improving with new scoring function designs.
However, as different KGs
may have different properties,
it is unclear
how 
a proper scoring function 
can be designed 
for a particular KG.
This raises the question:
\textit{Can we automatically design a scoring function for a given KG}?
To address this question,
we first  provide a unified view of
BLMs, 
and then
formulate the 
design of scoring function 
as an AutoML problem:
AutoBLM
(``automated bilinear model").

\subsection{A Unified View of BLM}
\label{ssec:unified}

Recall 
from Section~\ref{ssec:kg} 
that a BLM may operate in  the
real/complex/hypercomplex space.
To write the different BLMs in the same form,
we first unify them to the same representation space.
The idea is to partition each of the 
embeddings $\bm h, \bm r, \bm t$
to 
$K=4$ equal-sized chunks, as
$\bm h = [\bm h_1, \dots, \bm h_4],
\bm r = [\bm r_1, \dots, \bm r_4]$ and
$\bm t = [\bm t_1, \dots, \bm t_4]$.
The BLM 
is then written
in terms of $\{\langle\bm h_i, \bm r_j, \bm
t_k\rangle\}_{i,j,k\in \{1,\dots, 4\}}$.

\begin{itemize}[leftmargin=*]
\item DistMult~\cite{yang2014embedding}, which uses
$f(h,r,t) 
=\left\langle \bm h, \bm r, \bm t\right\rangle$.
We simply split 
$\bm h \in \R^d$  
(and analogously $\bm r$ and $\bm t$)
into 4 parts as $\{\bm h_1, \bm h_2 ,\bm h_3, \bm h_4\}$, where
$\bm h_i\in \R^{d/4}$ for $i=1,2,3,4$. Obviously, 
\begin{align*}
& 
\left\langle \bm h, \bm r, \bm t\right\rangle
\\
& = \left\langle \bm{h}_1,\bm{r}_1,\bm{t}_1\right\rangle 
+ \left\langle\bm{h}_2,\bm{r}_2,\bm{t}_2\right\rangle 
+ \left\langle\bm{h}_3,\bm{r}_3,\bm{t}_3\right\rangle
+ \left\langle\bm{h}_4,\bm{r}_4,\bm{t}_4\right\rangle.
\end{align*}
\item SimplE~\cite{kazemi2018simple} / CP~\cite{lacroix2018canonical},
which uses
$f(h,r,t) 
=\langle \hat{\bm{h}}, \hat{\bm{r}}, \underline {\bm{t}} \rangle+\langle
\underline {\bm{t}}, \underline {\bm{r}}, \hat{\bm{h}} \rangle$.
We split $\hat{\bm h} \in \R^d$ 
(and analogously $\hat{\bm r}$ and $\hat{\bm t}$)
into 2 parts as $\{\bm h_1, \bm h_2\}$ (where
$\bm h_1, \bm h_2 \in \R^{d/2}$),
and similarly $\underline{\bm h}$ as $\{\bm h_3, \bm h_4\}$
(and analogously $\underline{\bm r}$ and $\underline{\bm t}$). Then,
\begin{align*}
& 
\langle \hat{\bm{h}}, \hat{\bm{r}}, \underline {\bm{t}} \rangle+\langle \underline {\bm{t}}, \underline {\bm{r}}, \hat{\bm{h}} \rangle
\\
& 
= \left\langle \bm{h}_1, \bm{r}_1, \bm{t}_3\right\rangle 
+ \left\langle\bm{h}_2, \bm{r}_2, \bm{t}_4\right\rangle
+ \left\langle\bm{h}_3, \bm{r}_3, \bm{t}_1\right\rangle
+ \left\langle\bm{h}_4, \bm{r}_4, \bm{t}_2\right\rangle.
\end{align*}

\item ComplEx~\cite{trouillon2017knowledge} / HolE~\cite{nickel2016holographic},
which uses
$f(h,r,t) 
 =\text{Re}(\bm h\otimes \bm r\otimes \bar{\bm t})$, where
 $\bm h,\bm r,\bar{\bm t}$ are complex-valued.
Recall that any complex vector $\bm v
\in \mathbb C^d$
 is of the form $\bm v_{r} + i\bm v_{i}$,
where $\bm v_{r}
\in\mathbb R^d$ is the real part and
$\bm v_{i}\in\mathbb R^d$ is the imaginary part.
 Thus,
 \begin{eqnarray}
 \text{Re}(\bm h\otimes \bm r\otimes \bar{\bm t}) & = & \langle \bm h_{r}, \bm r_{r}, \bm t_{r}\rangle
 + \langle \bm h_{i}, \bm r_{r}, \bm t_{i}\rangle  \nonumber\\
&& + \langle \bm h_{r}, \bm r_{i}, \bm t_{i}\rangle 
 - \langle \bm h_{i}, \bm r_{i}, \bm t_{r}\rangle. \label{eq:compiex}
 \end{eqnarray}
We split
$\bm h_{r} \in \R^d$ 
(and analogously $\bm r_r$ and $\bm t_r$)
into 2 parts $\{\bm h_1, \bm h_2\}$ (where
$\bm h_1, \bm h_2 \in \R^{d/2}$),
and  similarly $\bm h_{i}=\{\bm h_3, \bm h_4\}$
(and analogously $\bm r_i$ and $\bm t_i$).
Then,
\begin{eqnarray*}
\lefteqn{\text{Re}(\bm h\otimes \bm r\otimes \bar{\bm t})} \\ 
\!\!& \!\!=\!\! & \!\! \big(\!\langle \bm{h}_1,\!\bm{r}_1,\!\bm{t}_1\rangle 
\!+\! \langle\bm{h}_2,\!\bm{r}_2,\!\bm{t}_2\rangle\!\big)
\!+\! \big(\!\langle\bm{h}_3,\!\bm{r}_1,\!\bm{t}_3\rangle 
\!+\! \langle\bm{h}_4,\!\bm{r}_2,\!\bm{t}_4\rangle\!\big)\\
\!\!& \!\! &
\!\! + \big(\!\langle \bm{h}_1,\!\bm{r}_3,\!\bm{t}_3\rangle 
\!+\! \langle\bm{h}_2,\!\bm{r}_4,\!\bm{t}_4\rangle\!\big)
\!-\! \big(\!\langle\bm{h}_3,\!\bm{r}_3,\!\bm{t}_1\rangle 
\!-\! \langle\bm{h}_4,\!\bm{r}_4,\!\bm{t}_2\rangle\!\big).
\end{eqnarray*}

\item Analogy~\cite{liu2017analogical}, which uses
$f(h,r,t) 
=\langle \hat{\bm h}, \hat{\bm r}, \hat{\bm t}\rangle 
+ \text{Re}\big(\underline {\bm h}\otimes \underline {\bm r}\otimes\bar{\underline {\bm t}}\big)$.
We split 
$\hat{\bm h} \in \R^d$ 
(and analogously $\hat{\bm r}$ and $\hat{\bm t}$)
into 2 parts
$\{\bm h_1, \bm h_2\}$ (where
$\bm h_1, \bm h_2 \in\mathbb R^{d/2}$),  and similarly
$\underline{\bm h}
\in\mathbb C^{d/2}$ 
(and analogously $\underline{\bm r}$ and $\underline{\bm t}$)
into 2 parts
$\{\bm h_3, \bm h_4\}$ (where
$\bm h_3, \bm h_4
\in\mathbb R^{d/2}$).
Then,
\begin{align*}
& \langle \hat{\bm h}, \hat{\bm r}, \hat{\bm t}\rangle 
+ \text{Re}\big(\underline {\bm h}\otimes \underline {\bm r}\otimes\bar{\underline {\bm t}}\big)\\
&= \left\langle \bm{h}_1, \bm{r}_1, \bm{t}_1\right\rangle 
+ \left\langle\bm{h}_2, \bm{r}_2, \bm{t}_2\right\rangle
+ \left\langle\bm{h}_3, \bm{r}_3, \bm{t}_3\right\rangle 
+ \left\langle\bm{h}_3, \bm{r}_4, \bm{t}_4\right\rangle \\
&
\quad
+ \left\langle\bm{h}_4, \bm{r}_3, \bm{t}_4\right\rangle
- \left\langle\bm{h}_4, \bm{r}_4, \bm{t}_3\right\rangle.
\end{align*}

\item QuatE~\cite{zhang2019quaternion}, which uses
$f(h,r,t) 
=\bm h\odot \bm r \odot \bm t$.
Recall that any hypercomplex vector $\bm v
\in \mathbb H^d$
is of the form $\bm v_1 + i\bm v_2 + j\bm v_3 + k \bm v_4$, where
$\bm v_1, \bm v_2,\bm v_3, \bm v_4 \in 
\mathbb R^{d}$. Thus, 
\begin{align*}
& \bm h\odot \bm r \odot \bm t \\
&= \left\langle \bm{h}_1, \bm{r}_1, \bm{t}_1\right\rangle 
- \left\langle \bm{h}_1, \bm{r}_2, \bm{t}_2\right\rangle 
- \left\langle \bm{h}_1,\bm{r}_3,\bm{t}_3\right\rangle 
- \left\langle \bm{h}_1,\bm{r}_4,\bm{t}_4\right\rangle \\
&+ \left\langle \bm{h}_2,\bm{r}_2, \bm{t}_1\right\rangle 
+ \left\langle \bm{h}_2, \bm{r}_1, \bm{t}_2\right\rangle 
+ \left\langle \bm{h}_2, \bm{r}_4, \bm{t}_3\right\rangle 
- \left\langle \bm{h}_2, \bm{r}_3, \bm{t}_4\right\rangle \\
& + \left\langle \bm{h}_3, \bm{r}_3, \bm{t}_1\right\rangle 
- \left\langle \bm{h}_3, \bm{r}_4, \bm{t}_2\right\rangle 
+ \left\langle \bm{h}_3, \bm{r}_1, \bm{t}_3\right\rangle 
+ \left\langle \bm{h}_3, \bm{r}_2, \bm{t}_4\right\rangle \\
& + \left\langle \bm{h}_4, \bm{r}_4, \bm{t}_1\right\rangle 
+ \left\langle \bm{h}_4, \bm{r}_3, \bm{t}_2\right\rangle
- \left\langle \bm{h}_4, \bm{r}_2, \bm{t}_3\right\rangle 
+ \left\langle \bm{h}_4, \bm{r}_1, \bm{t}_4\right\rangle.
\end{align*}
\end{itemize}
As $\langle\bm h_i, \bm r_k, \bm t_j\rangle = \bm h_i^\top\diag{\bm r_k}\bm t_j$,
all the above BLMs can be written in the form of a bilinear function
\begin{equation} \label{eq:bilinear}
\bm h^\top \bm R_{(\bm r)}\bm t,
\end{equation} 
where\footnote{With a slight abuse of notations, we still use $d$ to denote the dimensionality after this transformation.}
$\bm h= [\bm h_1^\top, \dots, \bm h_4^\top]^\top,
\bm t = [\bm t_1; \dots; \bm t_4] \in \R^d$,
and
$\bm R_{(\bm r)}\in\mathbb R^{d\times d}$
is a matrix
with $4\times 4$ blocks,
each block being either $\bm 0, \pm \diag{\bm r_1}, \dots$, or $\pm\diag{\bm r_4}$.
Figure~\ref{fig:graphsf} shows graphically
the $\bm R_{(\bm r)}$ for the BLMs considered.

\subsection{Unified Bilinear Model}
\label{ssec:unifiedBLM}

Using the above unified representation, 
the design of BLM becomes designing $\bm R_{(\bm r)}$ in (\ref{eq:bilinear}).

\begin{definition}[Unified BiLinear Model] 
	\label{def:unify}
The desired scoring function is  of the form
\begin{align}
f_{\bm{A}}({h},{r},{t})  
= 
\sum\nolimits_{i,j=1}^{K}
\!{\emph{sign}}(A_{ij})\!\left\langle\bm{h}_i, \bm r_{|A_{ij}|}, \bm{t}_j\right\rangle,
\label{eq:uni1}
\end{align}
where
\begin{equation} \label{eq:A}
\bm{A}\in \{0,\pm 1,\dots, \pm K\}^{ K\times K }
\end{equation} 
is called the {\em structure matrix}.
Here, we define $\bm r_0\equiv \bm 0$, and $\emph{sign}(0)=0$.
\end{definition}
It can be easily seen that this covers all the BLMs 
in Section~\ref{ssec:unified}
when $K=4$.  
Let $g_K(\bm{A}, \bm r)$ 
be a matrix with $K\times K$ blocks, with its $(i,j)$-th block:
\begin{equation} \label{eq:gk}
[g_K(\bm{A}, \bm r)]_{ij}=\text{sign}(A_{ij})\cdot \text{diag}(\bm r_{|A_{ij}|}).
\end{equation}
The form in (\ref{eq:uni1}) can be written more compactly as
\begin{equation} 
\label{eq:uni} 
f_{\bm{A}}({h},{r},{t})  
= \bm{h}^{\top} g_{K}(\bm{A}, \bm{r} ) \bm{t}.
\end{equation} 
A graphical illustration is shown in Figure~\ref{fig:autosf}.

\begin{figure}[ht]
\centering
\vspace{-4px}
\includegraphics[width=0.9\columnwidth]{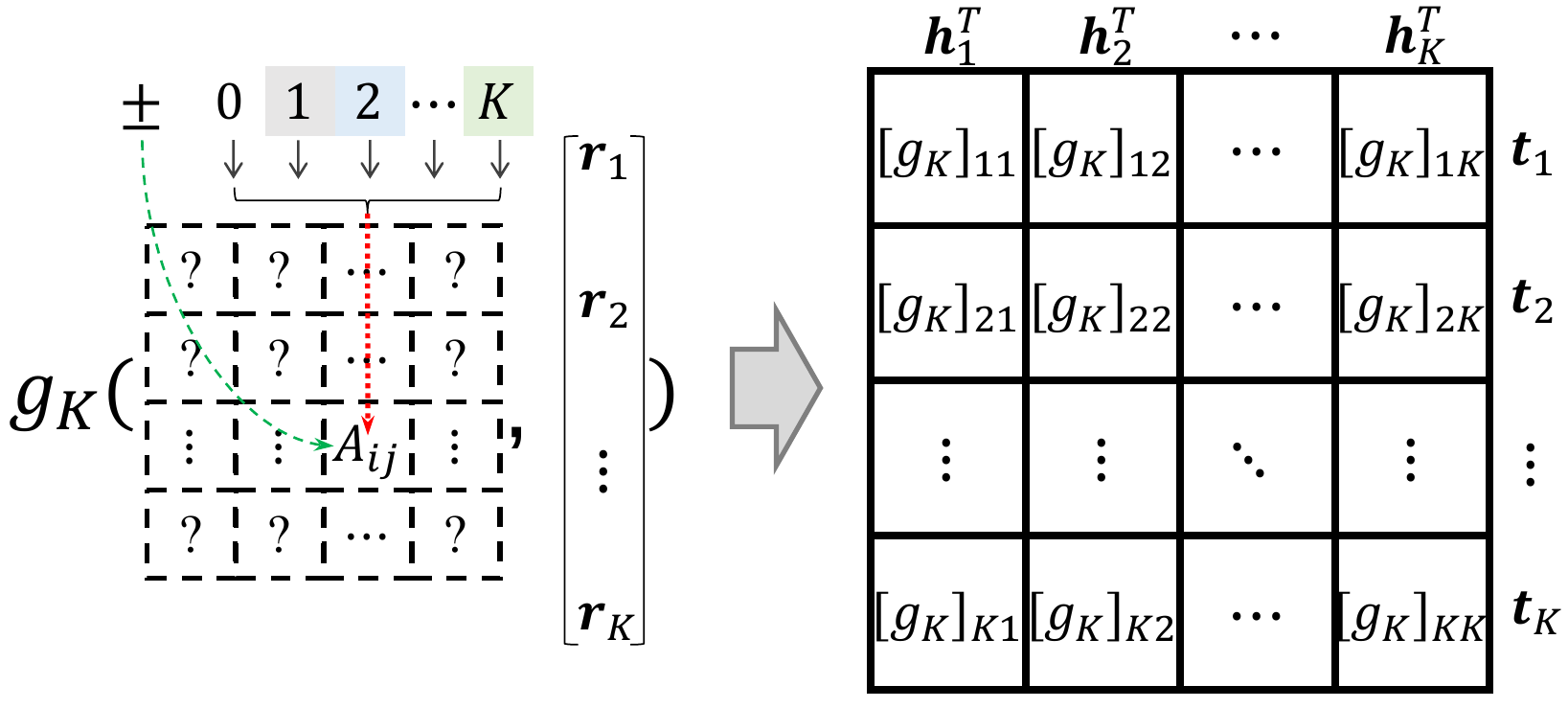}
\vspace{-8px}
\caption{A graphical illustration of 
the proposed form of $f_{\bm{A}}({h},{r},{t})$.}
\vspace{-4px}
\label{fig:autosf}
\end{figure}

The following Proposition gives a necessary and sufficient condition for the
BLM  with scoring function
in (\ref{eq:uni})
to be fully expressive.
The proof is in Appendix~\ref{app:expBLMs}.

\begin{prop} \label{pr:expBLMs}
Let
\begin{align} 
\mathcal{C} 
\equiv 
\{ \bm{r} & \in\mathbb R^K \,|\, 
\bm r\neq \bm 0, 
\notag
\\ 
&r[i]\in
\{0,\pm1,\dots,\pm K\},
i=1,\dots, K \}.
\label{eq:C}
\end{align} 
Given an $\bm{A}$
in (\ref{eq:A}),
the bilinear model with scoring function 
\eqref{eq:uni}
is fully expressive
if
\begin{enumerate}
\item $\exists \hat{\bm{r}} \in \mathcal{C}$ 
such that 
$g_{K}(\bm{A},\hat{\bm{r}})$ is symmetric (i.e.,
$g_{K}(\bm{A},\hat{\bm{r}})^\top=g_{K}(\bm{A}, \hat{\bm{r}})$),
and
\item $\exists \breve{\bm{r}}
\in \mathcal{C}$ 
such that 
$g_{K}(\bm{A},\breve{\bm{r}})$ is skew-symmtric
(i.e.,
$g_{K}(\bm{A},\breve{\bm{r}})^\top = -g_{K}(\bm{A},\breve{\bm{r}})$).
\end{enumerate}
\end{prop}
Table~\ref{tab:conditions} shows examples 
		of $\hat{\bm{r}}$   and $\breve{\bm{r}}$  
for the existing BLMs  (ComplEx, HolE, Analogy, SimplE, CP, and QuatE), thus justifying that they are fully expressive.

\begin{table}[ht]
	\centering
%	\vspace{-8px}
\caption{Example
		$\hat{\bm{r}}$   (resp. $\breve{\bm{r}}$) for 
the 
two conditions
in Proposition~\ref{pr:expBLMs}.}
	\label{tab:conditions}
	\vspace{-10px}
	\setlength\tabcolsep{8pt}
	\renewcommand{\arraystretch}{1.1}
	\begin{tabular}{c|c|c}
		\toprule
		model & $\hat{\bm{r}}$   & $\breve{\bm{r}}$   \\
		\midrule
		ComplEx / HolE & $[1,2,0,0]$  & $[0,0,3,4]$ \\
		Analogy  & $[1,2,3,0]$  & $[0,0,0,4]$ \\
		SimplE / CP & $[1,2,1,2]$  & $[1,2,-1,-2]$ \\
		QuatE  &  $[1,0,0,0]$   & $[0,2,3,4]$  \\
		\bottomrule
	\end{tabular}
\vspace{-10px}
\end{table}

\subsection{Searching for BLMs}
\label{ssec:searSFs}

Using the family of unified BLMs in Definition~\ref{def:unify} as the search
space $\mathcal A$ for structure matrix
$\bm{A}$,
the search for a good data-specific BLM 
can be formulated
as
the following 
AutoML problem.

\begin{definition}[Bilinear Model Search (AutoBLM)]
\label{def:autoSF}
Let $F(\bm{P}; \bm{A})$ be
a KG embedding model (where 
$\bm{P}$ includes the
entity embedding matrix
$\bm{E}$ and relation embedding matrix $\bm{R}$,
and 
$\bm{A}$
is the structure matrix)
and $M(F, \mathcal S)$ be the performance measurement
of $F$ on triples $\mathcal{S}$
(the higher the better).
The AutoBLM
 problem is formulated as:
\begin{eqnarray}
&\bm{A}^* &
\in     
 \emph{Arg}\max\nolimits_{\bm{A} \in \mathcal{A}}
M\left(F(\bm{P}^*; \bm{A}), \mathcal{S}_{\text{val}}\right)
\label{eq:autosf:l1}
\\
\text{s.t.} 
&\bm{P}^* &
=
\arg\max\nolimits_{\bm{P}} M\left(F(\bm{P}; \bm{A}), \mathcal{S}_{\text{tra}}\right),
\label{eq:autosf}
\end{eqnarray}
where
\begin{align}
\mathcal{A}= \{ \bm{A} 
& = [A_{ij}]\in \R^{ K\times K } 
\notag
\\
& |\; A_{ij}\in \{
0,\pm 1,\dots, \pm K\} \;\forall i,j=1,\dots,K\},
\label{eq:space}
\end{align}
contains all the possible choices of $\bm{A}$,
$\mathcal{S}_{\text{tra}}$
is the training set, and 
$\mathcal{S}_{\text{val}}$ is the
validation set.
\end{definition}
As a bi-level optimization problem,
we first train the model to obtain $\bm{P}^*$ (converged model parameters) 
on the training set $\mathcal{S}_{\text{tra}}$ by 
\eqref{eq:autosf},
and then search for a better $\bm{A}$ (and consequently a
better relation matrix 
$g_{K}(\bm{A}, \bm r)$)
based on its performance $M$ on the validation set $\mathcal{S}_{\text{val}}$
in \eqref{eq:autosf:l1}.
Note that the objectives
in (\ref{eq:autosf:l1}) and 
(\ref{eq:autosf})
are non-convex, and
the search space is large
(with $(2K+1)^{K^2}$ candidates, as can be seen from
(\ref{eq:A})).
Thus, solving \eqref{eq:autosf} can be expensive and
challenging.

\subsection{Degenerate and Equivalent Structures}
\label{ssec:challange}

In this section, 
we introduce properties specific to the proposed search space $\mathcal{A}$.
A careful exploitation of these would be key to an efficient
search.

\subsubsection{Degenerate Structures}
\label{sssec:degenerate}

Obviously, not all structure matrices 
in (\ref{eq:A})
are
equally good.
For example, if all the nonzero blocks 
in $g_K(\bm{A}, \bm r)$
are in the first column,
$f_{\bm{A}}$ will be zero
for all head embeddings with $\bm h_1=\bm 0$.
These structures should be avoided.

\begin{definition}[Degenerate structure]
Matrix
${\bm{A}}$ 
is degenerate
if (i) there exists $\bm h\neq {\bf 0}$
such that $\bm h^\top g_K(\bm{A},\bm r)\bm t=0, \forall \bm r, \bm t$;
or (ii) there exists $\bm r\neq {\bf 0}$
such that $\bm h^\top g_K(\bm{A},\bm r)\bm t=0, \forall \bm h, \bm t$.
	\label{def:degenerate}
\end{definition}
With a degenerate ${\bm{A}}$,
the triple $(h,r,t)$ is always non-plausible for every nonzero
head embedding $\bm h$ or relation embedding $\bm r$,
which limits expressiveness of the scoring function.
The following Proposition
shows that
it is easy to check whether $\bm{A}$ is degenerate.
Its proof is in Appendix~\ref{app:degenerate}.
\begin{prop} \label{pr:degenerate}
$\bm{A}$ is not degenerate if and only if
$\emph{rank}(\bm{A})=K$
and
$\{1,\dots, K\} \subset \{|A_{ij}|: i,j=1,\dots, K\}$.
\end{prop}

Since $K$  is very small (which is equal to 4 here), the above conditions are
inexpensive to check. Hence,
we can efficiently filter out degenerate 
$\bm{A}$'s
and avoid wasting
time in training and evaluating these structures.

\subsubsection{Equivalence}
\label{sssec:equiv}

In general, two different $\bm{A}$'s can have the same performance (as measured by 
$F$ in Definition~\ref{def:autoSF}).
This is captured in the following notion of equivalence.
If a group of $\bm{A}$'s are equivalent, we only need to evaluate one of them.

\begin{definition}[Equivalence]
	\label{def:equiv}
Let 
${\bm P}^* =\arg\max\nolimits_{\bm P} $
$M(F(\bm{P}; \bm{A}), \mathcal{S})$  and
${\bm P'}^* =\arg\max\nolimits_{\bm P'} $
$M(F(\bm{P'}; \bm
A'), \mathcal{S})$.
If $\bm{A} \neq \bm{A}'$ but
$M\big(F({\bm P}^*; \bm{A}), \mathcal S \big)$
$= M\big(F({\bm P'}^*;
\bm{A}'), \mathcal S \big)$
for all 
$\mathcal S$,
then 
$\bm{A}$ is equivalent to $\bm{A}'$
(denoted $\bm{A}\equiv \bm{A}'$).	

\end{definition}

The following
Proposition
shows several conditions for two structures to be equivalent.
Its proof is in Appendix~\ref{app:equiv}.
Examples are shown in Figure~\ref{fig:equiv}.

\begin{figure}[t]
	\centering
	\vspace{-3px}
	\subfigure[Analogy. \label{fig:ana}]
	{\includegraphics[height=3.1cm]{figure/analogy.pdf}}
	\qquad\quad
	\subfigure[Permuting rows and columns. \label{fig:cond1}]
	{\includegraphics[height=3.1cm]{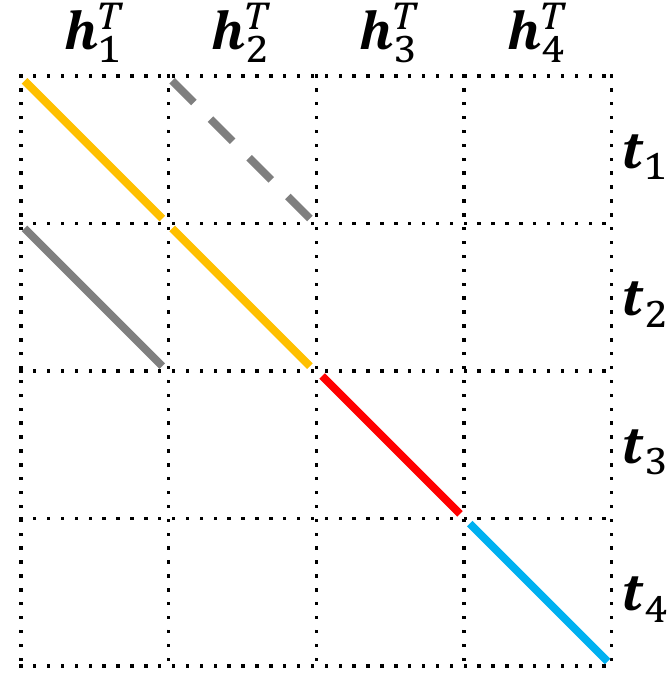}}
	
	\subfigure[Permuting values. \label{fig:cond2}]
	{\includegraphics[height=3.1cm]{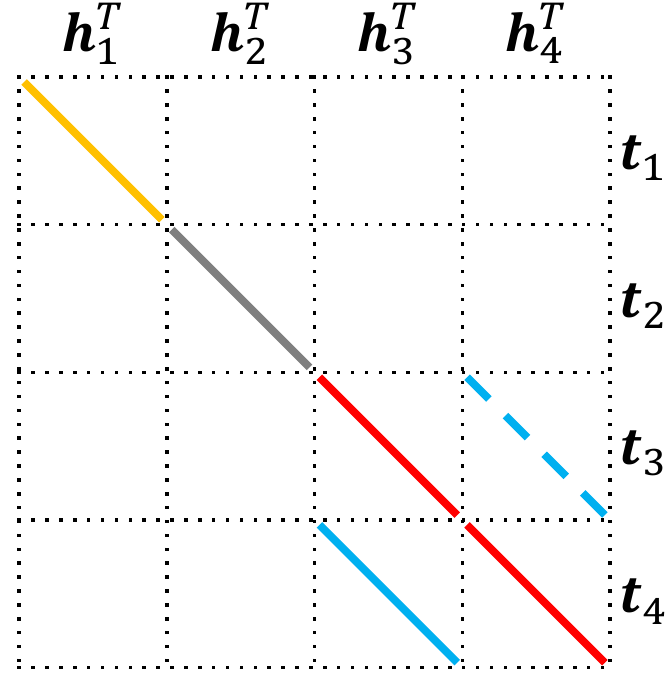}}
	\qquad\quad
	\subfigure[Flipping signs. \label{fig:cond3}]
	{\includegraphics[height=3.1cm]{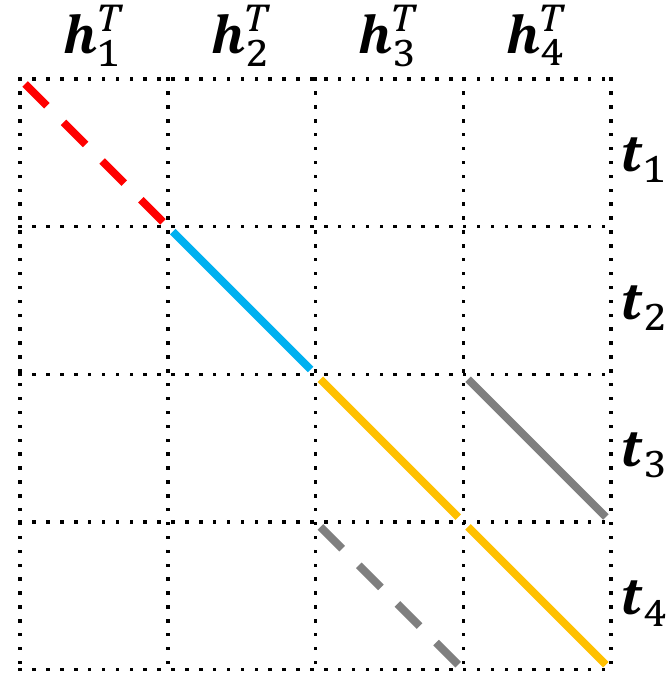}}
	\vspace{-8px}
	\caption{Illustration of $\bm R_{(\bm r)}$ for Analogy (Figure~\ref{fig:ana}) and
		three example equivalent structures
		based on the conditions in Proposition~\ref{pr:equiv}.
		Figure~\ref{fig:cond1} permutes the index $[1,2,3,4]$ of rows and columns in $\bm{A}$ to $[3,4,1,2]$;
		Figure~\ref{fig:cond2} permutes the values $[1,2,3,4]$ in $\bm{A}$ to $[3,4,1,2]$;
		Figure~\ref{fig:cond3} flips the signs of values $[1,2,3,4]$ in $\bm{A}$ to $[-1,2,-3,4]$.}
	\label{fig:equiv}
	\vspace{-8px}
\end{figure}

\begin{prop} \label{pr:equiv}
Given an $\bm{A}$ in (\ref{eq:A}), construct
$\bm \Phi_{\bm{A}}\in \mathbb R^{K\times K^2}$ such
	that
	$[\bm \Phi_{\bm{A}}]_{
	|A_{ij}|
	,(i-1)K+j} = \emph{sign}(A_{ij})$
	if $|A_{ij}|
	%=k
	\in\{1,\dots, K\}$,
	and 0
otherwise.\footnote{Intuitively,
in $\bm \Phi_{\bm{A}}$, the indexes of nonzero values 
in its 
	$|A_{ij}|$-th
row 
indicate positions of elements in ${\bm{A}}$ whose absolute values are 
	$|A_{ij}|$.}
Two structure matrices
$\bm{A}$ and $\bm{A}'$ are equivalent
if any one of the following conditions 
is satisfied.
\begin{enumerate}[label=(\roman*)]
	\item 
	Permuting rows and columns: There
	%of $\bm{A}$:
	exists a permutation matrix $\bm \Pi\in\mathbb \{0,1\}^{K\times K}$ such that $\bm{A}' = \bm \Pi^{\top} \bm{A}\bm \Pi$.
	%where $\bm P\in\mathbb R^{K\times K}$ is a permutation matrix.
	%Permuting the values:
	%Let $\phi_{\bm{A}}(k)=\{A_{ij} \;:\; |A_{ij}|=k\}$, where $k=1,\dots, K$.
		\item 
Permuting values:
There exists a permutation matrix $\bm \Pi\in\{0,1\}^{K\times K}$ such
that
		$\bm\Phi_{\bm{A}'} = \bm \Pi\bm \Phi_{\bm{A}}$;
		\item Flipping signs: There exists a sign vector $\bm s\in \{\pm 1\}^K$
		such that 
$[\bm\Phi_{\bm{A}'}]_{i,\cdot} = s_i \cdot [\bm\Phi_{\bm{A}}]_{i,\cdot}, \forall i=1,\dots, K$.
\end{enumerate}
\end{prop}

There are $K!$ possible permutation matrices for conditions
(i) and (ii), and $2^K$ 
possible sign vectors for condition (iii). Hence, one has to check
a total of $(K!)^22^K$ combinations.

\section{Search Algorithm}
\label{sec:progred}

In this section, 
we 
design efficient 
algorithms 
to search  for
the structure matrix 
$\bm{A}$ 
in (\ref{eq:A}).
As discussed in Section~\ref{ssec:nas},
 the
model-based approach
requires a proper surrogate model for such a complex space.
%while
%as the zeroth-order optimization approach works only with continuous spaces (rather
%than in a discrete space here) \cite{golovin2019gradientless}.
Thus,
we will focus on the sample-based approach, particularly on 
the progressive algorithm 
and evolutionary algorithm.
To search efficiently, 
one needs to (i) ensure that each new $\bm{A}$ is neither 
degenerate 
nor equivalent to an already-explored structure; and (ii) the
scoring function $f_{\bm{A}}({h},{r},{t})$ obtained from the new $\bm{A}$ is likely
to have high
performance.
These can be achieved by designing an efficient filter
(Section~\ref{ssec:filter}) and performance predictor
(Section~\ref{sssec:predictor}).
Then, we introduce two search algorithms: progressive search
(Section~\ref{ssec:progressive})
and evolutionary
algorithm
(Section~\ref{ssec:evolution}).

\subsection{Filtering Degenerate and Equivalent Structures}
\label{ssec:filter}

Algorithm~\ref{alg:filter} shows the filtering procedure.
First,
step~\ref{step:filt:rank}
removes
degenerate structure matrices 
by
using the conditions in Proposition~\ref{pr:degenerate}.
Step~\ref{step:filt:gen} then generates a set of $(K!)^22^K$ structures that are equivalent to $\bm{A}$
(Proposition~\ref{pr:equiv}).
$\bm{A}$ is filtered out
if any of its equivalent structures appears in the set 
$\mathcal H$ containing
structure matrices that have
already been explored.
As $K$ is small,
this filtering 
cost 
is very low
compared with the cost of model training in \eqref{eq:autosf}.

\begin{algorithm}[ht]
	\caption{
		Filtering degenerate and equivalent structure matrices. The
		output is ``False" if the input structure matrix
		$\bm{A}$ is to be filtered out.}
	\label{alg:filter}
	\small
	\begin{algorithmic}[1]
		\REQUIRE $\bm{A}$: a $K\times K$ structure matrix, $\mathcal H$: a set of structures.
		\STATE \textbf{initialization:} $\mathcal Q(\bm{A}, \mathcal H)=\text{True}$.
		\STATE \textbf{if} $\det(\bm{A}) = 0$ or 
		$\{1,\dots, K\} \not\subset \{|A_{ij}|: i,j=1,\dots, K\}$,
		\textbf{then} $\mathcal Q(\bm{A}, \mathcal H)=\text{False}$. \label{step:filt:rank}
		\STATE 
		generate a set of equivalent structures $\{\bm{A}'\!: \!\bm{A}'\!\equiv\! \bm{A}\}$ by
		enumerating permutation matrices $\bm P$'s and sign vectors $\bm s$'s. \label{step:filt:gen}$\!\!\!$
		\FOR{$\bm{A}'$ in $\{\bm{A}': \bm{A}'\equiv \bm{A}\}$}
		\STATE \textbf{if} $\bm{A}' \in \mathcal H$,  \textbf{then} $\mathcal Q(\bm{A}, \mathcal H)=\text{False}$,
		and exit the loop.
		\label{step:filt:equiv}
		\ENDFOR
		%		\FOR all the permutation matrix $\bm P_1$ 
		%		  	\FOR all the permutation matrix $\bm P_2$
		%		  		\FOR all the sign vector $\bm s$
		%		  			\STATE \textbf{if} $$
		%		  		\ENDFOR 
		%		  	\ENDFOR 
		%		\ENDFOR
		\RETURN $\mathcal Q(\bm{A}, \mathcal H)$.
	\end{algorithmic}
\end{algorithm}

\subsection{Performance Predictor}
\label{sssec:predictor}

After collecting $N$ structures 
in $\mathcal H$, 
we construct a predictor $\mathcal P$ to estimate the goodness of each 
$\bm{A}$.
As mentioned in Section~\ref{sec:automl}, 
search efficiency depends heavily on  how to 
evaluate the candidate models.

A highly efficient approach is parameter sharing, 
as is popularly used in
one-shot neural architecture search (NAS)~\cite{pham2018efficient,liu2018darts}.
However,
parameter sharing
can be problematic
when used to predict the performance 
of scoring functions.
Consider 
the following two 
$\bm{A}$'s: (i)
$\bm{A}_1$ is a $4\times 4$ matrix of all $+1$'s, and so $f_{\bm{A}_1}(h, r, t) =
\sum_{i,j=1}^4\langle\bm h_i, \bm r_1, \bm t_j\rangle$, and
(ii)
$\bm{A}_2$ is a $4\times 4$ matrix of all $-1$'s, and so $f_{\bm{A}_2}(h, r, t) =
-\sum_{i,j=1}^4\langle\bm h_i, \bm r_1, \bm t_j\rangle
= -f_{\bm{A}_1}(h, r, t)$.
When parameter sharing is used, it is likely that the performance predictor will output different
scores for
$\bm{A}_1$ 
and $\bm{A}_2$.
However, 
from 
Proposition~\ref{pr:equiv},
by setting $\bm s=[-1,-1,-1,-1]$
in condition (iii),
we have
$\bm{A}_1\equiv \bm{A}_2$ and thus they indeed have the same performance.
This problem 
will also be empirically demonstrated in Section~\ref{sec:exp:ps}.
Hence, instead,
we train and evaluate the models separately as in the 
stand-alone NAS evaluation 
\cite{zoph2017neural,liu2018progressive}.

\begin{algorithm}[ht]
	\caption{Construction of the symmetry-related feature (SRF) vectors.}
	\small
	\label{alg:srf}
	\begin{algorithmic}[1]
		\REQUIRE structure matrix $\bm{A}$.
		\STATE \textbf{initialization:} $\bm \alpha, \bm \beta :=\bm 0$.
		\FOR{$\bm r\in \mathcal C$}
		\IF{$\bm r\neq \bm 0$}
		\STATE $x=\big|\{i: r_i\!=\! 0\}\big|$;
		\STATE $y=\big|\{j>0: r_i\!=\! j \text{ or } r_i\!=\! -j\}\big|$;
		\\ // \textit{for symmetric case}
		\STATE \textbf{if} {$g_{K}(\bm{A}, \bm r) \!-\! g_{K}(\bm{A}, \bm r)^\top \!\!=\! \mathbf 0$} 
		\textbf{then} $\alpha_{(x,y)} \!=\! 1$; \label{step:srf:sym}
		\\ // \textit{for skew-symmetric case}
		\STATE \textbf{if} {$g_{K}(\bm{A}, \bm r) \!+\! g_{K}(\bm{A}, \bm r)^\top \!\!=\! \mathbf 0$} 
		\textbf{then} $\beta_{(x,y)} \!=\! 1$;  \label{step:srf:ssym}
		\ENDIF
		\ENDFOR
		\RETURN $[\text{vec}(\bm \alpha);\text{vec}(\bm \beta)]$.
	\end{algorithmic}
\end{algorithm}

Recall  from Section~\ref{ssec:kg}
that 
it is desirable for
the scoring function
to be fully expressive.
Proposition~\ref{pr:expBLMs} shows that this requires looking for 
a $\hat{\bm{r}}
\in \mathcal{C}$ 
such that $g_K(\bm{A}, \hat{\bm{r}})$ is symmetric and
a $\breve{\bm{r}} 
\in \mathcal{C}$ 
such that $g_K(\bm{A}, \breve{\bm{r}})$ is skew-symmetric.
This motivates us to examine 
each of the 
$(2K + 1)^K-1$ 
${\bm{r}}$'s in $\mathcal{C}$ (defined in (\ref{eq:C}))  and see whether it 
leads to a symmetric or skew-symmetric
$g_K(\bm{A}, {\bm{r}})$.
However, 
directly using all these $(2K + 1)^K-1$ choices as features to a performance predictor 
can be computationally expensive.
Instead,
empirically we find  
that
the following two features can be used to group
the scoring functions:
(i) number of zeros in $\bm{r}$:
$|\{i
\in \{1,\dots, K\}
: r_i= 0\}|$;
and (ii)
number of nonzero absolute values in $\bm{r}$:
$|\{j
>0: r_i= j \text{ or } r_i =-j, \;
i\in \{1,\dots, K\}
\}|$.
The possible choices
is reduced 
to $K(K+1)/2$ (groups of scoring functions).
We keep two 
	symmetry-related feature (SRF)
	as
	$\bm \alpha$ and $\bm \beta$.
If $g_K(\bm{A}, \bm r)$ is symmetric (resp. skew-symmetric) for any
${\bm{r}}$ in $\mathcal{C}$,
the 
entry in $\bm \alpha$ (resp. $\bm \beta$) 
corresponding 
to ${\bm{r}}$ 
is set to 1.
The construction process is also shown in Algorithm~\ref{alg:srf}.
Finally,
the SRF vector is 
composed with $\text{vec}(\bm \alpha)$ and $\text{vec}(\bm \beta)$,
which vectorize the values in $\bm \alpha$ and $\bm \beta$,
and
fed as input 
to a two-layer MLP
for performance prediction.

\begin{algorithm}[t]
	\caption{Progressive search algorithm (AutoBLM).}
	\label{alg:greedy}
	\small
	\begin{algorithmic}[1]
		\REQUIRE
		$I$: number of top structures;
		$N$: number of generated structures;
		$P$: number of structures selected by $\mathcal P$;
		$b_0$:  number of nonzero elements in initial structures;
		filter $\mathcal Q$ and performance predictor $\mathcal P$.
		\STATE \textbf{initialization:} 
		$b:=b_0$, create a candidate set $\mathcal H^b=\emptyset$;
		\label{step:greedyinit}
		\FOR{each $\bm{A}^{(b)}\in\{\bm{A}^{(b)}\}$} \label{step:Kstart}
		\STATE \textbf{if} $\mathcal Q(\bm{A}^{(b)}, \mathcal H^b)$ from Algorithm~\ref{alg:filter} is true \\ \textbf{then} $\mathcal{H}^b \leftarrow \mathcal{H}^b \cup \{ \bm{A}^{(b)}\} $;
		\STATE \textbf{if} $|\mathcal{H}^b|=I$, break loop;
		\ENDFOR  \label{step:Kend}
		\STATE \textit{train} and \textit{evaluate} all $\bm{A}^{(b)}$'s in $\mathcal{H}^b$;
		\STATE add $\bm{A}^{(b)}$'s to $\mathcal{T}^{b}$ and record the performance in $\mathcal Y^b$;
		\STATE update predictor $\mathcal P$ with records in $(\mathcal T^b, \mathcal Y^b)$.
		\REPEAT
		\STATE $b:= b+1$;  \label{step:greedyb}
		\STATE $\mathcal H^b=\emptyset$;
		\REPEAT	 \label{step:gen-start}
		\STATE randomly select a top-$I$ structure ${\bm{A}^{b-1}} \in \mathcal T^{b-1}$; \label{step:top-struct}
		\STATE randomly generate $i_b, j_b, k_b \in \{1, \dots, K \}$, $s_b\in \{ \pm 1 \}$, 
		and form $\bm{A}^{(b)}$ with
		%		$		f^b \leftarrow f^{b-2} 
		%		+ s_1 \left\langle \bm{h}_{i_1}, \bm{r}_{j_1}, \bm{t}_{k_1}  \right\rangle 
		%		+ s_2 \left\langle \bm{h}_{i_2}, \bm{r}_{j_2}, \bm{t}_{k_2} \right\rangle$;
		$f_{\bm{A}^{(b)}} \leftarrow f_{\bm{A}^{b-1}} 
		+ s_b \left\langle \bm{h}_{i_b}, \bm{r}_{k_b}, \bm{t}_{j_b}  \right\rangle$;
		\label{step:gen}
		\STATE \textbf{if} $\mathcal Q(\bm{A}^{(b)}, \mathcal H^b\cup\mathcal T^b)$ from Algorithm~\ref{alg:filter} is true \\ \textbf{then} $\mathcal{H}^b \leftarrow \mathcal{H}^b \cup \{ \bm{A}^{(b)}\} $;
		%		\STATE generate a candidate $f^{b}$ based on top $K_1$ scoring functions in $\mathcal T^{b-2}$.
		%		\label{step:gen}
		%		\STATE use the \textit{filter} $\mathcal{Q}$ to judge whether $f^{b}$ can be added in $\mathcal H^{b}$ 
		\label{step:filter}
		\UNTIL{$\left|\mathcal H^{b}\right|=N$}  \label{step:gen-end}
		
		\STATE select top-$P$ $\bm{A}^{(b)}$'s in $\mathcal H^{b}$ based on the \textit{predictor} $\mathcal P$;
		\label{step:predict}    
		
		\STATE \textit{train} embeddings and \textit{evaluate} the performance of $\!\bm{A}^{(b)}$'s;
		\label{step:train}
		
		\STATE add $\bm{A}^{(b)}$'s in $\mathcal{T}^{b}$ and record the performance in $\mathcal Y^b$;
		\label{step:record}
		%			\STATE $\mathcal{C}^{b_i} \leftarrow$ keep $f^{b_i}$ with top-$K_1$ performance in $\mathcal T^{b_i}$;
		\STATE update the predictor  (the following commented out)
		$\mathcal{P}$ with ($\mathcal{T}={\mathcal{T}}^{b_0}\cup \dots \cup{\mathcal {T}}^b$, ${\mathcal{Y}}={\mathcal {Y}}^{b_0} \cup \dots \cup{\mathcal {Y}}^b$);
		\label{step:update} 
		\UNTIL{budget is exhausted or $b=K^2$;}
		\STATE select the top-$I$ structures in $\mathcal{T}$ based on performance in $\mathcal Y$ to form the set $\mathcal I$.
		\RETURN $\mathcal I$.
		\label{step:return}
	\end{algorithmic}
\end{algorithm}

\subsection{Progressive Algorithm}
\label{ssec:progressive}

To explore the search space $\mathcal A$ in \eqref{eq:space},
the simplest approach 
is by direct sampling.
However,
it can be expensive as the
space
is large.
Note from \eqref{eq:uni1} that the 
complexity of $f_{\bm{A}}({h},{r},{t})$ is controlled by
the number of nonzero elements in $\bm{A}$. 
Inspired by~\cite{liu2018progressive},
we 
propose in this section a progressive algorithm
that starts with
$\bm{A}$'s having only a few
nonzero elements, 
and
then gradually expands the search space
by allowing more nonzeros.

The procedure,
which is called AutoBLM,
is in Algorithm~\ref{alg:greedy}.
Let $\bm{A}^{(b)}$ be an
$\bm{A}$ with
$b$ nonzero elements, and the corresponding BLM be $f_{\bm{A}^{(b)}}$.
In step~\ref{step:greedyinit}, we initialize $b$ to some $b_0$ and create an empty candidate set $\mathcal H^b$.
As $\bm{A}$'s with fewer than $K$ nonzero elements are degenerate
(Proposition~\ref{pr:degenerate}),
we
use $b_0=K$. 
We first sample 
positions of  
$b_0$ 
nonzero elements, and
then
randomly assign them
values in $\{\pm 1, \pm 2, \dots, \pm K\}$.
The other entries
are set 
to zero.

Steps~\ref{step:Kstart}-\ref{step:Kend}
	filter away degenerate and equivalent structures.
The number of nonzero elements $b$ is then increased by $1$
(step~\ref{step:greedyb}).
For each such $b$,
steps~\ref{step:gen-start}-\ref{step:gen-end}
greedily  select a top-performing 
structure (evaluated based on the mean reciprocal rank (MRR)~\cite{wang2017knowledge} 
performance on $\mathcal{S}_{\text{val}}$) in $\mathcal T^{b-1}$,
and
generate  $N$ candidates.
All the candidates 
are checked by the filter
$\mathcal Q$ (Section~\ref{ssec:filter}) 
to avoid degenerate 
or equivalent solutions.
Next, the
predictor $\mathcal P$ in Section~\ref{sssec:predictor}
selects
the top-$P$ $\bm{A}^{(b)}$'s,
which are then trained and evaluated in step~\ref{step:train}.
The training data for $\mathcal P$ is collected with the recorded structures 
and performance
in $(\mathcal{T}, \mathcal Y)$ at step~\ref{step:update}.
Finally,
the top-$I$ structures in $\mathcal T$ evaluated by the corresponding performance
in $\mathcal Y$ are returned.

\begin{algorithm}[t]
	\caption{Evolutionary search algorithm. (AutoBLM+).}
	\label{alg:evolution}
	\small
	\begin{algorithmic}[1]
		\REQUIRE 
		$I$: number of top structures;
		$N$: number of generated structures;
		$P$: number of structures selected by $\mathcal P$;
		$b_0$: number of nonzero elements in initial structures;
		filter $\mathcal Q$, and performance predictor $\mathcal P$.
		\STATE \textbf{initialization:} $\mathcal{I} = \emptyset$; \label{step:evol:init}
		\FOR{each $\bm{A}\in \{\bm{A}^{(b_0)} \}$}
		\STATE \textbf{if} $\mathcal Q(\bm{A}, \mathcal I)$ from Algorithm~\ref{alg:filter} is true  \textbf{then} $\mathcal{I} \leftarrow \mathcal{I} \cup \left\lbrace \bm{A}\right\rbrace $;
		\STATE \textbf{if} $|\mathcal{I}|=I$, break loop;
		\ENDFOR 
		\STATE \textit{train} and \textit{evaluate} all $\bm{A}$'s in $\mathcal{I}$; \label{step:evol:endinit}
		\STATE add $\bm{A}$'s to $\mathcal{T}$ and record the performance in $\mathcal Y$;
		\REPEAT	\label{step:evo-init}
		
		\STATE update predictor $\mathcal P$ with records in $(\mathcal T, \mathcal Y)$.
		\REPEAT
		%				\STATE with probability $p_1$ to do crossover, otherwise do mutation.
		\STATE $\mathcal H=\emptyset$;
		\STATE \textbf{mutation:} sample $\bm{A}\in\mathcal I$ and mutate to $\bm{A}_{\text{new}}$; \textbf{or}
		\STATE \textbf{crossover:} sample $\bm{A}_{(a)}, \bm{A}_{(b)} \!\in\!
		\mathcal I$, and use crossover to generate $\bm{A}_{\text{new}}$;
		\label{step:evo-co}
		%				\ELSE
		\label{step:evo-mu}
		%				\ENDIF
		\STATE \textbf{if} $\mathcal Q(\bm{A}_{\text{new}}, \mathcal H\cup\mathcal T)$ is true by Algorithm~\ref{alg:filter}, 
		\\ \textbf{then} $\mathcal{H} \leftarrow \mathcal{H} \cup \left\lbrace \bm{A}_{\text{new}}\right\rbrace $;
		\label{step:evo-filter}
		\UNTIL{$|\mathcal H| = N$}; \label{step:evo-end}
		\STATE  
		select top-$P$ structures $\bm{A}$ in $\mathcal H$ based on the 
		the \textit{predictor} $\mathcal P$;
		%		based on the predictor $\mathcal P$ to
		\label{step:evo-pred}
		\FOR{
			each top-$P$ structure
			$\bm{A}$ 
		}
		\STATE  \textit{train} embeddings and \textit{evaluate} the performance of ${\bm{A}}$;
		\label{step:evo-train}
		\STATE \textbf{survive:} update $\mathcal I$ with $\bm{A}$ if ${\bm{A}}$ is better than the worst structure in $\mathcal I$;
		\label{step:evo-update}
		\ENDFOR
		\STATE add $\bm{A}$'s in $\mathcal{T}$ and record the performance in $\mathcal Y$;
		\UNTIL{budget is exhausted;}
		\RETURN $\mathcal{I}$.
	\end{algorithmic}
\end{algorithm}

\subsection{Evolutionary Algorithm}
\label{ssec:evolution}

While progressive search 
can be efficient,
it may not fully explore the search space and can lead to 
sub-optimal solutions~\cite{tropp2004greed}.
The progressive search can only generate structures from fewer non-zero elements
to more ones.
Thus, it can not visit and adjust the structures with fewer non-zero elements
when $b$ is increased.
To address these problems, we consider in this section the use of evolutionary
algorithms~\cite{back1996evolutionary}.

The procedure,
which is called AutoBLM+,
is in Algorithm~\ref{alg:evolution}.
As in Algorithm~\ref{alg:greedy},
we start with structures having $b_0=K$ nonzero elements.
Steps~\ref{step:evol:init}-\ref{step:evol:endinit}
initializes a set $\mathcal I$ of $I$ non-degenerate and non-equivalent structures.
The main difference with 
Algorithm~\ref{alg:greedy}
is in steps~\ref{step:evo-init}-\ref{step:evo-end},
in which new structures are generated 
by mutation and crossover.
For a given structure $\bm{A}$,
mutation changes the value of 
each entry to another one in $\{0, \pm 1, \dots, \pm K \}$ 
with 
a small 
probability $p_m$.
For crossover,
given two structures  $\bm{A}_{(a)}$ and $ \bm{A}_{(b)}$,
each entry of the new structure
has equal probabilities
to be selected from 
the corresponding entries in 
$\bm{A}_{(a)}$  or $\bm{A}_{(b)}$.
After mutation or crossover,
we check if the newly generated $\bm{A}_{\text{new}}$ 
has to be filtered out. 
After $N$ structures are collected,
we 
use the performance predictor $\mathcal P$ in Section~\ref{sssec:predictor}
to select the top-$P$ structures.
These are then trained and evaluated for actual performance.
Finally, structures in $\mathcal I$ with performance 
worse than the newly
evaluated ones are replaced
(step~\ref{step:evo-update}).

\section{Experiments}

In this section,
experiments are performed on a number of KG tasks.
Algorithm~\ref{alg:full} shows the general 
procedure for each task. First,
we find a good hyper-parameter setting 
to train and evaluate different structures
(steps~\ref{step:before-for}-\ref{step:before-best}).
Based on the observation that the performance ranking of scoring functions is
consistent across different $d$'s (details are in Appendix~\ref{sssec:transfer}),
we set $d$ to a smaller value ($64$)
to reduce model training time.
The search algorithm 
is then used
to obtain the set $\mathcal I$ of top-$I$
structures (step~\ref{step:search}).
Finally,
the hyper-parameters 
are fine-tuned
with a larger $d$, and 
the best structure
selected 
(steps~\ref{step:after}-\ref{step:after-endfor}).
Experiments are run on a RTX 2080Ti GPU with 11GB memory.  All algorithms are implemented in python~\cite{paszke2017automatic}.

\begin{algorithm}[ht]
	\caption{Experimental procedure for each KG task.
Here, $H\!P$ denotes
the hyper-parameters 
$\{ \eta, \lambda, m, d \}$.}
	\label{alg:full}
	\small
	\begin{algorithmic}[1]
		\STATE // \textbf{stage 1}: \textit{configure hyper-parameters for scoring function search}.
			
		\FOR{$i = 1, \dots, 10$} \label{step:before-for}
		\STATE fix $d = 64$, randomly select $\eta_i $$\in $$[0$$, 1]$, $\lambda_i $$\in $$[10^{-5}$$, 10^{-1}]$ and $m_i \in \{256$$, 512$$, 1024\}$;
		
		\STATE train \textit{SimplE} with $H\!P_i = \{ \eta_i, \lambda_i, m_i, d \}$, and evaluate the validation MRR; 
		\label{step:before-run}
		
		%\STATE $\text{iter}_1 := \text{iter}_1 +1$;
		\ENDFOR
		%fix $d = 64$ and
		%randomly select 10 sets of hyper-parameters from $H\!P$ with $d=64$ for model training; 
		
		\STATE  
		select the best hyper-parameter setting $\bar{H\!P} \in \{H\!P_i\}_{i = 1}^{10}$; \label{step:before-best} 
	 
		\STATE // \textbf{stage 2}: \textit{search scoring function} 
		\STATE 
		using hyper-parameter setting $\bar{H\!P}$,
		obtain the set $\mathcal I$ of top-$I$ structures from Algorithm~\ref{alg:greedy} or Algorithm~\ref{alg:evolution};
		\label{step:search}
		
		\STATE
		// \textbf{stage 3}: \textit{fine-tune the obtained scoring function}
		\FOR{$j = 1, \dots, 50$
		($j = 1, \dots, 10$
for YAGO3-10)}
		\label{step:after}
			\STATE randomly select a structure $\bm{A}_j \in \mathcal{I}$;
			\STATE randomly select $\eta_j \in [0, 1]$, $\lambda_j \in [10^{-5}, 10^{-1}]$, 
			$m_j \in \{256, 512, 1024\}$,
			and $d_j \in \{256, 512, 1024, 2048\}$;
			
			\STATE train the KG learning model with 
			structure $\bm{A}_j$ and hyper-parameter setting ${H\!P}_j = \{ \eta_j, \lambda_j, m_j, d_j \}$

		\ENDFOR \label{step:after-endfor}
		
		\STATE 
		%\textbf{optimal configure:} 
		select the best structure $\{ \bm{A}^*, {H\!P}^* \}
		\in \{ \bm{A}_j, {H\!P}_j \}_{j = 1}^{50}$. 
		 
		% along with the hyper-parameter setting  $HP_2$ based on the validation MRR performance; \label{step:optimal-config}
	\end{algorithmic}
\end{algorithm}

\subsection{Knowledge Graph (KG) Completion}
\label{ssec:KGC}

In this section, we perform experiments 
on KG completion as introduced in Section~\ref{sec:app1}.  
we use the full multi-class log-loss~\cite{lacroix2018canonical},
which is more robust and
has better performance than negative sampling \cite{lacroix2018canonical,zhang2020autosf}.

\subsubsection{Setup}
\label{sssec:kgcsetup}

\parabegin{Datasets.}
Experiments are performed on
the following popular benchmark 
datasets:
WN18, FB15k, WN18RR, FB15k237, YAGO3-10,
ogbl-biokg and ogbl-wikikg2
(Table~\ref{tab:dataset}).
WN18 and
FB15k 
are introduced in~\cite{bordes2013translating}.
WN18 is
a subset of the lexical database WordNet~\cite{miller1995wordnet},
while FB15k is
a subset of the Freebase KG~\cite{bollacker2008freebase} for human knowledge.
WN18RR~\cite{dettmers2017convolutional} and FB15k237~\cite{toutanova2015observed}
are obtained
by removing the near-duplicates and inverse-duplicate relations
from
WN18 and FB15k.
YAGO3-10 is
created by~\cite{dettmers2017convolutional}, and
is a subset of 
the semantic KG
YAGO~\cite{suchanek2007yago},
which
unifies WordNet and Wikipedia. 
The
ogbl-biokg and ogbl-wikikg2 datasets are from
the open graph benchmark (OGB) \cite{hu2020open}, which
contains
realistic and large-scale datasets for graph learning.
The ogbl-biokg dataset is a biological KG
describing interactions among proteins, drugs, side effects and functions.
The ogbl-wikikg2 dataset is extracted from the Wikidata knowledge base \cite{vrandevcic2014wikidata}
describing relations among entities in Wikipedia.

\begin{table}[ht]
	\centering
	\vspace{-7px}
	\caption{Statistics of the KG completion datasets.}
	\vspace{-10px}
	\label{tab:dataset}
	\setlength\tabcolsep{4pt}
	\begin{tabular}{c|ccccc}
		\toprule
		& & &  \multicolumn{3}{c}{number of samples}\\ 
		data set                 & \#entity & \#relation &  training & validation &
		testing  \\ \midrule
		WN18~\cite{bordes2013translating}     &  40,943  &     18     &  141,442  &  5,000  & 5,000    \\
		FB15k~\cite{bordes2013translating}    &  14,951  &   1,345    &  484,142  & 50,000  & 59,071 \\
		WN18RR~\cite{dettmers2017convolutional} &  40,943  &     11     &  86,835   &  3,034  & 3,134   \\
		FB15k237~\cite{toutanova2015observed}   &  14,541  &    237     &  272,115  & 17,535  & 20,466  \\
		YAGO3-10~\cite{mahdisoltani2013yago3}   & 123,188  &     37     & 1,079,040 &  5,000  & 5,000    \\ 
		\hline
		ogbl-biokg & 94k & 51 & 4,763k  & 163k  & 163k \\
		ogbl-wikikg2 & 2500k  & 535  & 16,109k &  429k  & 598k \\ 
		\bottomrule	
	\end{tabular}
\vspace{-7px}
\end{table}

\begin{table*}[ht]
	\caption{Testing performance of MRR, H@1 and H@10 on KG completion.
		The best model is highlighted in bold and the second best is underlined.
		%We use our implementations of DistMult, ComplEx, Analogy and SimplE, 
		%while the
		%other baseline results are from the corresponding cited papers
		``--'' means that results are not reported in those papers or their code on that
		data/metric is not available.
		CompGCN  uses the entire KG in each iteration and  so runs out of memory  on the
		larger data sets of WN18, FB15k and YAGO3-10.}
	\label{tb:comparison}
	\centering
	\setlength\tabcolsep{4pt}
	\vspace{-10px}
	\begin{tabular}{cc|ccc|ccc|ccc|ccc|ccc}
		\toprule
		& &                \multicolumn{3}{c|}{WN18}                &               \multicolumn{3}{c|}{FB15k}                &               \multicolumn{3}{c|}{WN18RR}               &              \multicolumn{3}{c|}{FB15k237}              &              \multicolumn{3}{c}{YAGO3-10}               \\ 
		\multicolumn{2}{c|}{model}                  &        MRR        &       H@1        &       H@10       &        MRR        &       H@1        &       H@10       &        MRR        &       H@1        &       H@10       &        MRR        &       H@1        &       H@10       &        MRR        &       H@1        &       H@10       \\ \midrule
		(TDM)  
		&    TransH     &       0.521       &       ---        &       94.5       &       0.452       &       ---        &       76.6       &       0.186       &       ---        &       45.1       &       0.233       &       ---        &       40.1       &        ---        &       ---        &       ---        \\
		&       RotatE       &       0.949       &       94.4       &       95.9       &       0.797       &       74.6       &       88.4       & {0.476} &       42.8       &  \underline{57.1}   &       0.338       &       24.1       &       53.3       &        0.488        &       39.6        &       66.3        \\ 
		&	PairE   &   --- & --- & --- &  0.811  &  76.5   &  89.6  &  --- &   --- &    ---  &  0.351  &   25.6  &   54.4   &   ---   &   ---   &    --- \\
		\midrule
		(NNM)  
		& ConvE   &       0.942       &       93.5       &      {95.5}      &       0.745       &       67.0       &       87.3       &       0.46        &       39.        &      {48.}       &      {0.316}      &       23.9       &       49.1       &       0.52        &       45.        &      {66.}       \\
		&       RSN        &       0.94        &       92.2       &       95.3       &        ---        &       ---        &       ---        &        ---        &       ---        &       ---        &       0.28        &       20.2       &       45.3       &        ---        &       ---        &       ---        \\ 
		&       Interstellar        &       ---        &       ---       &       ---       &        ---        &       ---        &       ---        &        0.48        &       44.0        &       54.8       &       0.32        &       23.3       &       50.8       &        0.51        &      42.4        &       66.4        \\ 
		& CompGCN   &  --- & --- & --- & --- & --- & --- & 0.479 & 44.3 & 54.6 & 0.355 & 26.4 & 53.5 & --- & --- & --- \\ \midrule
		(BLM)  &   TuckER     &  \textbf{0.953}   &  \textbf{94.9}   &       95.8       &       0.795       &       74.1       &       89.2       &       0.470       & {44.3} &       52.6       & {0.358} & {26.6} &       54.4       &        ---        &       ---        &       ---        \\
		&                DistMult                 &       0.821       &       71.7       &       95.2       &       0.775       &       71.4       &       87.2       &       0.443       &       40.4       &       50.7       &      {0.352}      &       25.9       &       54.6       &       0.552       &       47.1       &       68.9       \\
		&                SimplE/CP                &       0.950       &       94.5       & \underline{95.9} &       0.826       & {79.4} &      {90.1}      &      {0.462}      &       42.4       &       55.1       &      {0.350}      &       26.0       &       54.4       &       0.565       & {49.1} &      {71.0}      \\ 
		&                HolE/ComplEx                 &       0.951       &       94.5       &       95.7       &  {0.831} &       79.6       & {90.5} &      {0.471}      &       43.0       &       55.1       &       0.345       &       25.3       &       54.1       & {0.563} & {49.0} &       70.7       \\
		&                 Analogy                 &       0.950       &       94.6       &       95.7       &       0.816       &       78.0       &   {89.8} &      {0.467}      &       42.9       &      {55.4}      &       0.348       &       25.6       & {54.7} &       0.557       &       48.5       & {70.4} \\
		&    QuatE      &       0.950       &       94.5       &       95.9       &       0.782       &       71.1       &       90.0       & {0.488} &       43.8       &  \textbf{58.2}   &       0.348       &       24.8       &       55.0       &        0.556       &       47.4        &       70.4        \\  \midrule
		\multicolumn{2}{c|}{AutoBLM}           & \underline{0.952} & \underline{94.7} &  \textbf{96.1}   &  \underline{0.853}   &  \underline{82.1}   &  \underline{91.0}   &  \underline{0.490}   &  \underline{45.1}   & {56.7} &  \underline{0.360}   &  \underline{26.7}   &  \underline{55.2}   &  \underline{0.571}   &  \underline{50.1}   &  \textbf{71.5}   \\ 
		\multicolumn{2}{c|}{AutoBLM+}       &     \underline{0.952}   &    \underline{94.7}     &    \textbf{96.1}     &    \textbf{0.861}  &  \textbf{83.2}  &   \textbf{91.3}   &  \textbf{0.492}       & 	\textbf{45.2}	   &  	{56.7}	   &  	\textbf{0.364}	   &    \textbf{27.0}	 &   	\textbf{55.3}   &  \textbf{0.577}  &   \textbf{50.2}    &  \textbf{71.5}  \\ \bottomrule
	\end{tabular}
\vspace{-5px}
\end{table*}

\begin{figure*}[ht]
	\centering
	%	\subfigure[WN18.]
	{\includegraphics[height=3.5cm]{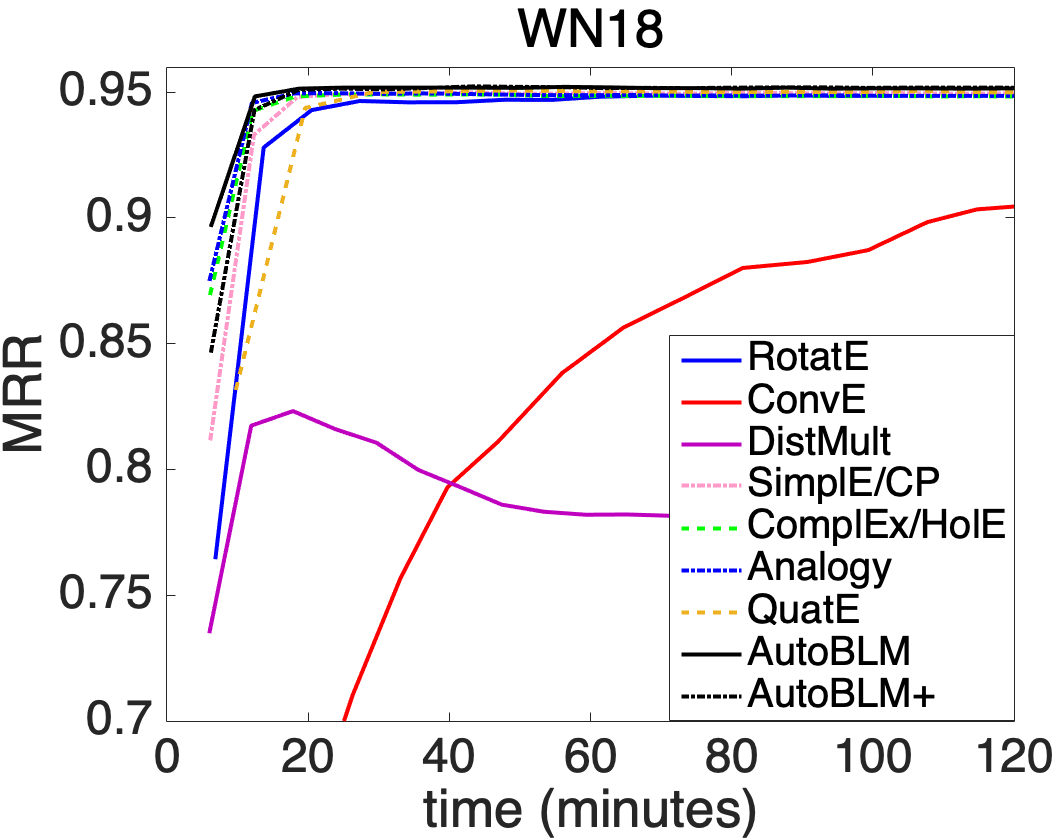}}\qquad
	%	\subfigure[FB15k.]
	{\includegraphics[height=3.5cm]{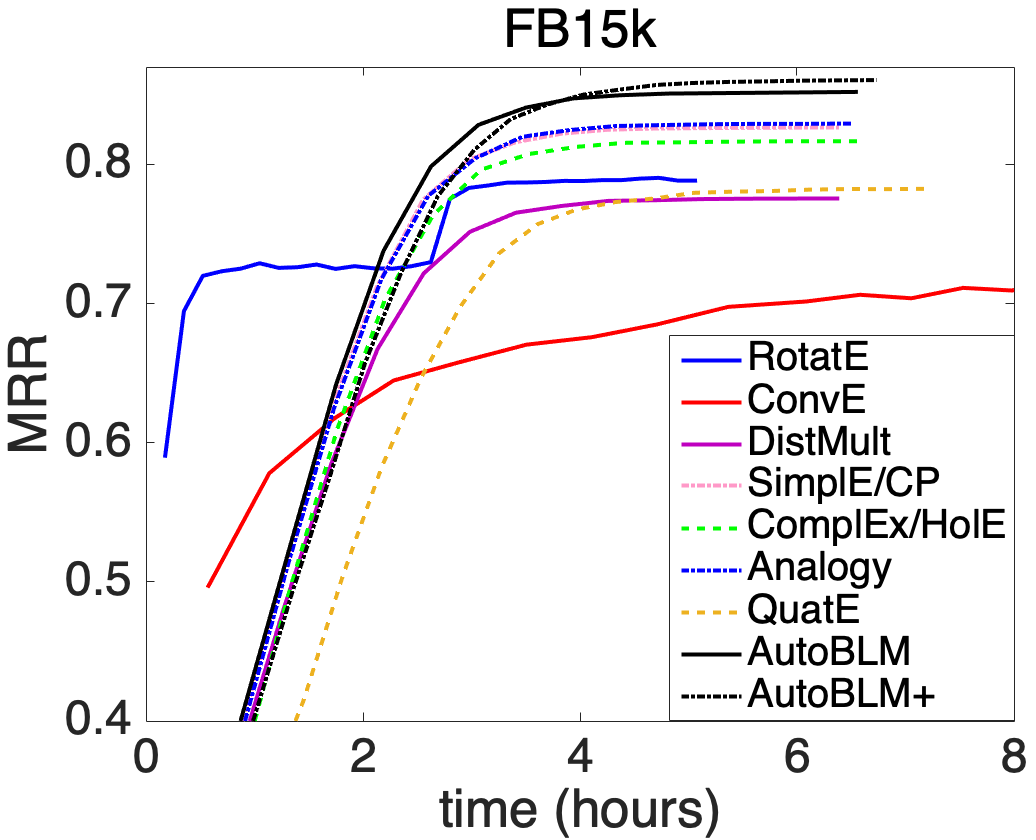}}\qquad
	%	\subfigure[WN18RR.]
	{\includegraphics[height=3.5cm]{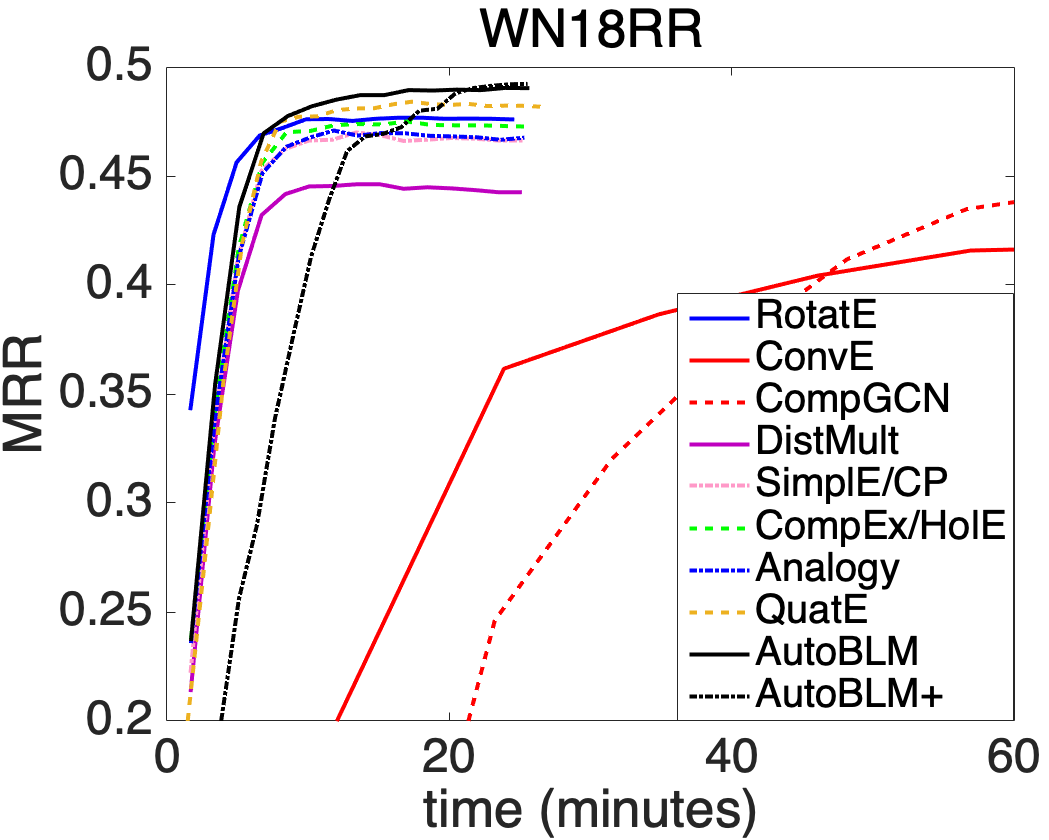}}
	
	%	\subfigure[FB15k237.]
	{\includegraphics[height=3.5cm]{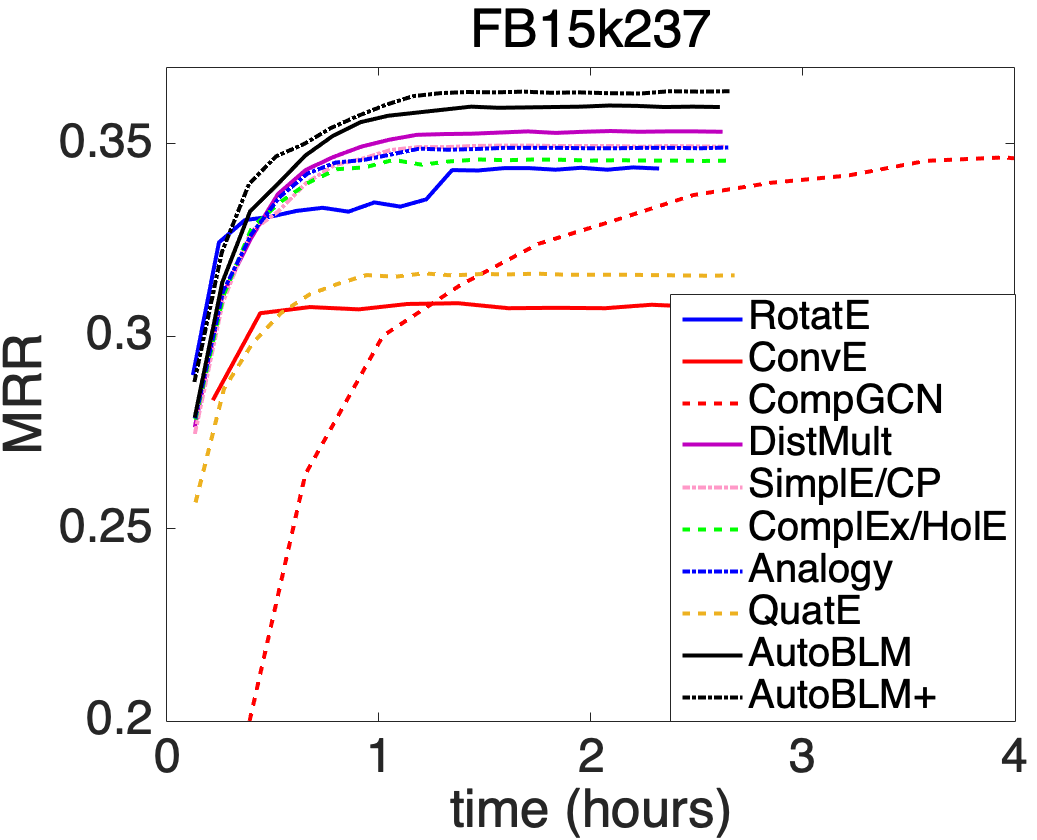}}\qquad
	%	\subfigure[YAGO3-10.]
	{\includegraphics[height=3.5cm]{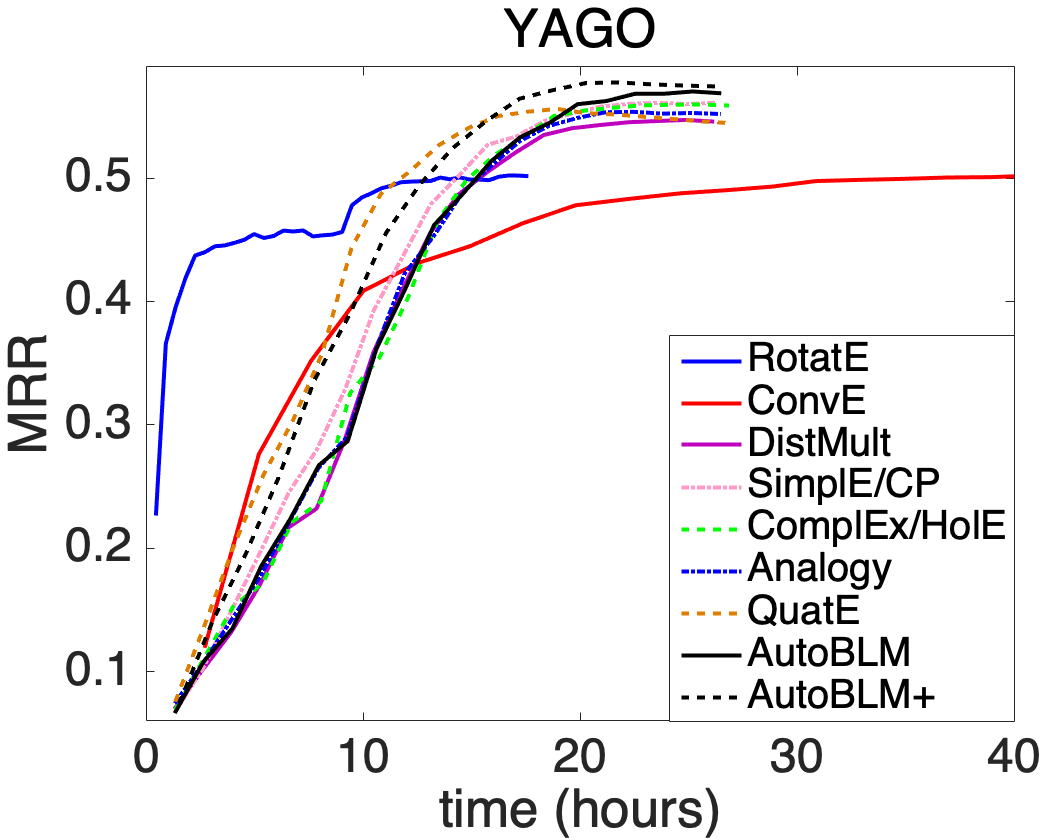}}
	
	\vspace{-10px}
	\caption{Convergence of the testing MRR versus running time 
		on the KG completion task.}
	\vspace{-10px}
	\label{fig:curve}
\end{figure*}

\parabegin{Baselines.}
For AutoBLM and AutoBLM+, 
we 
select the structure for evaluation 
from the set returned by 
Algorithm~\ref{alg:greedy}
or \ref{alg:evolution}
based on the 
MRR performance on
the validation set.

For WN18, FB15k, WN18RR, FB15k237, YAGO3-10,
AutoBLM and AutoBLM+
are compared
with the following popular
 human-designed 
KG embedding models\footnote{Obtained from
\url{https://github.com/thunlp/OpenKE} and \url{https://github.com/Sujit-O/pykg2vec}}:
(i) TDM, including
TransH~\cite{wang2014knowledge}, RotatE~\cite{sun2019rotate}
and PairE~\cite{chao2021pairre};
(ii)
NNM, including
ConvE~\cite{dettmers2017convolutional},
 RSN~\cite{guo2019learning}
 and CompGCN~\cite{vashishth2019composition};
(iii)
BLM, including
TuckER~\cite{balavzevic2019tucker},
Quat~\cite{zhang2019quaternion},
DistMult~\cite{yang2014embedding},
ComplEx~\cite{trouillon2017knowledge},
HolE~\cite{nickel2016holographic},
Analogy~\cite{liu2017analogical}
SimplE~\cite{kazemi2018simple},
and
CP~\cite{lacroix2018canonical}.
We do not
compare
with NASE~\cite{kou2020nase} as
its code is not publicly available.

For ogbl-biokg and ogbl-wikikg2 \cite{hu2020open}, we 
compare with
the models reported in the
%are tuned and compared in a unified framework.
OGB leaderboard\footnote{\url{https://ogb.stanford.edu/docs/leader_linkprop/}},
namely,
TransE~\cite{bordes2013translating},
		RotatE,
		PairE,
		DistMult, and
		ComplEx.

%As studied in Section~\ref{ssec:KGC}, AutoBLM+ is the state-of-the-art method in KG completion.  To further demonstrate its effectiveness, we evaluate AutoBLM+ on two large-scaled datasets, i.e., ogbl-biokg and ogbl-wikikg2.

\parabegin{Performance Measures.}
The learned $f_{\bm{A}}(h,r,t)$ 
is evaluated in the context
of link prediction.  
Following 
~\cite{yang2014embedding,trouillon2017knowledge,liu2017analogical,kazemi2018simple,dettmers2017convolutional,wang2017knowledge},
for each triple $(h,r,t)$,
we first take $(?,r,t)$ as the query and obtain the filtered rank on the head 
\begin{equation}
\text{rank}_h = \left|\left\{e\in\mathcal E: 
\begin{matrix}
\big(f(e,r,t)\geq f(h,r,t)\big)\wedge \\
\big((e,r,t)\notin \mathcal S_{\text{tra}}\cup\mathcal{S}_{\text{val}}\cup\mathcal S_{\text{tst}}\big)
\end{matrix}
\right\} \right|+1,
\label{eq:headrank}
\end{equation}
where
$\mathcal S_{\text{tra}},
\mathcal{S}_{\text{val}},
\mathcal S_{\text{tst}}$ are the training,
validation, and
test 
sets, respectively.
Next
we take  $(h,r,?)$ 
as the query 
and obtain the filtered rank on the tail 
\begin{equation}
\text{rank}_t
\! = \! \left|\left\{e\in\mathcal E: 
\begin{matrix}
\big(f(h,r,e)\geq f(h,r,t)\big)\wedge \\
\big((h,r,e)\notin \mathcal S_{\text{tra}} \cup \mathcal{S}_{\text{val}}\cup\mathcal S_{\text{tst}}\big)
\end{matrix}
\right\} \right|
\! + \! 1.
\label{eq:tailrank}
\end{equation}
The following metrics
are computed from both the head and tail ranks on all triples:
(i) Mean reciprocal ranking (MRR):
%on the validation set (the same for testing set):
\[\text{MRR} = \frac{1}{2|\mathcal S|}\sum_{(h,r,t)\in \mathcal S}\!\!\Big(\frac{1}{\text{rank}_h}+\frac{1}{\text{rank}_t}\Big);\]
and
(ii) 
H@$k$: 
ratio of ranks no larger than $k$, i.e., 
\[\text{H}@k = \frac{1}{2|\mathcal S|}\sum_{(h,r,t)\in \mathcal S}\!\!\Big(\mathbb I(\text{rank}_h\leq k)+\mathbb I(\text{rank}_t\leq k)\Big),\]
where $\mathbb I(a)=1$ if $a$ is true, otherwise $0$. 
The larger the MRR or H@$k$, the better is the embedding.
Other metrics 
for the completion task 
\cite{wang2019evaluating,tabacof2019probability}
can also be adopted here.

For ogbl-biokg and ogbl-wikikg2 \cite{hu2020open},
we only use the MRR as
the H@$k$ is not reported by the baselines in the OGB leaderboard.

\begin{figure*}[ht]
	\centering
	%\subfigure[WN18.]
%	{\includegraphics[width=0.31\columnwidth]{figure/wn18-2}}
%	\quad
%	%	\ 
%	%\subfigure[FB15k.]
%	{\includegraphics[width=0.31\columnwidth]{figure/fb15k-2}}
%	%	\ 
%	\quad
%	%\subfigure[WN18RR.]
%	{\includegraphics[width=0.31\columnwidth]{figure/wn18rr-2}}
%	\quad
%	%\subfigure[FB15k237.]
%	{\includegraphics[width=0.31\columnwidth]{figure/fb15k237-2}}
%	\quad
%	%\subfigure[YAGO3-10.]
%	{\includegraphics[width=0.31\columnwidth]{figure/yago-2}}
%	\\
	\subfigure[WN18.]
	{\includegraphics[width=0.31\columnwidth]{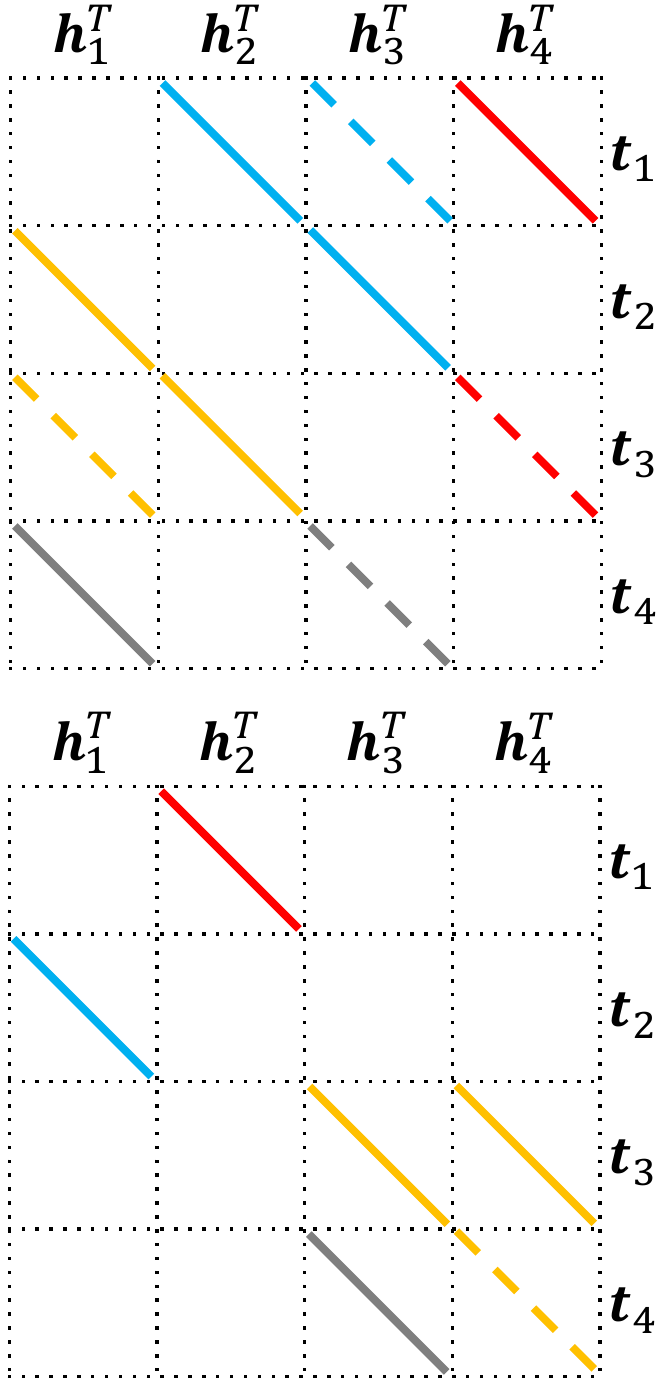}
	\label{fig:wn18}}
	\qquad
	%	\ 
	\subfigure[FB15k.]
	{\includegraphics[width=0.31\columnwidth]{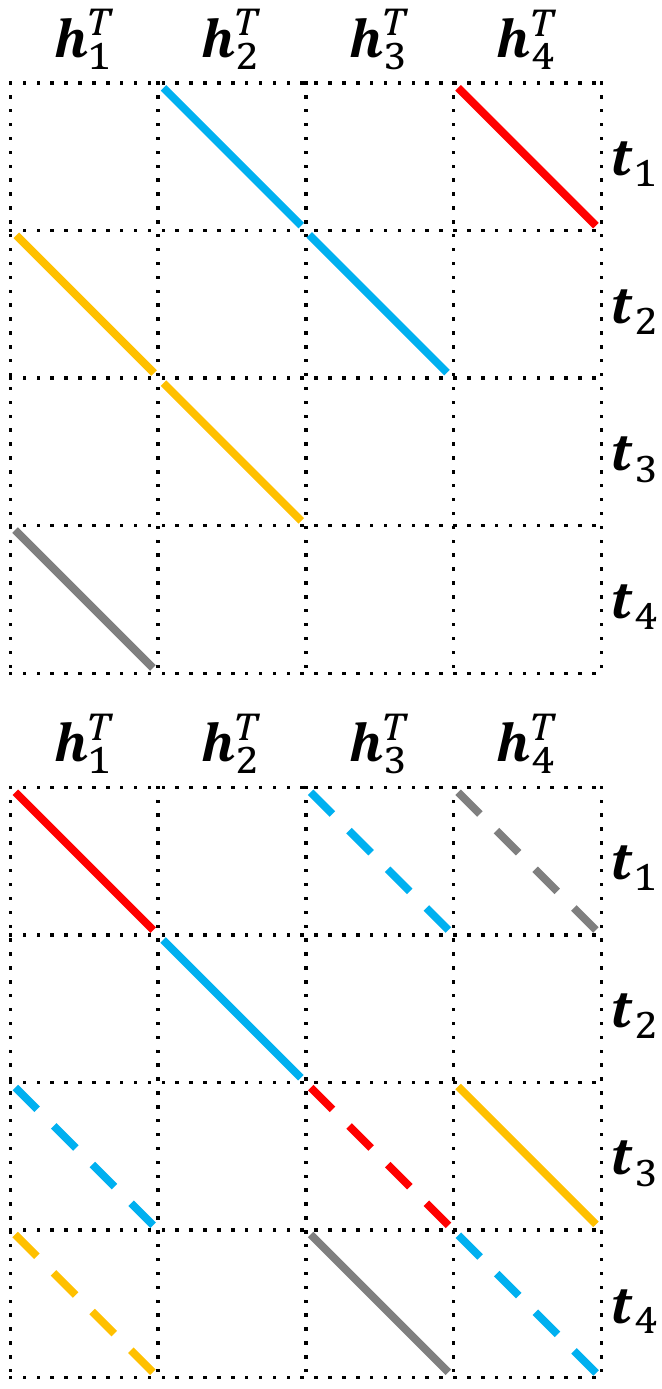}
	\label{fig:fb15k}}
	%	\ 
	\qquad
	\subfigure[WN18RR. \label{fig:model1}]
	{\includegraphics[width=0.31\columnwidth]{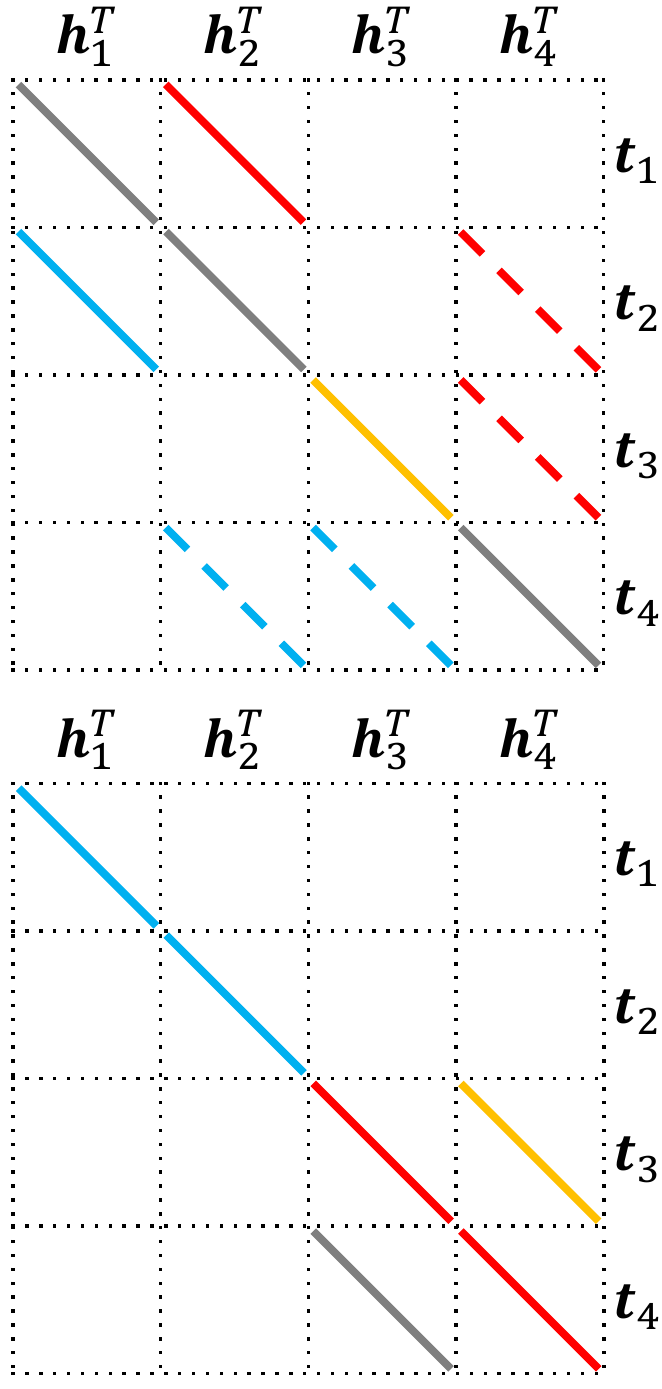}
	\label{fig:wn18rr}}
	\qquad
	\subfigure[FB15k237.]
	{\includegraphics[width=0.31\columnwidth]{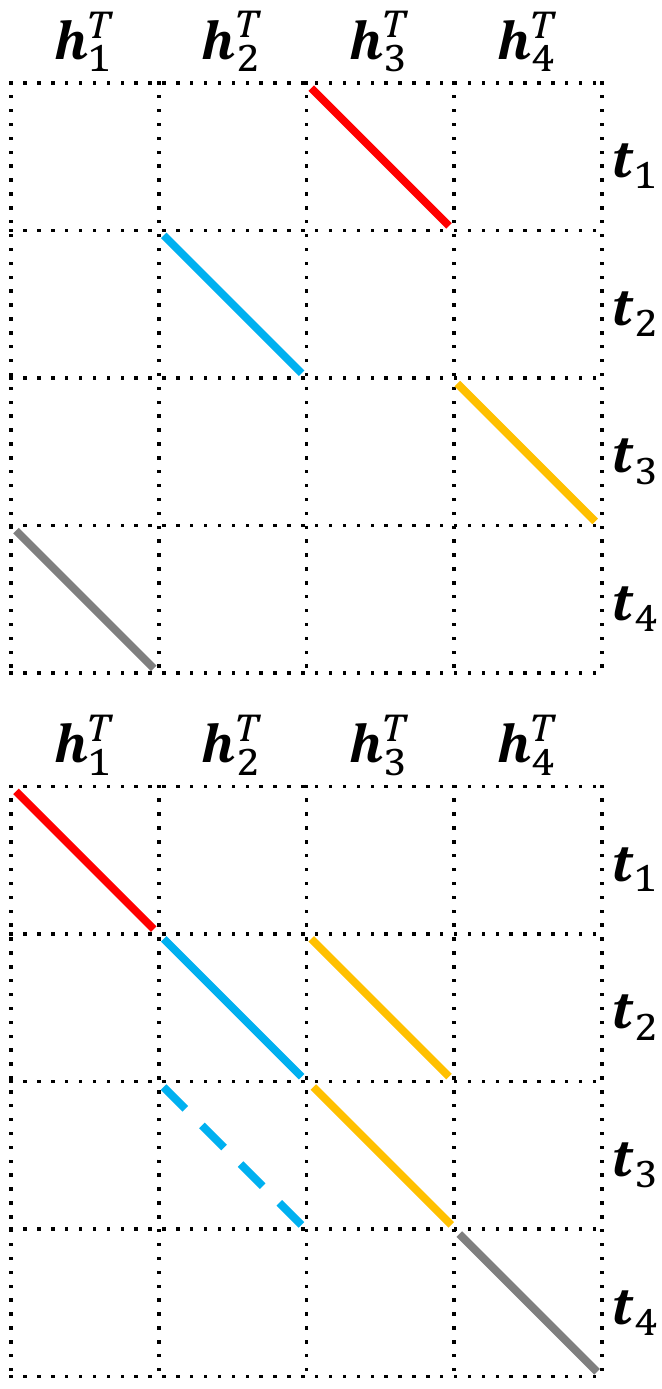}
	\label{fig:fb15k237}}
	\qquad
	\subfigure[YAGO3-10. \label{fig:model2}]
	{\includegraphics[width=0.31\columnwidth]{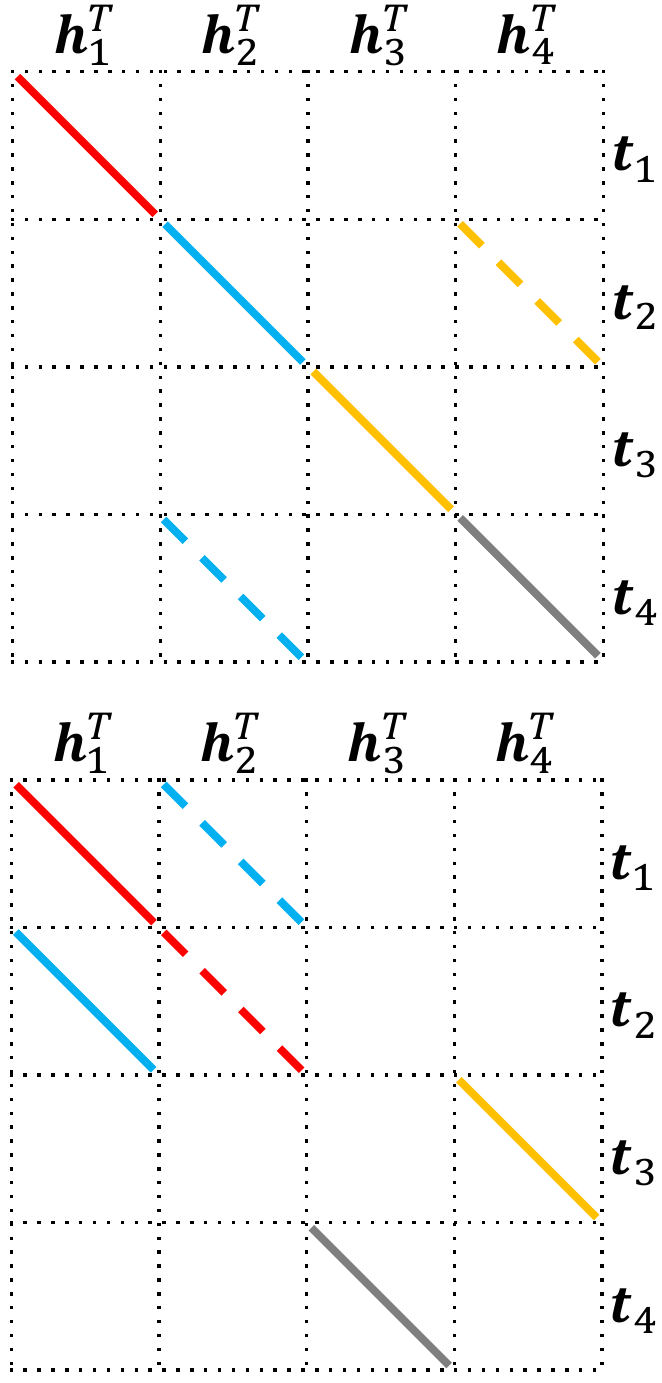}
	\label{fig:yago}}
	
	\vspace{-10px}
	\caption{Graphical illustration of the BLMs obtained by 
		AutoBLM (top) and
		AutoBLM+ (bottom) on the KG
		completion task (Section~\ref{exp:kgc:performance}).
		Different colors correspond to 
		different parts of $[\bm r_1$(red), $\bm r_2$(blue), $\bm r_3$(yellow), $\bm r_4$(gray)$]$.
		Solid lines mean positive values, while 
		dashed lines mean negative values. 
		The empty parts have value zero.}
	\label{fig:searchedsf}
	\vspace{-10px}
\end{figure*}

\parabegin{Hyper-parameters.}
The search algorithms 
have the following hyper-parameters:
(i) $N$: number of candidates generated after filtering;
(ii) $P$: number of scoring functions selected by the predictor; 
(iii) $I$: number of top structures selected in 
Algorithm~\ref{alg:greedy}
(step~\ref{step:top-struct}),
or the
number of structures survived in $\mathcal I$ in Algorithm~\ref{alg:evolution};
and
(iv) $b_0$: number of nonzero elements in the initial set.
Unless otherwise specified,
we use 
$N=128$, 
$P=8$, 
$I=8$
and $b_0=K$.
For the evolutionary algorithm,
the mutation and crossover operations are selected with
equal probabilities. 
When mutation is selected,
the value of each entry has a mutation probability of $p_m=2/K^2$.
A budget is used to terminate the algorithm. This is set to 256 structures
on WN18, FB15k, WN18RR, FB15k-237,
128 on YAGO3-10,
64 on ogbl-biokg,
and 32 on ogbl-wikikg2.

We follow~\cite{lacroix2018canonical,trouillon2017knowledge} 
to use Adagrad~\cite{duchi2011adaptive} as optimizer.
The Adagrad
hyper-parameters 
are selected
from the following
ranges:
learning rate
$\eta$ in $[0, 1]$, 
$\ell_2$-penalty $\lambda$ in $[10^{-5}, 10^{-1}]$,
batch size $m$ in $\{256, 512, 1024\}$,
and dimension $d$ in $\{64, 256, 512, 1024, 2048\}$.

\subsubsection{Results
on WN18, FB15k, WN18RR, FB15k237, YAGO3-10}
\label{exp:kgc:performance}

{\bf Performance.}
Table~\ref{tb:comparison}
shows
the testing results on
WN18, FB15k, WN18RR, FB15k237, and YAGO3-10.
%the five commonly used benchmark dataset in KG completion.
As can be seen,
there is no clear winner
among the baselines.
On the other hand,
AutoBLM
performs consistently well.
It outperforms the baselines on FB15k, WN18RR, FB15k237 and YAGO3-10,
and is the first runner-up on WN18.
AutoBLM+
further improves AutoBLM on FB15k, WN18RR, FB15k237 and YAGO3-10.

\noindent
{\bf 
Learning curves.}
Figure~\ref{fig:curve}
shows the learning curves of 
representative models in each type of scoring functions,
including: RotatE in TDM;
ConvE and CompGCN in NNM; and 
DistMult, SimplE/CP, ComplEx/HolE, Analogy, QuatE 
and the proposed AutoBLM/AutoBLM+
in BLM.
As can be seen,
NNMs 
are much slower and inferior than 
BLMs.
On the other hand, 
AutoBLM+
has 
better performance 
and comparable time as the other BLMs.

\begin{table}[t]
	\centering
	\vspace{-5px}
	\caption{Testing MRR on applying the BLMs obtained
		from a source dataset (row) to a target dataset (column). 
		Bold numbers indicate the best performance each dataset for the models searched 
		by AutoBLM and AutoBLM+ respectively.}
	\label{tb:sf-dependent}
	\setlength\tabcolsep{1.5pt}
	\vspace{-10px}
	\begin{tabular}{cc|ccccc}
		\toprule
		& &	WN18	& 	FB15k	& 	WN18RR	&	FB15k237	& 	YAGO3-10		\\
		\midrule
		\multirow{5}{*}{AutoBLM} & WN18		& 	\textbf{0.952}			& 0.841	&	0.473	&	0.349	&	0.561	\\
		& FB15k		& 	0.950	&	\textbf{0.853}   &	0.470	&	0.350	&	0.563	\\
		& WN18RR	& 	0.951	&	 0.833	&	\textbf{0.490}	&	0.345	&	0.568	\\
		& FB15k237	& 	0.894	&	0.781	&	0.462	&	 \textbf{0.360}	&	0.565	\\
		& YAGO3-10	& 	0.885	&	0.835	&	0.466	& 0.352    &	\textbf{0.571}	\\
		\midrule
		\multirow{5}{*}{AutoBLM+} & WN18		& 	\textbf{0.952}	 	&	0.848	&	0.482	&	0.350	&	0.564	\\
		& FB15k		& 	0.951	&	\textbf{0.861 }  &		0.479	&	0.352	&	0.563	\\
		& WN18RR	& 		0.947	&	0.841	&	\textbf{0.492}	&	0.347	&	0.551	\\
		& FB15k237	& 	0.860	&	0.821	&	0.463	&	 \textbf{0.364}	&	0.546	\\
		& YAGO3-10	& 	0.951	&	0.833	&	0.469	&	0.345	&	\textbf{0.577}	\\
		\bottomrule
	\end{tabular}	
	\vspace{-8px}
\end{table}

\noindent
{\bf Data-dependent BLM structure.}
Figure~\ref{fig:searchedsf} 
shows the BLMs
obtained by 
AutoBLM and AutoBLM+. As can be seen,
they are different from 
the human-designed BLMs in Figure~\ref{fig:graphsf}
and are also different from each other.
To demonstrate that these data-dependent structures also have different accuracies on the same dataset,
we take the 
BLM obtained by AutoBLM (or AutoBLM+) on a source dataset and then 
evaluate it
on a different target dataset.
%(WN18, FB15k, WN18RR, FB15k237 and YAGO3-10).
%\footnote{$\surd$ again, this will be
%more complete if u include the 2 obgl datasets}
Table~\ref{tb:sf-dependent} shows the testing MRRs obtained (the trends
for H@1 and H@10  are
similar).
As can be seen,
the different BLMs perform differently on the same dataset, 
again confirming the
need for data-dependent structures.

\subsubsection{Results
on 
ogbl-biokg and ogbl-wikikg2}
\label{exp:kgc:ogb}

%As studied in Section~\ref{ssec:KGC}, AutoBLM+ is the state-of-the-art method in KG completion.  To further demonstrate its effectiveness, we evaluate AutoBLM+ on two large-scaled datasets, i.e., ogbl-biokg and ogbl-wikikg2.

Table~\ref{tb:comparison:ogb} shows the testing MRRs of
the baselines (as reported in the OGB leaderboard),
the BLMs obtained by AutoBLM and AutoBLM+.
%Since it is expensive to tune on the two large-scale datasets, we directly copy and compare with the well-tuned results of the baselines on the OGB leaderboard.  Both the MRR metric and the number of parameters are compared as provided by the OGB's evaluator.
As can be seen,
AutoBLM and 
AutoBLM+ achieve significant gains on the testing MRR 
on both datasets, 
even though
fewer model parameters are needed for AutoBLM+.
The searched structures are provided in Figure~\ref{fig:structure:ogb}
in Appendix~\ref{app:figures}.

\begin{table}[ht]
	\vspace{-7px}
	\caption{Testing MRR and number of parameters on ogbl-biokg and ogbl-wikikg2.
		%The number of parameters is represented in millions (M).
		The best performance is indicated in boldface.}
	\label{tb:comparison:ogb}
	\centering
		\setlength\tabcolsep{8pt}
	\vspace{-10px}
	\begin{tabular}{c|cc|cc}
		\toprule
		& \multicolumn{2}{c|}{ogbl-biokg} & \multicolumn{2}{c}{ogbl-wikikg2} \\
		model    & MRR & \# params & MRR  & \# params \\
		\midrule
		TransE   & 0.745      & 188M      & 0.426         & 1251M     \\
		RotatE   & 0.799  & 188M      & 0.433        & 1250M     \\
		PairE    	& 0.816     & 188M      & 0.521         & 500M      \\
		DistMult & 0.804      & 188M      & 0.373        & 1250M     \\
		ComplEx  & 0.810      & 188M      & 0.403         & 1250M     \\
		\midrule
		AutoBLM  &  {0.828}    &  188M   &   {0.532}  &  500M \\
		AutoBLM+ & \textbf{0.831}     & 94M       & \textbf{0.546}    & 500M     \\
		\bottomrule
	\end{tabular}
\vspace{-8px}
\end{table}

\subsubsection{Ablation Study 1: Search Algorithm Selection}
\label{exp:alg:compare}

First,
we 
study the following 
search algorithm
choices.
\begin{enumerate}[label=(\roman*)]
\item Random, which
samples each element 
of $\bm{A}$ 
	independently and uniformly 
	from $\{0, \pm1,\dots, \pm K\}$;
	
\item {Bayes}, which 
	selects each element of $\bm{A}$ 
	from $\{0, \pm1,\dots, \pm K\}$
	by
performing hyperparameter optimization using the Tree
	Parzen estimator \cite{bergstra2011algorithms} and Gaussian mixture model
	(GMM);
	
	\item Reinforce,
	which
	generates the $K^2$ elements in $\bm{A}$
	by using a LSTM~\cite{hochreiter1997long} 
	recurrently 
	as in NAS-Net~\cite{zoph2017neural}.
	The LSTM is optimized with REINFORCE~\cite{williams1992simple};
	
	\item AutoBLM (no Filter, no Predictor, $b_0\!\!=\!\!1$) with initial $b_0=1$;
	
	\item AutoBLM+ (no Filter, no Predictor, $b_0\!\!=\!\!1$) with initial $b_0=1$.
\end{enumerate}

For a fair comparison, we do not use the filter and performance
predictor
in the proposed 
	AutoBLM and
	AutoBLM+ here.
All structures selected by each of the above algorithms are trained and evaluated
with the same hyper-parameter settings in step~\ref{step:before-best} of Algorithm~\ref{alg:full}.
Each algorithm evaluates 
a total
of 256 structures.

\begin{figure}[ht]
	\vspace{-3px}
	\centering
	\includegraphics[height=3.4cm]{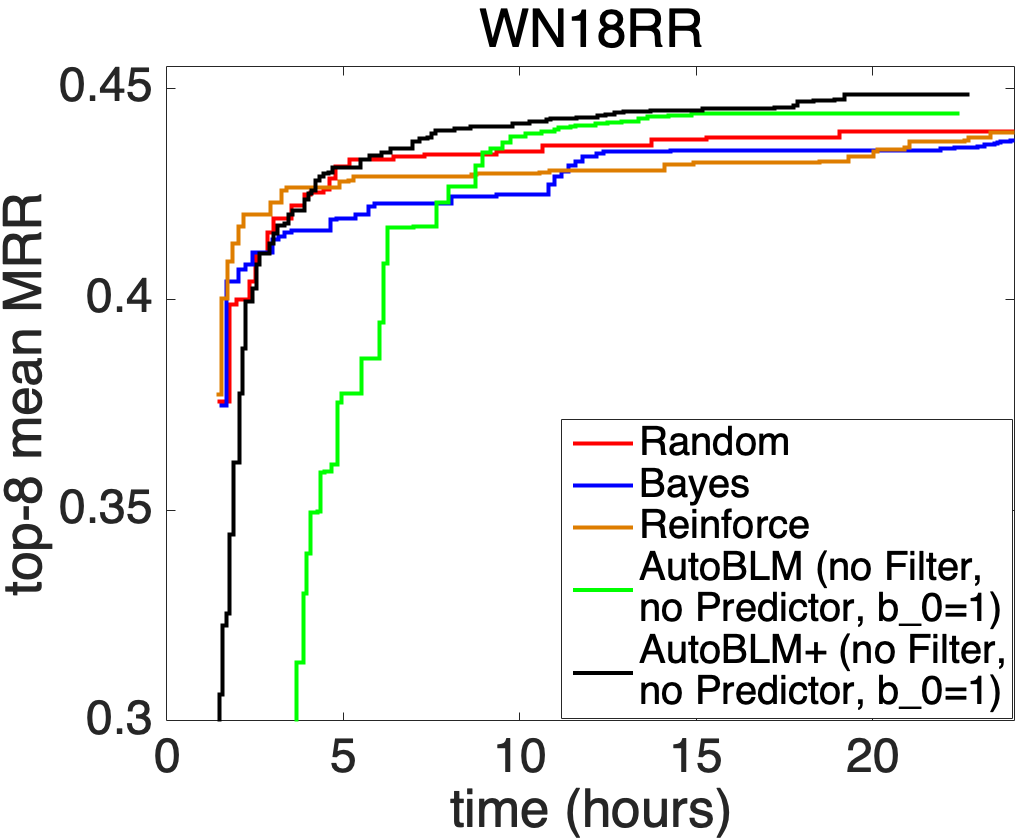}\hfill
	%	\hfill
	\includegraphics[height=3.4cm]{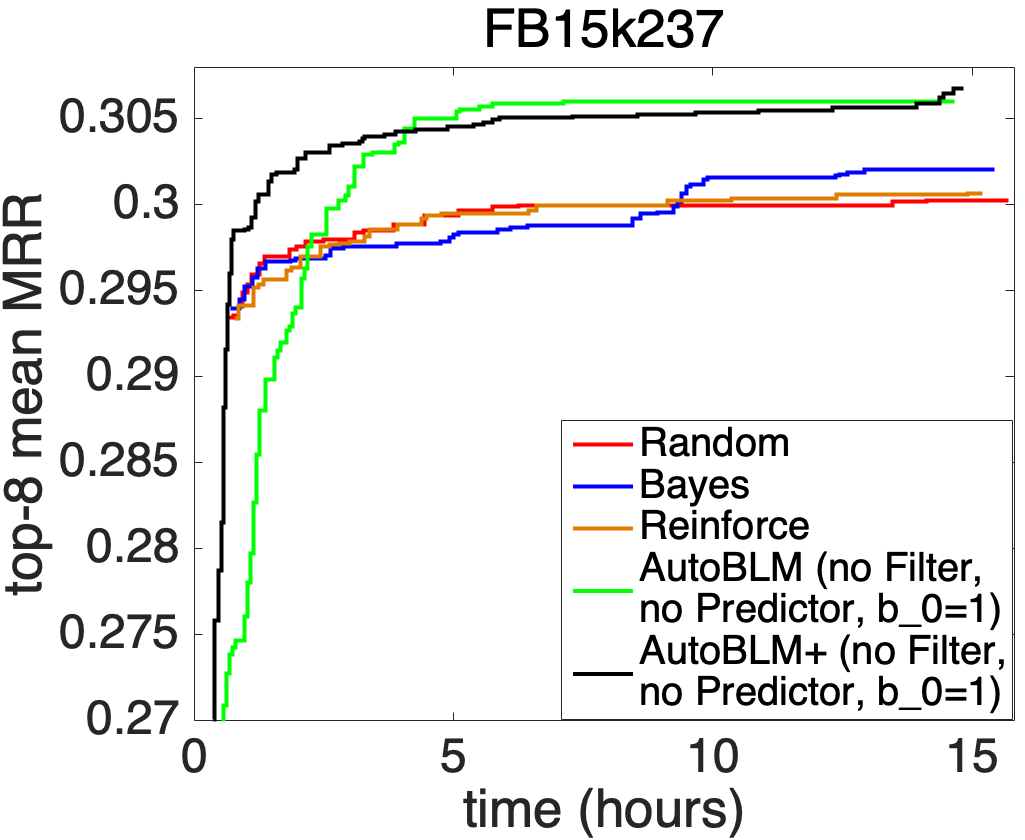}
	\vspace{-10px}
	\caption{Comparison of different search algorithms.}
	\label{fig:automl}
	\vspace{-4px}
\end{figure}

Figure~\ref{fig:automl}
shows the mean 
validation MRR of the top 
$I=8$ structures
w.r.t. clock time during the search process.
As can be seen,
AutoBLM (no Filter, no Predictor, $b_0\!\!=\!\!1$) and 
AutoBLM+ (no Filter, no Predictor, $b_0\!\!=\!\!1$) outperform the rest at the later stages.
They have
poor initial performance as they start with structures having few nonzero elements,
which can be degenerate.
This will be further demonstrated in the next section.

\subsubsection{Ablation Study 2: 
Effectiveness of  the
Filter}
\label{exp:filter}

Structures with more nonzero elements are more likely to satisfy the
two conditions in 
Proposition~\ref{pr:degenerate}, and 
thus less likely to be
degenerate.
Hence,
the filter is expected to be particularly useful
when there are few nonzero elements 
in the structure.
In this experiment,
we demonstrate this by
comparing
AutoBLM/AutoBLM+ with and without
the filter.
The performance predictor is always enabled.

\begin{figure}[ht]
	\centering
%	\vspace{-5px}
	\includegraphics[height=3.4cm]{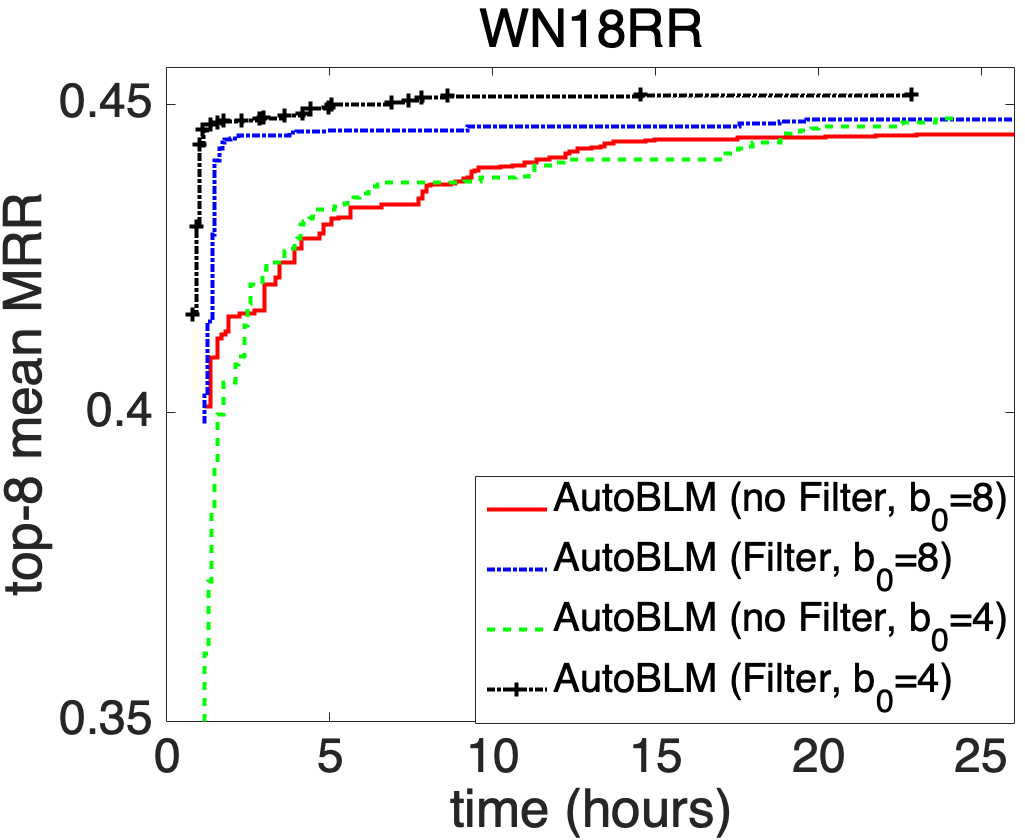}\hfill
	\includegraphics[height=3.4cm]{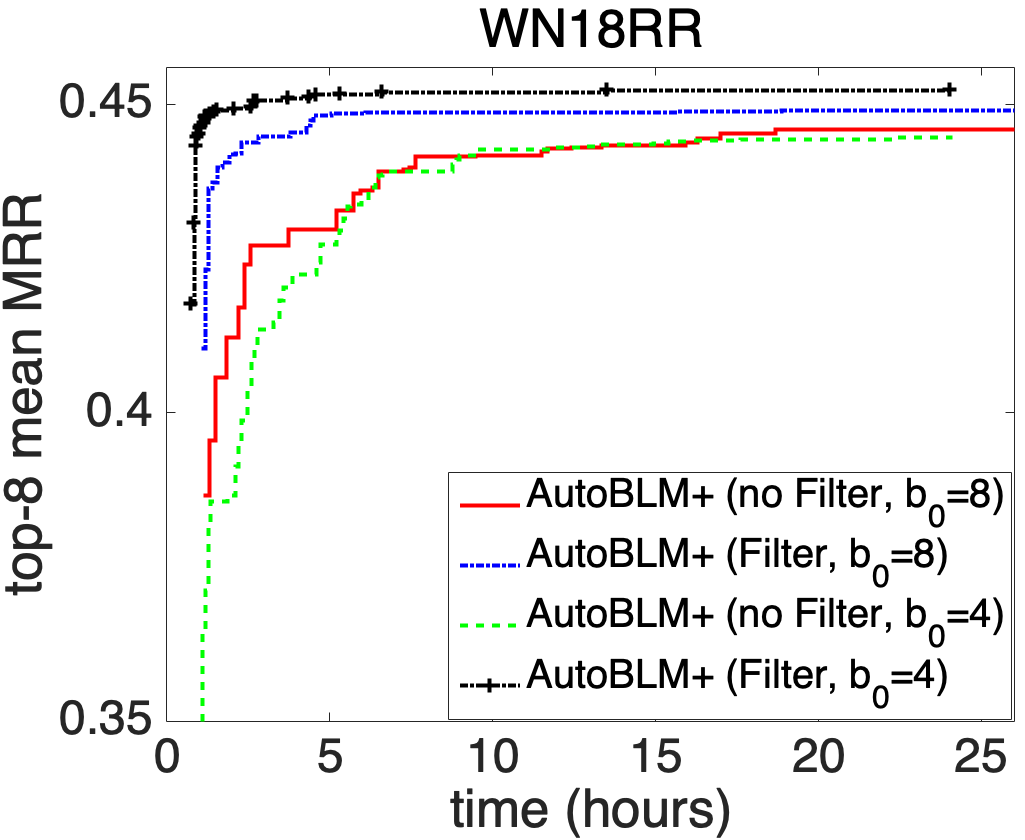}
	
	\includegraphics[height=3.4cm]{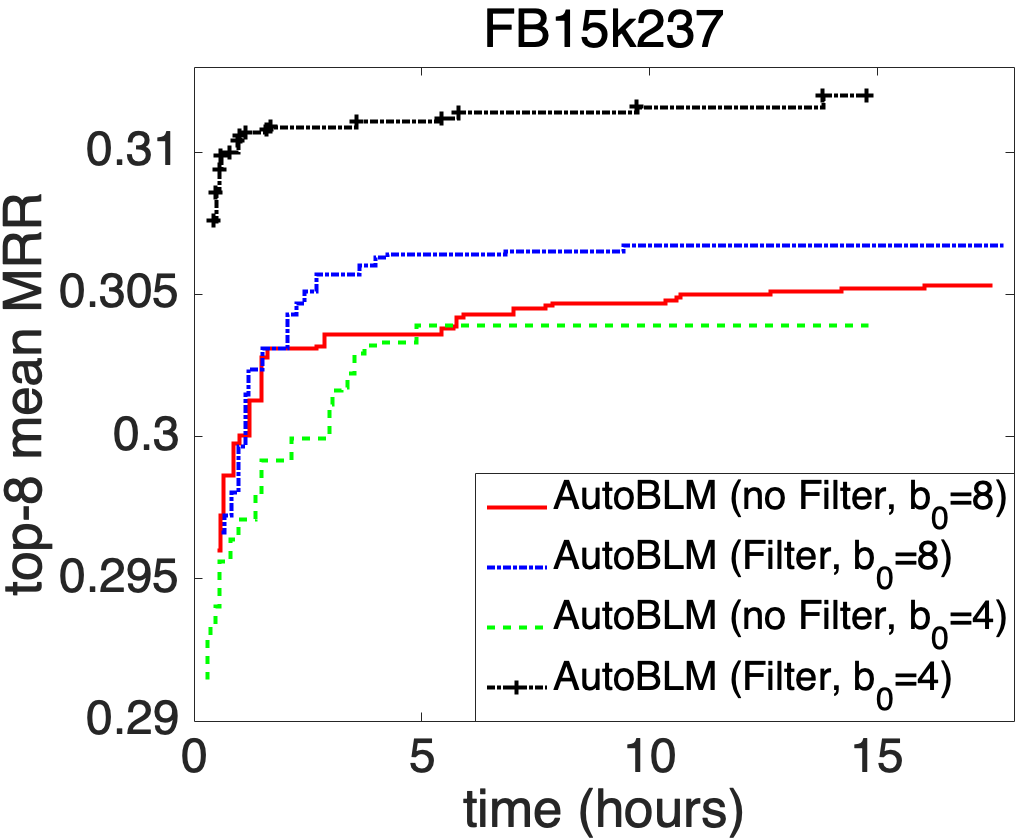}\hfill
	\includegraphics[height=3.4cm]{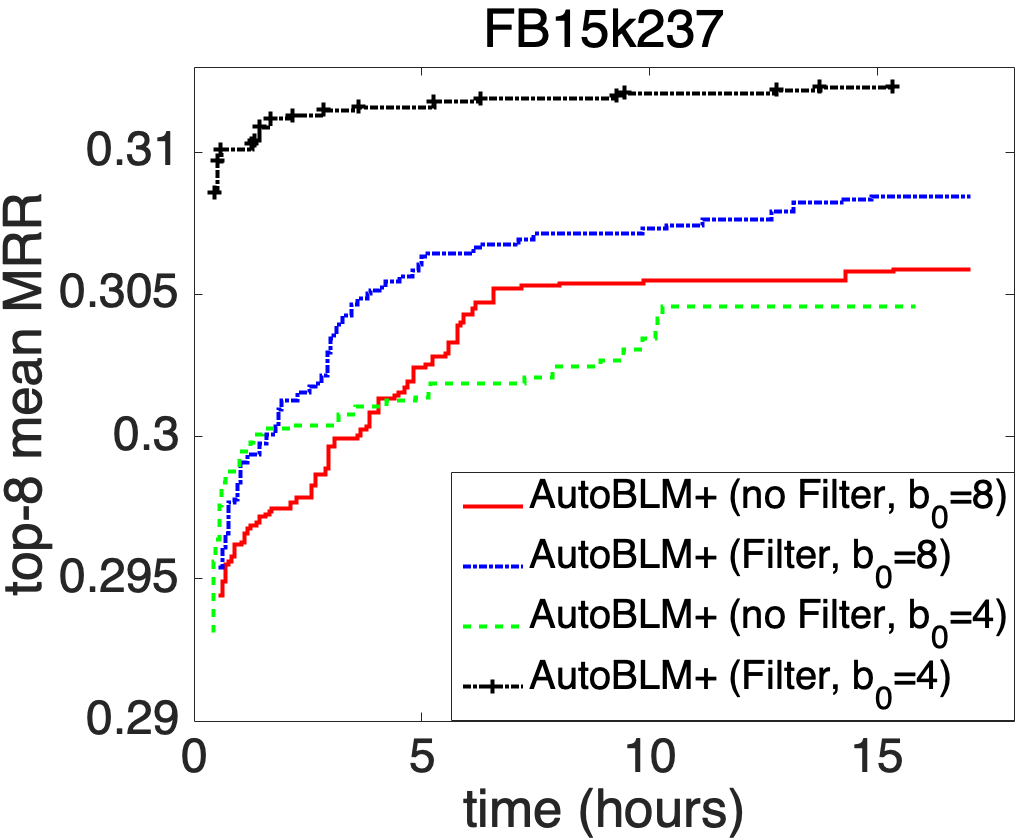}
	\vspace{-10px}
	\caption{Comparison of the effect of filter.}
	\vspace{-10px}
	\label{fig:filter}
\end{figure}

Figure~\ref{fig:filter}
shows the mean 
validation 
MRR 
of the top $I = 8$ structures
w.r.t. clock time.
As expected,
when the filter is not used, using a larger $b_0$ 
will be more likely to have non-degenerate
structures  and thus
better performance, especially at the initial stages.
When the filter is used, the performance of both 
$b_0$ 
settings are improved. In particular,
with $b_0=4$,
the initial search space is simpler and leads to
better performance.

\subsubsection{Ablation Study 3: Performance Predictor}
\label{exp:predictor}

In this experiment, we
compare
the following 
AutoBLM/AutoBLM+
variants: (i)
AutoBLM (no-predictor) and AutoBLM+ (no-predictor), which simply
randomly select $P$ structures for evaluation (in step~17 of
Algorithm~\ref{alg:greedy} and step 16 of
Algorithm~\ref{alg:evolution}, respectively);
(ii)
AutoBLM (Predictor+SRF) and AutoBLM+ (Predictor+SRF),
using the proposed SRF (in Section~\ref{sssec:predictor}) as input features to the performance predictor;
and (iii) AutoBLM (Predictor+1hot) and AutoBLM+ (Predictor+1hot),
which 
map each of
the $K^2$ entries in $\bm{A}$ (with values in $\{0,\pm1,\dots,\pm K\}$)
to 
a simple ($2K+1$)-dimensional
one-hot vector, and then 
use these as features to the performance predictor.
The resultant feature vector is thus $K^2(2K+1)$-dimensional, which is
much longer
than the $K(K+1)$-dimensional
SRF 
representation.

\begin{figure}[ht]
	\centering
	\vspace{-5px}
	\includegraphics[height=3.4cm]{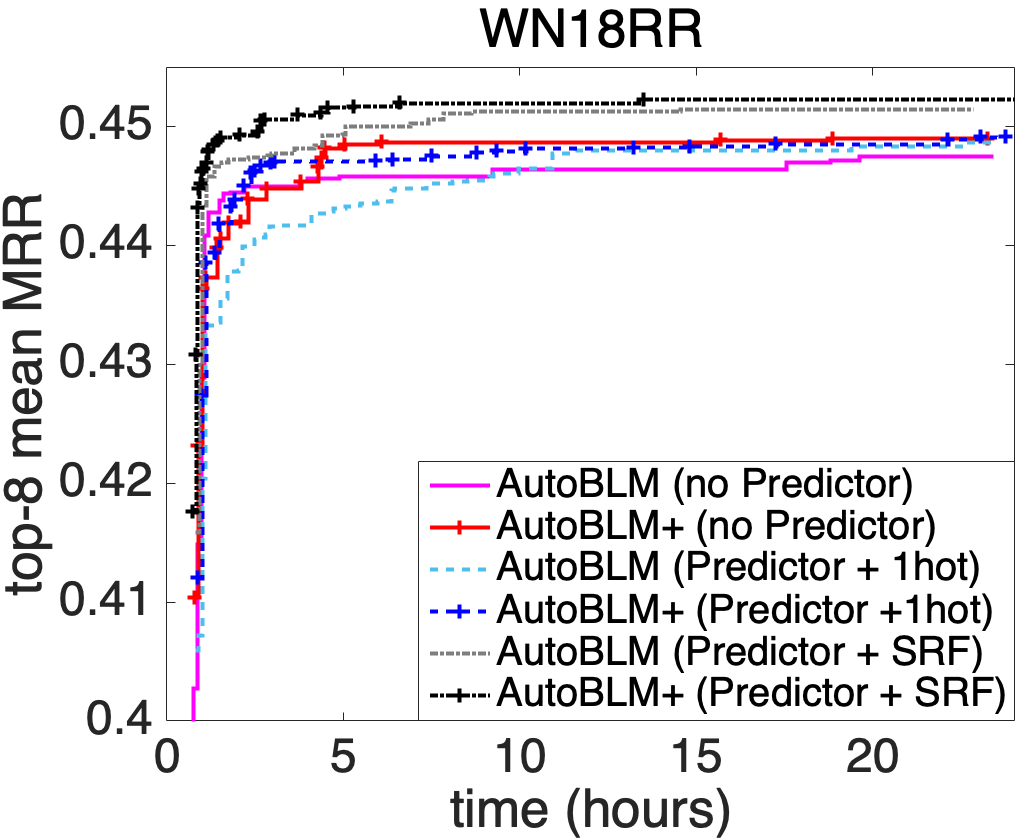}\hfill
	\includegraphics[height=3.4cm]{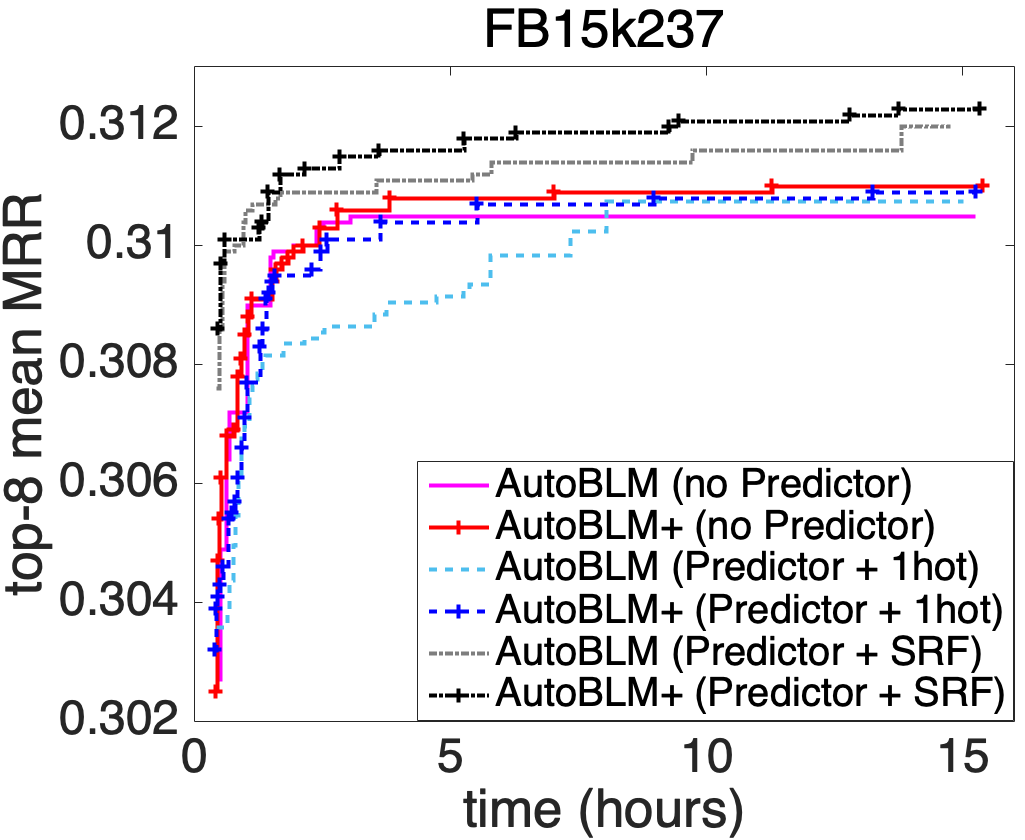}
	\vspace{-10px}
	\caption{Effectiveness of the performance predictor.}
	\vspace{-2px}
	\label{fig:predictor}
\end{figure}

Figure~\ref{fig:predictor} shows the mean validation MRR of the top $I=8$
structures w.r.t. clock time.
As can be seen,
the use of performance predictor improves the  results over
AutoBLM (no-Predictor) and AutoBLM+ (no-Predictor).
The SRF features also perform better than the one-hot features,
as the one-hot features are higher-dimensional and more difficult to learn.
Besides,
we observe that  
AutoBLM+ performs better than 
AutoBLM,
 as
it can 
more flexibly
explore  the search space.
Thus,
in the remaining ablation studies,
we will only focus on 
AutoBLM+.

\subsubsection{Ablation Study 4: Varying $K$}
\label{sec:exp:varyK}

As $K$ increases,
the search space, which has a size of $(2K+1)^{K^2}$
(Section~\ref{ssec:searSFs}), increases dramatically.
Moreover, the SRF also needs to enumerate a lot more 
($K^{2K+1}$) vectors in $\mathcal C$.
In this experiment,
we demonstrate the dependence on $K$ by running
AutoBLM+ with
$K=3,4,5$.
To ensure that $d$ is divisible by $K$,
we set $d=60$.
Figure~\ref{fig:ks} shows
the top-8 mean MRR performance on the validation set of the searched models
versus clock time.
As can be seen, the best performance 
attained by different $K$'s
are similar.
However,
$K=5$ runs slower.

\begin{figure}[ht]
	\centering
	\vspace{-3px}
	\includegraphics[height=3.4cm]{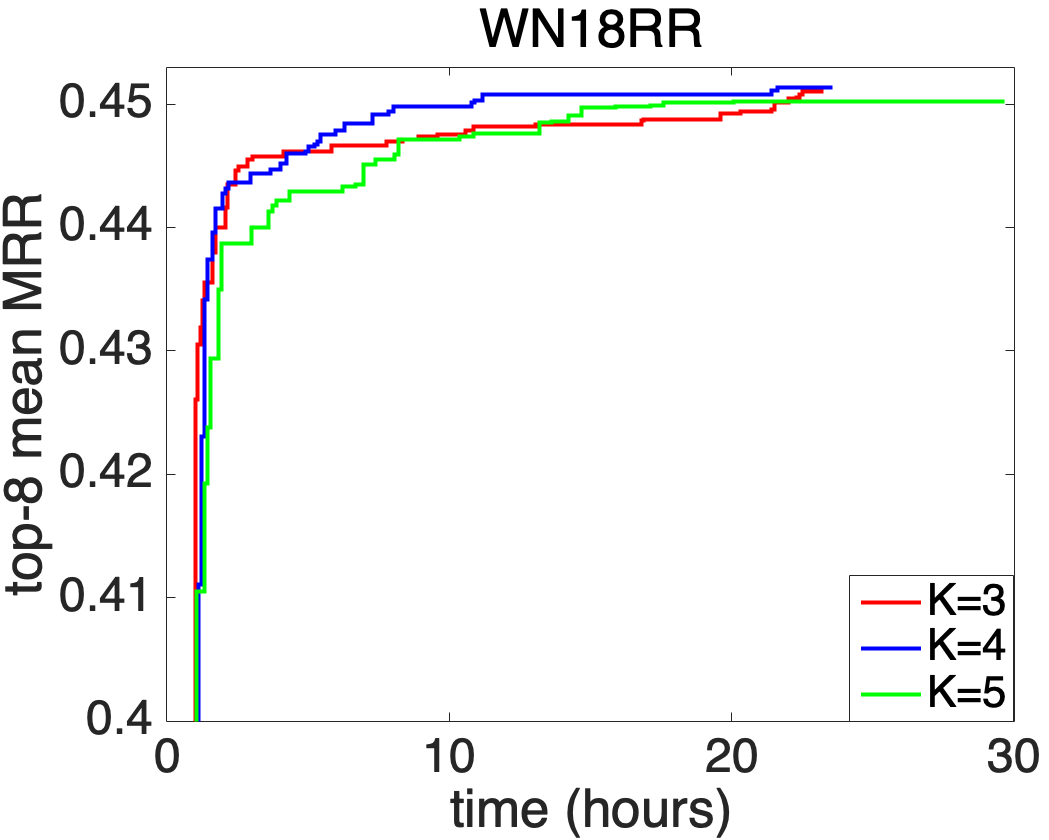}
	\hfill
	\includegraphics[height=3.4cm]{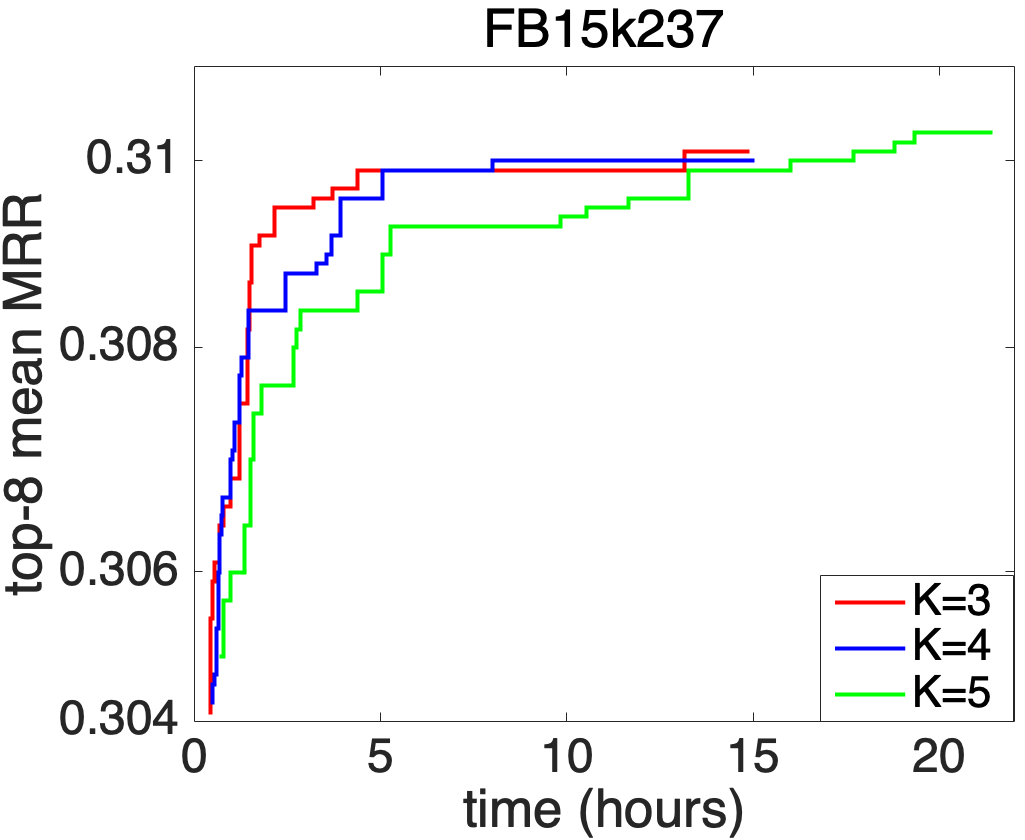}
	
	\vspace{-10px}
	\caption{Comparison of different $K$ values.}
	\vspace{-5px}
	\label{fig:ks}
\end{figure}

Table~\ref{tab:Ktime} shows
the running time of the
filter, performance predictor (with SRF features),
training and evaluation in Algorithm~\ref{alg:evolution}
with different $K$'s.
As can be seen,
the costs of filter and performance predictor increase a lot with $K$,
while the model training and evaluation time are relatively stable for different $K$'s.

\begin{table}[ht]
	\centering
	\vspace{-5px}
	\caption{Running time (in minutes) of different components in Algorithm~\ref{alg:evolution}.}
	\label{tab:Ktime}
	\vspace{-10px}
	\setlength\tabcolsep{4pt}
	\begin{tabular}{c|c|cccc}
		\toprule
		dataset &  $K$  &  filter  & performance predictor  & train  & evaluate \\
		\midrule
		\multirow{3}{*}{WN18RR} & 3  & 0.04 & 1 & 1217  & 152 \\
		& 4  & 1.4 &  23 &  1231 & 156 \\
		& 5  & 90 & 276 & 1252 & 161 \\  
		\midrule
		\multirow{3}{*}{FB15k237} & 3  & 0.04& 1 & 714 & 178  \\
		& 4  & 1.5 & 22 & 721 & 181 \\
		& 5  & 91 & 283 &  728 & 186 \\  
		\bottomrule
	\end{tabular}
\vspace{-10px}
\end{table}

\subsubsection{Ablation Study 5: Analysis of Parameter Sharing}
\label{sec:exp:ps}

As mentioned in 
Section~\ref{sssec:predictor},
parameter sharing may not reliably predict the model performance.
To demonstrate this,
we empirically compare the parameter-sharing approach, which 
shares parameter $\bm P=\{\bm E, \bm R\}$ 
(where $\bm E\in\mathbb R^{d\times |\mathcal E|}$ is the entity embedding matrix and 
$\bm R\in\mathbb R^{d\times |\mathcal R|}$ is the
relation embedding matrix
in Section~\ref{ssec:kg})
and the stand-alone approach, which trains each model separately.
For parameter sharing,
we randomly sample a $\bm{A}$ 
in each training  iteration
from the set of top candidate structures 
($\mathcal H^b$ in Algorithm \ref{alg:greedy} or $\mathcal H$ in Algorithm \ref{alg:evolution}),
and then update parameter $\bm P$.
After one training epoch,
the sampled structures 
are evaluated. 
After 500 training epochs,
the top-100 $\bm{A}$'s 
are output.
For the stand-alone approach,
the 100 $\bm{A}$'s 
are separately trained and evaluated.

\begin{figure}[ht]
	\centering
	\vspace{-5px}
	\includegraphics[height=3.4cm]{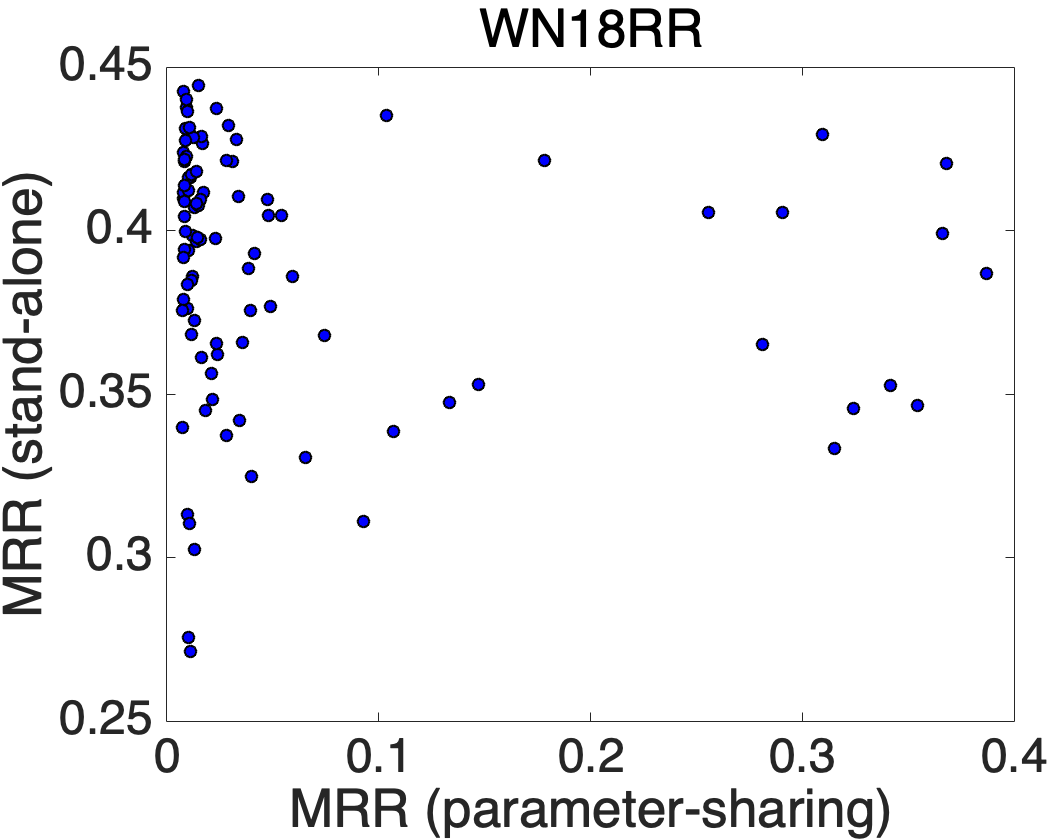}
	\hfill
	\includegraphics[height=3.4cm]{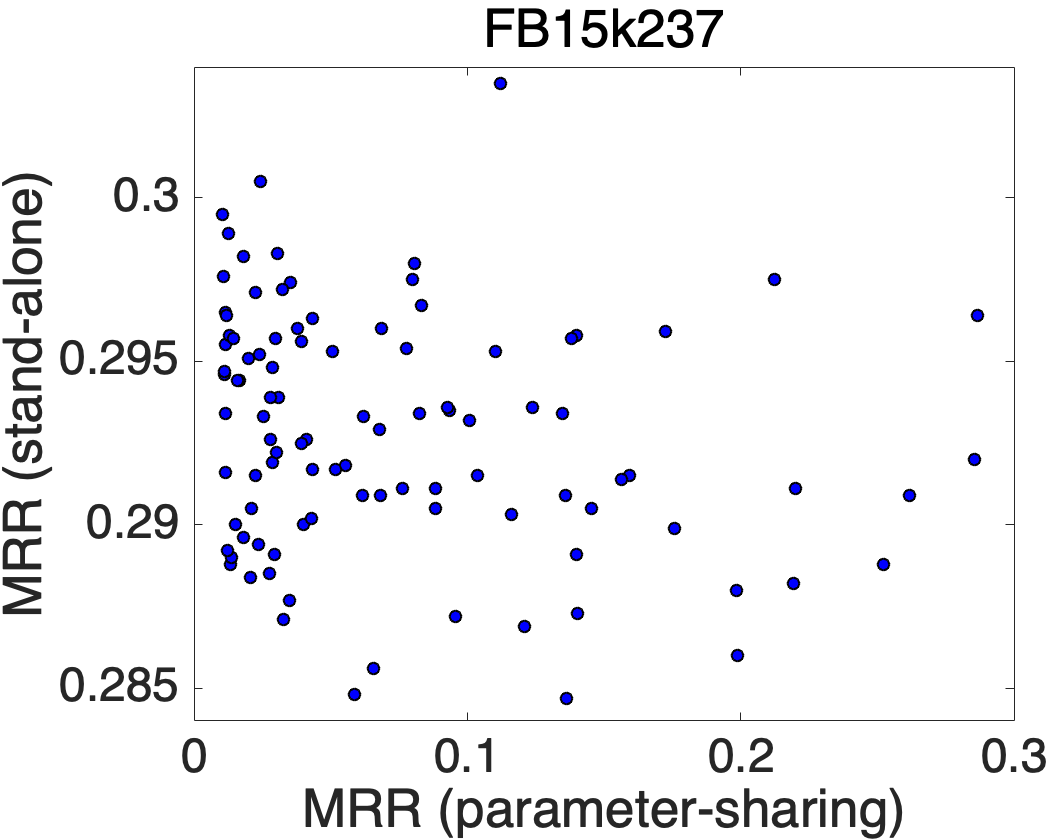}
	
	\vspace{-10px}
	\caption{
		MRRs of structures as estimated by the parameter-sharing approach and stand-alone
		approach.}
	\label{fig:share}
	\vspace{-6px}
\end{figure}

Figure~\ref{fig:share} shows the 
MRR estimated by parameter-sharing versus 
the true MRR obtained by individual model training.
As can be seen,
structures that have high
estimated MRRs 
(by parameter sharing)
do not truly have 
high MRRs.
Indeed, the Spearman's rank correlation coefficient\footnote{\url{https://en.wikipedia.org/wiki/Spearman\%27s_rank_correlation_coefficient}}
between the two sets of
MRRs is negative
($-0.2686$ 
on WN18RR  
and $-0.2451$ 
on FB15k237).
This demonstrates that 
the one-shot approach, though faster,
cannot find good structures.

\subsection{Multi-Hop Query}
\label{sec:exp:hop}

In this section, we perform experiment on 
multi-hop query
as introduced in Section~\ref{sec:app2}.  
The 
entity and relation
embeddings 
are optimized
by maximizing the scores on positive queries 
and minimizing the scores on negative queries, which 
are generated by replacing $e_L$ with an incorrect entity.
On evaluation,
we rank the scores of queries  $(e_0, r_1\circ r_2 \circ \dots \circ r_L, e_L)$ of all $e_L\in\mathcal E$
to obtain the ranking of ground truth entities.

\subsubsection{Setup}

Following~\cite{ren2019query2box},
we use the FB15k and FB15k237 datasets
in Table~\ref{tab:dataset}.
Evaluation is based  on two-hop (2p) and three-hop 
(3p)
queries.
Interested readers are referred to~\cite{ren2019query2box} 
for a more detailed description on query generation.
For FB15k,
there are 273,710 queries in the training set,
8,000 non-overlapping queries in the validation and testing sets.
For FB15k237,
there are 143,689 training queries,
and 5,000 queries for validation and testing.
The setting of the search algorithms' hyper-parameters are 
the same as in Section~\ref{ssec:KGC}.
For the learning hyper-parameters,
we search the dimension $d\in\{32,64\}$,
and the other hyper-parameters are the same as those in Section~\ref{ssec:KGC}.
We use the MRR performance on the validation set
to search for structures as well as hyper-parameters.
For performance evaluation,  
we follow 
\cite{hamilton2018embedding,ren2019query2box},
and use 
the testing Hit@3 and MRR.

We compare with the following baselines:
(i) TransE-Comp~\cite{guu2015traversing}
(based on
TransE); (ii)
Diag-Comp~\cite{guu2015traversing}
(based on DistMult); 
(iii) GQE~\cite{hamilton2018embedding}, which
uses a 
$d\times d$ trainable 
matrix $\bm R_{(r)}$ 
for composition,
and can be regarded as a composition based on RESCAL~\cite{nickel2011three}; and
(iv)
Q2B~\cite{ren2019query2box},
which is a recently proposed box embedding method.

\subsubsection{Results}

Results are
shown in Table~\ref{tab:exp:query}.
As can be seen, among the baselines,
TransE-Comp, Diag-Comp and GQE are inferior
to Q2B.
This shows that the general scoring functions
cannot be directly applied to model the complex interactions
in multi-hop queries.
On the other hand,
AutoBLM and
AutoBLM+ have better performance as
they can adapt to the different tasks with different matrices $g_K(\bm{A}, \bm r)$.
The obtained structures can be found in Appendix \ref{app:figures}.

\begin{table}[ht]
	\vspace{-5px}
	\caption{Testing performance of H@3 and MRR on multi-hop query task. Results of $\star$'s are copied from \cite{ren2019query2box}.}
	\label{tab:exp:query}
	\vspace{-10px}
	\centering
	\setlength\tabcolsep{2.5pt}
	\renewcommand{\arraystretch}{1.1}
	\begin{tabular}{c|cc|cc|cc|cc}
		\toprule
		& \multicolumn{4}{c|}{FB15K}                    & \multicolumn{4}{c}{FB15K237}                       \\ 
		& \multicolumn{2}{c|}{2p} & \multicolumn{2}{c|}{3p} & \multicolumn{2}{c|}{2p}	
		& \multicolumn{2}{c}{3p} \\ 
		& H@3        & MRR       &   H@3       &  MRR     &  H@3        & MRR        &  H@3        &  MRR         \\ \midrule 
		TransE-Comp~\cite{guu2015traversing}	&  27.3     &   .264    & 15.8 &  .153     &  19.4 &  .177   & 14.0 & .134  \\ 
		Diag-Comp~\cite{guu2015traversing}   &   32.2   &	  .309  &  27.5   & .266  &   19.1   &  .187   & 15.5  &   .147         \\  
		GQE~\cite{hamilton2018embedding}$\star$                      & 34.6       & .320      & 25.0      &  .222   & 21.3       &  .193     & 15.5      &  .145          \\ 
		Q2B~\cite{ren2019query2box}$\star$     & 41.3       & .373      & 30.3     &    .274     & 24.0     &    .225     & 18.6      &  .173        \\ \midrule
		AutoBLM			& 41.5	& .402	&  29.1 &	.283  &	 23.6	&  .232	 &	  18.2	&  .180	\\  
		AutoBLM+        &   \textbf{43.2}   &    \textbf{.415}  &   \textbf{30.7}   &    \textbf{.293}   &     \textbf{24.9}     &      \textbf{.248}   &  \textbf{19.9}     &  	\textbf{.196}    \\ \bottomrule
	\end{tabular}
\vspace{-10px}
\end{table}

\subsection{Entity Classification}
\label{sec:exp:cls}

In this section, we perform experiment on 
entity classification
as introduced in 
Section~\ref{sssec:entclass}.

\subsubsection{Setup}

After aggregation for $L$ layers,
representation $\bm e^L$
at the last layer 
is transformed by a multi-layer perception (MLP) to
$\bm e_i^o = MLP(\bm e_i^L)\in\mathbb R^C$,
where
$d$ 
is 
the intermediate layer dimension,
and 
is the number of classes.
The parameters,
including embeddings of entities, relations,
$\bm W_0^\ell$, $\bm W^{\ell}$'s and the MLP,
are optimized by minimizing the cross-entropy loss 
on the labeled entities:
$\mathcal L = -\sum_{i\in\mathcal B}\sum_{c=1}^C y_{ic}\ln e^o_{ic}$,
where $\mathcal B$ is the set of labeled entities,
$y_{ic} \in \{0,1\}$ indicates whether the $i$th entity belongs to class $c$,
and $e^o_{ic}$ is the $c$th dimension of $\bm e_i^o$.

Three graph datasets 
are used 
(Table~\ref{tab:stat:entclass}):
AIFB, an affiliation graph;
MUTAG, 
a bioinformatics graph;
and BGS, 
a geological graph.
More details can be found in~\cite{ristoski2016rdf2vec}.
All entities do not have attributes.
The entities' and relations'
trainable embeddings
are used 
as input to the GCN.

\begin{table}[ht]
	\centering
	\vspace{-5px}
	\caption{Data sets used in entity classification.
	Sparsity is computed as $\text{\#edges}/(\text{\#entity}^2\cdot \text{\#relation})$.}
	\label{tab:stat:entclass}
	\vspace{-10px}
	\setlength\tabcolsep{2.5pt}
	\begin{tabular}{c|ccccccc}
		\toprule
		dataset & \#entity     & \#relation & \#edges   & \#train & \#test & \#classes & sparsity \\ \midrule
		AIFB    & 8,285   & 45  & 29,043  &  140     & 36     & 4     &  9.4e-6\\
		MUTAG   & 23,644  & 23  & 74,227  &  272     &   68   & 2   &  5.7e-6  \\
		BGS     & 333,845 & 103 & 916,199 &   117    &  29    & 2     & 8.0e-8 \\ \bottomrule
	\end{tabular}
\vspace{-5px}
\end{table}

The following five models
are compared:
(i) GCN~\cite{kipf2016semi},
with $\phi(\bm e_j^{\ell}, \bm r^{\ell}) = \bm e_j^\ell$,
does not leverage relations of the edges;
(ii) 
R-GCN~\cite{schlichtkrull2018modeling}, with 
$\phi(\bm e_j^{\ell}, \bm r^{\ell}) = \bm R_{(r)}^{\ell}\bm e_j^{\ell}$;
(iii) CompGCN~\cite{vashishth2019composition} with 
$\phi(\bm e_j^{\ell}, \bm r^\ell) = \bm e_j^{\ell}$ (-/*/$\star$) $\bm r^{\ell}$,
in which the operator (subtraction/multiplication/circular correlation as
discussed in Section~\ref{sssec:entclass})
is chosen based on 
5-fold cross-validation;
(iv) AutoBLM; and (v) AutoBLM+.
oth 
AutoBLM and AutoBLM+
use the searched structure $\bm{A}$ to form 
$\phi(\bm e_j^{\ell}, \bm r^{\ell}) =
g_K(\bm{A}, \bm r^\ell)\bm e_j^{\ell}$.

Setting of the hyper-parameters  are the same as in
Section~\ref{ssec:KGC}.
As for the learning hyper-parameters,
we search the embedding dimension $d$ 
from $\{12,20,32,48\}$, 
learning rate from $[10^{-5}, 10^{-1}]$ with Adam as the optimizer~\cite{kingma2014adam}.
For the GCN structure,
the hidden size is the same as the embedding dimension,
the dropout rate 
for each layer is from $[0,0.5]$.
We search for 50 hyper-parameter settings
for each dataset based on the 5-fold classification accuracy.

For performance evaluation, we use 
the testing accuracy.
Each model runs 5 times, and then the average testing accuracy reported.

\subsubsection{Results}

Table~\ref{tab:classification}
shows the average testing accuracies.
Among the baselines, R-GCN is slightly better than CompGCN on the AIFB dataset,
but worse on the other two sparser datasets.
By searching the composition operators,
AutoBLM and AutoBLM+  outperform all the baseline methods.
AutoBLM+ is better than AutoBLM 
since it can find better structures 
with the same budget 
by the evolutionary algorithm.
The structures obtained are in Appendix \ref{app:figures}.

\begin{table}[ht]
	\centering
	\vspace{-5px}
	\caption{Classification accuracies (in \%) on entity classification task. Values marked ``*'' are copied from~\cite{vashishth2019composition}.}
	\label{tab:classification}
	\vspace{-10px}
	\begin{tabular}{c|ccc}
		\toprule
		 dataset   &        AIFB         &        MUTAG        &         BGS         \\ \midrule
 GCN     &   86.67    &   68.83     &   73.79     \\
R-GCN    &   92.78      &   74.12    &   82.97    \\
CompGCN   &  90.6$^*$   &  85.3$^*$   &   84.14    \\
AutoBLM  &   95.55 	&   85.00   &  84.83   \\
AutoBLM+ & \bf{96.66} & \bf{85.88} & \bf{86.17} \\ \bottomrule
	\end{tabular}
\vspace{-12px}
\end{table}

\section{Conclusion}
\label{sec:conclusion}

In this paper, we propose AutoBLM and AutoBLM+, 
the algorithms to automatically design and discover better scoring functions for KG learning.
By analyzing the limitations of existing scoring functions,
we setup the problem as searching relational matrix for BLMs.
In AutoBLM,
we use a progressive search algorithm
which is enhanced by a filter and a predictor with domain-specific knowledge,
to search in such a space.
Due to the limitation of progressive search,
we further design an evolutionary algorithm,
enhanced by the same filter and predictor,
called AutoBLM+.
AutoBLM and AutoBLM+ can efficiently design scoring functions
that outperform existing ones on tasks
including 
KG completion,
multi-hop query
and entity classification from the large search space.
Comparing AutoBLM with AutoBLM+,
 AutoBLM+ can
 design better scoring functions
with the same budget.

\section*{Acknowledgment}
This work was supported by the National Key Research and Development Plan under Grant 2021YFE0205700, the Chinese National Natural Science Foundation Projects 61961160704, and the Science and Technology Development Fund of Macau Project 0070/2020/AMJ.

\bibliographystyle{plain}
\bibliography{bib}

\begin{IEEEbiography}[{\includegraphics[width = 1\textwidth]{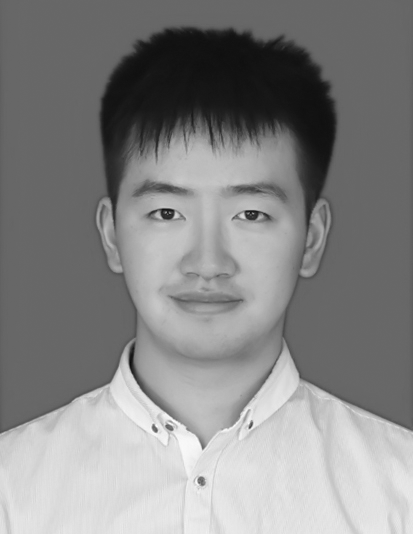}}]{Yongqi Zhang}
	(Member, IEEE)
is a senior researcher in 4Paradigm.
He obtained his Ph.D. degree at the Department of Computer Science and Engineering of Hong Kong University of Science and Technology (HKUST) in 2020 and received his bachelor degree at Shanghai Jiao Tong University (SJTU) in 2015.
He has published five top-tier conference/journal papers as first-author, including NeurIPS, ACL, WebConf, ICDE, VLDB-J.
His research interests focus on knowledge graph embedding,
automated machine learning and graoh learning.
He was a Program Committee for AAAI 2020-2022, IJCAI 2020-2022, CIKM 2021, KDD 2022, ICML 2022, and a reviewer for TKDE and NEUNET.
\end{IEEEbiography}

\begin{IEEEbiography}[{\includegraphics[width = 1\textwidth]{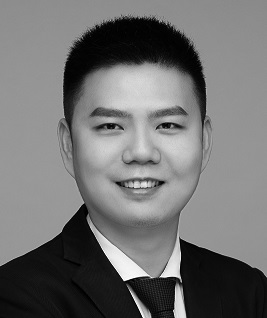}}]{Quanming Yao} (Member, IEEE)
is a tenure-track assistant professor at Department of Electronic Engineering, Tsinghua University. Before that, he spent three years from a researcher to a senior scientist in 4Paradigm INC, where he set up and led the company's machine learning research team.  He is a receipt of Wunwen Jun Prize of Excellence Youth of Artificial Intelligence (issued by CAAI), the runner up of Ph.D. Research Excellence Award (School of Engineering, HKUST), and a winner of Google Fellowship (in machine learning).
Currently,
his main research topics are Automated Machine Learning (AutoML) and neural architecture search (NAS).
He was an Area Chair for ICLR 2022, IJCAI 2021 and ACML 2021; 
Senior Program Committee for IJCAI 2020 and AAAI 2020-2021; and a guest editor of IEEE TPAMI AutoML special issue in 2019.
\end{IEEEbiography}

%\begin{IEEEbiography}[{\includegraphics[width = 1\textwidth]{figure/bio/li}}]{Yong Li} (Senior Member, IEEE)
%Yong Li (Senior Member, IEEE) received the B.S. degree in electronics and information engineering from the Huazhong University of Science and Technology, Wuhan, China, in 2007, and the Ph.D. degree in electronic engineering from Tsinghua University, Beijing, China, in 2012.
%He is currently a Faculty Member of the Department of Electronic Engineering, Tsinghua University. He has served as the General Chair, the TPC Chair, the SPC/TPC Member for several international workshops and conferences, and he is on the editorial board of two IEEE journals. His papers have total citations more than 6900. Among them, ten are ESI Highly Cited Papers in Computer Science, and four receive conference Best Paper (run-up) Awards. He received IEEE 2016 ComSoc Asia–Pacific Outstanding Young Researchers, Young Talent Program of China Association for Science and Technology, and the National Youth Talent Support Program.
%\end{IEEEbiography}

\begin{IEEEbiography}[{\includegraphics[width = 1\textwidth]{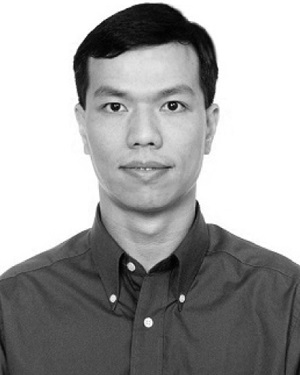}}]{James T. Kwok} (Fellow, IEEE)
	received the Ph.D. degree in computer science from The Hong Kong University of Science and Technology in 1996. 
	He is a Professor with the Department of Computer Science and Engineering, Hong Kong
	University of Science and Technology. His research interests include machine
	learning, deep learning, and artificial intelligence. He received the
	IEEE Outstanding 2004 Paper Award and the Second Class Award in Natural Sciences by
	the Ministry of Education, China, in 2008. 
	He is serving as an Associate Editor for the IEEE Transactions on Neural Networks and Learning Systems, Neural Networks, Neurocomputing, Artificial Intelligence Journal, International Journal of Data Science and Analytics, Editorial Board Member of Machine Learning,
	Board Member, and Vice President for Publications of the Asia Pacific Neural Network
	Society. He also served/is serving as Senior Area Chairs / Area Chairs of major machine learning /
	AI conferences including NIPS, ICML, ICLR, IJCAI, AAAI and ECML.
\end{IEEEbiography}

\cleardoublepage
\appendices

\section{Proofs}
\label{app:proof}

Denote $N=|\mathcal E|$. 
For vectors and matrix, we use the uppercase  italic bold letters, such as 
$\bm G,\bm E$, to denote matrices,
use uppercase normal bold letters, such as $\mathbf G$, to denote tensors,
lowercase bold letters, such as $\bm r$ to denote vectors,
and normal characters to indicate scalers, such as $K$.
$i, j$ are generally used to indicate index.
$\bm E_i, \bm r_i, i=1\dots K$, still a matrix or vector, are the $i$-th block of $K$-chunk split,
while $r_i, i=1\dots d$ is a scalar in the $i$-th dimension of vector $\bm r$.
$\bm E_{i,:}$ indicates the $i$-th row of matrix $\bm E$,
and $\bm E_{:,i}$ indicates the $i$-th column.
For $g_K(\bm{A}, \bm r)\in\mathbb R^{d\times d}$, 
we use $[g_K(\bm{A}, \bm r)]_{i,j}$ with $i,j=1\dots K$ to indicate the $(i,j)$-th block,
while $\{g_K(\bm{A}, \bm r)\}_{i,j}$ with $i, j = 1 \dots d$ 
to indicate the element in the $i$-th row and $j$-th column of $g_K(\bm{A}, \bm r)$.
$\lceil x\rceil$ means the smallest integer that is equal or larger than $x$.

\subsection{Proposition~\ref{pr:expBLMs}}
\label{app:expBLMs}

Here,
we first show some useful lemmas
in Appendix~\ref{app:lem1},
then we prove

\subsubsection{Auxiliary Lemmas}
\label{app:lem1}

Recall that 
\begin{align*} 
\mathcal{C} 
\equiv 
\{ \bm{r} & \in\mathbb R^K \,|\, 
\bm r\neq \bm 0, 
\notag
\\ 
&r_i\in
\{0,\pm1,\dots,\pm K\},
i=1,\dots, K \}.
\end{align*} 
We firstly show that 
any symmetric matrix can be factorized as a bilinear form if 
$\exists \; \hat{\bm{r}}\in \mathcal C, g_{K}(\bm{A},\hat{\bm{r}})^\top=g_{K}(\bm{A}, \hat{\bm{r}})$ 
in Lemma~\ref{lm:sym};
any skew-symmetric matrix can be factorized as a bilinear form if
$\exists \; \breve{\bm{r}}\in \mathcal C, g_{K}(\bm{A},   \breve{\bm{r}})^\top=-g_{K}(\bm{A}, \breve{\bm{r}})$
in Lemma~\ref{lm:ssym}.

\begin{lemma}
\label{lm:sym}
If $\exists \; \hat{\bm{r}}\in\mathcal C, g_{K}(\bm{A},\hat{\bm{r}})^\top=g_{K}(\bm{A}, \hat{\bm{r}})$,
there exist $\acute{\bm{E}}\in\mathbb R^{KN\times N}$ and  $\acute{\bm{r}}\in\mathbb R^{KN}$
such that any symmetric matrix $\hat{\bm{G}} = {\acute{\bm{E}}}^\top g_K(\bm{A}, \acute{\bm r})\acute{\bm{E}}$.
\end{lemma}

\begin{proof}
For any real symmetric matrix $\hat{\bm{G}}\in\mathbb R^{N\times N}$,
it can be decomposed as~\cite{horn2012matrix}
\begin{equation}
\hat{\bm{G}} 
= 
\bm P^\top \bm{\Lambda} \bm P,
\label{eq:lemma1:gsym}
\end{equation}
where $\bm P\in\mathbb R^{N\times N}$ is an orthogonal matrix 
and $\bm{\Lambda} \in\mathbb R^{N\times N}$ 
is a diagonal matrix with $N$ elements.
If $\exists \; \hat{\bm{r}}\in\mathcal C, g_{K}(\bm{A},\hat{\bm{r}})^\top=g_{K}(\bm{A}, \hat{\bm{r}})$,
we discuss with two cases: 
1) there exist some non-zero elements in the diagonal;
2) all the diagonal elements are zero.

To begin with,
we evenly split
the relation embedding into $K$ parts as
$\acute{\bm r} = [\acute{\bm r}_1;\acute{\bm r}_2; \dots; \acute{\bm r}_K]\in\mathbb R^{KN}$,
and
the entity embedding matrix into $K$ blocks as
$\acute{\bm{E}} = [\acute{\bm{E}}_1; \acute{\bm{E}}_2; \dots; \acute{\bm{E}}_K]\in\mathbb R^{KN\times N}$.
Then for the two cases:

\begin{enumerate}[leftmargin=*]
	\item If there exist non-zero elements in the diagonal,
	i.e. $\exists i\in\{1\dots K\}:\;\{g_K(\bm A, \hat{\bm r})\}_{i,i}\neq 0$,
	we denote any one of the index in the diagonal as $(a_1, a_1)$ with $a_1\in\{1\dots K\}$.
	Then, 
	we have
	\begin{equation*}
	\{g_K(\bm A, \hat{\bm r})\}_{a_1,a_1} = \text{sign}(A_{a_1,a_1})\cdot \hat{\bm r}[|A_{a_1,a_1}|]\neq 0,
	\end{equation*}
	and $\text{sign}(A_{a_1,a_1})\neq 0$.
	Next,
	we assign $\acute{\bm{E}}$ with
	\begin{equation}
		\acute{\bm{E}}_i = 
		\begin{cases}
		\bm P &  i=a_1 \\
		\bm 0 & \text{otherwise}
		\end{cases}.
		\label{eq:lemma1:E1}
	\end{equation}
	And $\acute{\bm{r}}$ with
	\begin{equation}
		\acute{\bm r}_i = 
		\begin{cases}
		\text{sign}(A_{a_1,a_1})\cdot\text{vec}(\bm \Lambda)  &  i=|A_{a_1, a_1}| \\
		\bm 0 & \text{otherwise}
		\end{cases},
		\label{eq:lemma1:r1}
	\end{equation}
	where $\text{vec}(\bm \Lambda)$ returns the diagonal elements in the diagonal matrix $\bm \Lambda$
	and 
	\begin{align}
		&[g_K(\bm{A}, \acute{\bm{r}})]_{a_1,a_1}  
		\notag 
		\\
		=& 
		\text{sign}(A_{a_1,a_1})\cdot \text{diag}\big( \text{sign}(A_{a_1,a_1})\cdot\text{vec}(\bm \Lambda)\big) 
		\notag
		\\
		=& 
		\big(\text{sign}(A_{a_1,a_1})\big)^2 \cdot \bm \Lambda = 
		\bm \Lambda.
		\label{eq:lemma1:g1}
	\end{align}
	
	Based on \eqref{eq:lemma1:gsym}, \eqref{eq:lemma1:E1} and \eqref{eq:lemma1:g1},
	we have 
	\begin{align*}
	{\acute{\bm{E}}}^\top g_K(\bm{A}, \acute{\bm{r}})\acute{\bm{E}} 
	=&  
	\sum\nolimits_{i,j}^{K, K} {\acute{\bm{E}}_i}^\top [g_K(\bm{A}, \acute{\bm{r}})]_{i,j} \acute{\bm{E}}_j
	\\
	=& 
	\bm 0 + 
	{\acute{\bm{E}}_{a_1}}^\top [g_K(\bm{A}, \acute{\bm{r}})]_{a_1,a_1} \acute{\bm{E}}_{a_1}, 
	\\
	=&
	\bm 0 +
	\bm P^\top \bm \Lambda \bm P
	=
	\hat{\bm{G}}.
	\end{align*}
	
	\item If all the diagonal elements are zero,
	there must be some non-zero elements in the non-diagonal indices,
	i.e. 
	$\exists (i,j)\in\{1\dots K\} \times \{1\dots K\}:\;\{g_K(\bm A, \hat{\bm r})\}_{i,j}\neq 0 \wedge
	i\neq j$.
	We denote any one of the index in the non-diagonal indices as $(a_2, a_3)$
	with $a_2, a_3\in \{1\dots K\}$ and $a_2\neq a_3$.
	Then we have
	\begin{align*}
	\{g_K(\bm A, \hat{\bm r})\}_{a_2,a_3} &= \{g_K(\bm A, \hat{\bm r})\}_{a_3,a_2} \\
	&=\text{sign}(A_{a_2,a_3})\cdot \hat{\bm r}[|A_{a_2,a_3}|]\neq 0,
	\end{align*}
	and $\text{sign}(A_{a_2,a_3})\neq 0$.
	
	Similarly,
	we assign 
	$\acute{\bm{E}}$ with
	\begin{equation}
	\acute{\bm{E}}_i = 
	\begin{cases}
	\bm P &  i=a_2 \text{ or } a_3\\
	\bm 0 & \text{otherwise}
	\end{cases}.
	\label{eq:lemma1:E2}
	\end{equation}
	And $\acute{\bm{r}}$ with
	\begin{equation}
	\acute{\bm r}_i = 
	\begin{cases}
	\frac{\text{sign}(A_{a_2,a_3})}{2}\cdot\text{vec}(\bm \Lambda)  &  i=|A_{a_2, a_3}| \\
	\frac{\text{sign}(A_{a_3,a_2})}{2}\cdot\text{vec}(\bm \Lambda)  &  i=|A_{a_3, a_2}| \\
	\bm 0 & \text{otherwise}
	\end{cases},
	\label{eq:lemma1:r2}
	\end{equation}
	which leads to
	\begin{align}
	&[g_K(\bm{A}, \acute{\bm{r}})]_{a_2,a_3}  = [g_K(\bm{A}, \acute{\bm{r}})]_{a_3,a_2}
	\notag 
	\\
	=& 
	{\text{sign}(A_{a_2,a_3})}\cdot \text{diag}\big( \frac{\text{sign}(A_{a_2,a_3})}{2}\cdot\text{vec}(\bm \Lambda)\big) 
	\notag
	\\
	=& 
	\frac{1}{2} \bm \Lambda.
	\label{eq:lemma1:g2}
	\end{align}
	Based on \eqref{eq:lemma1:gsym}, \eqref{eq:lemma1:E2} and \eqref{eq:lemma1:g2},
	we have
	\begin{align*}
	&{\acute{\bm{E}}}^\top g_K(\bm{A}, \acute{\bm{r}})\acute{\bm{E}} 
	=  
	\sum\nolimits_{i,j}^{K, K} {\acute{\bm{E}}_i}^\top [g_K(\bm{A}, \acute{\bm{r}})]_{i,j} \acute{\bm{E}}_j
	\\
	=& 
	\bm 0 + 
	{\acute{\bm{E}}_{a_2}}^\top [g_K(\bm{A}, \acute{\bm{r}})]_{a_2,a_3} \acute{\bm{E}}_{a_3}
	+ {\acute{\bm{E}}_{a_3}}^\top [g_K(\bm{A}, \acute{\bm{r}})]_{a_3,a_2} \acute{\bm{E}}_{a_2}, 
	\\
	=&
	\bm 0 +
	\bm P^\top (\frac{1}{2}\bm \Lambda) \bm P + \bm P^\top (\frac{1}{2}\bm \Lambda) \bm P
	=
	\hat{\bm{G}}.
	\end{align*}
\end{enumerate}
Hence, 
there exist $\acute{\bm{E}}\in\mathbb R^{KN\times N}$ and  $\acute{\bm{r}}\in\mathbb R^{KN}$
such that any symmetric matrix $\hat{\bm{G}} = {\acute{\bm{E}}}^\top g_K(\bm{A}, \acute{\bm r})\acute{\bm{E}}$.
\end{proof}

\begin{lemma}
\label{lm:ssym}
If $\exists \; \breve{\bm{r}}\in \mathcal C, g_{K}(\bm{A},\breve{\bm{r}})^\top=-g_{K}(\bm{A}, \breve{\bm{r}})$,
there exist $\grave{\bm{E}}\in\mathbb R^{KN\times N}$ and  $\grave{\bm{r}}\in\mathbb R^{KN}$
such that any skew-symmetric matrix $\breve{\bm{G}} = {\grave{\bm{E}}}^\top g_K(\bm{A}, \grave{\bm{r}})\grave{\bm{E}} \in \mathbb R^{N\times N}$.
\end{lemma}

\begin{proof}
First,
if $\exists \; \breve{\bm{r}}\in\mathcal C, g_{K}(\bm{A},\breve{\bm{r}})^\top=-g_{K}(\bm{A}, \breve{\bm{r}})$,
all the elements in the diagonal should be zero,
i.e. $\{g_K(\bm{A}, \breve{\bm r})\}_{i,j}= 0$, with $i = j$.
Therefore,
there must be some non-zero elements in the non-diagonal indices,
i.e.,
$\exists (i,j)\in\{1\dots K\}^2:\;\{g_K(\bm A, \breve{\bm r})\}_{i,j}\neq 0 \wedge
i\neq j$.
We denote any one of the index in the non-diagonal indices as
$(b_1, b_2)$ with $b_1, b_2\in\{1\dots K\}$ and $b_1\neq b_2$.
Then we have
\begin{align*}
\{g_K(\bm A, \breve{\bm r})\}_{b_1,b_2} &= -\{g_K(\bm A, \breve{\bm r})\}_{b_2,b_1} \\
&=\text{sign}(A_{b_1,b_2})\cdot \breve{\bm r}[|A_{b_1,b_2}|]\neq 0,
\end{align*}
and $\text{sign}(A_{b_1,b_2})\neq 0$.
Next,
we assign 
$\acute{\bm{E}}$ with
\begin{equation}
\grave{\bm{E}}_i = 
\begin{cases}
\bm I &  i=b_1 \\
\breve{\bm G} & i=b_2 \\
\bm 0 & \text{otherwise}
\end{cases}.
\label{eq:lemma2:E}
\end{equation}
And $\acute{\bm{r}}$ with
\begin{equation}
\grave{\bm r}_i = 
\begin{cases}
\frac{\text{sign}(A_{b_1,b_2})}{2}\cdot\bm 1  &  i=|A_{b_1, b_2}| \\
\frac{\text{sign}(A_{b_2,b_1})}{2}\cdot\bm 1  &  i=|A_{b_2, b_1}| \\
\bm 0 & \text{otherwise}
\end{cases},
\label{eq:lemma2:r}
\end{equation}
which leads to
\begin{align}
&[g_K(\bm{A}, \grave{\bm{r}})]_{b_1,b_2}  = -[g_K(\bm{A}, \grave{\bm{r}})]_{b_2,b_1}
\notag 
\\
=& 
{\text{sign}(A_{b_1,b_2})}\cdot \text{diag}\big( \frac{\text{sign}(A_{b_1,b_2})}{2}\cdot{\bm 1}\big) 
\notag
\\
=& 
\frac{1}{2} \bm I.
\label{eq:lemma2:g}
\end{align}

Since $\breve{\bm G}$ is skew-symmetric,
we have
\begin{equation}
\breve{\bm G}^\top = - \breve{\bm G}.
\label{eq:lemma2:gssym}
\end{equation}
Based on \eqref{eq:lemma2:E}, \eqref{eq:lemma2:g} and \eqref{eq:lemma2:gssym}
we have
\begin{align*}
&{\grave{\bm{E}}}^\top g_K(\bm{A}, \grave{\bm{r}})\grave{\bm{E}} 
=  
\sum\nolimits_{i,j}^{K, K} {\grave{\bm{E}}_i}^\top [g_K(\bm{A}, \grave{\bm{r}})]_{i,j} \grave{\bm{E}}_j
\\
=& 
\bm 0 + 
{\grave{\bm{E}}_{b_1}}^\top [g_K(\bm{A}, \grave{\bm{r}})]_{b_1,b_2} \grave{\bm{E}}_{b_2}
+ {\grave{\bm{E}}_{b_2}}^\top [g_K(\bm{A}, \grave{\bm{r}})]_{b_2,b_1} \grave{\bm{E}}_{b_1}, 
\\
=&
\bm 0 +
\bm I^\top (\frac{1}{2}\bm I) \breve{\bm G} + \breve{\bm G}^\top (-\frac{1}{2}\bm I) \bm I \\
=&
\bm 0 + \frac{1}{2}\breve{\bm G} - \frac{1}{2}(-\breve{\bm G})
=
\breve{\bm{G}}.
\end{align*}

Hence,
there exist $\grave{\bm{E}}\in\mathbb R^{KN\times N}$ and  $\grave{\bm{r}}\in\mathbb R^{KN}$
such that any symmetric matrix $\breve{\bm{G}} = {\grave{\bm{E}}}^\top g_K(\bm{A}, \grave{\bm r})\grave{\bm{E}}$.
\end{proof}

Based on Lemma~\ref{lm:sym} and~\ref{lm:ssym},
we prove the following lemma for any real-valued square matrices.

\begin{lemma}
	\label{lm:matrix}
	If
	1) $\exists \; \hat{\bm{r}}\in\mathcal C,
	g_{K}(\bm{A},\hat{\bm{r}})^\top=g_{K}(\bm{A}, \hat{\bm{r}})$),
	and
	2) $\exists \; \breve{\bm{r}}\in\mathcal C,
	g_{K}(\bm{A},\breve{\bm{r}})^\top = -g_{K}(\bm{A},\breve{\bm{r}})$),
	then
	there exist $\bm E\in\mathbb R^{2KN\times N}$ and $\bm r\in\mathbb R^{2KN}$
	that any matrix $\bm G\in\mathbb R^{N\times N}$
	can be written as 
	\[G_{ht}= f(h,r,t) = \bm h^\top g_K(\bm{A}, \bm r) \bm t, \]
	where $\bm h = \bm E_{:,h}$, 
	$\bm t = \bm E_{:,t}$.
	%\footnote{[removed] blue part: wrong?}
\end{lemma}

\begin{proof}
In Lemma~\ref{lm:sym} and Lemma~\ref{lm:ssym},
we prove that any symmetric matrix can be factorized as $\hat{\bm{G}} = {\acute{\bm{E}}}^\top g_K(\bm{A}, \acute{\bm r})\acute{\bm{E}}$  
with $\acute{\bm{E}}\in\mathbb R^{KN\times N}$ and  $\acute{\bm{r}}\in\mathbb R^{KN}$
and any skew-symmetric matrix can be factorized as  
$\breve{\bm{G}} = {\grave{\bm{E}}}^\top g_K(\bm{A}, \grave{\bm{r}})\grave{\bm{E}} \in \mathbb R^{N\times N}$
with
$\grave{\bm{E}}\in\mathbb R^{KN\times N}$ and  $\grave{\bm{r}}\in\mathbb R^{KN}$.
In this part,
we show any square matrix $\bm G$
can be split as the sum of a particular $\hat{\bm{G}}$ and a particular $\breve{\bm{G}}$
and it can be factorized
in the bilinear form with the composition of $\acute{\bm{E}}, \grave{\bm{E}}$ and $\acute{\bm{r}}, \grave{\bm{r}}$.

We firstly composite
$\acute{\bm{E}}, \acute{\bm{r}}$ in the proof of Lemma~\ref{lm:sym}
and 
$\grave{\bm{E}}, \grave{\bm{r}}$ in the proof of Lemma~\ref{lm:ssym}
into
$\bm E\in\mathbb R^{2KN\times N}$ and  $\bm r\in\mathbb R^{2KN}$.
The basic idea is to add the symmetric part into odd rows and skew-symmetric part into even rows.
Specifically,
we define the rows of 
$\bm E$ as
\begin{equation}
\bm E_{i,:} = 
\begin{cases}
\acute{\bm{E}}_{\frac{i+1}{2},:} & i\textit{\,mod\,}2=1 \\
\grave{\bm{E}}_{\frac{i}{2},:} & i\textit{\,mod\,}2=0
\end{cases}.
\label{eq:lemma:matrix-E}
\end{equation}
And the relation embedding $\bm r$ is element-wise set as
\begin{equation}
\bm r[i] = 
\begin{cases}
\acute{\bm r}[{\frac{i+1}{2},:}] & i\textit{\,mod\,}2=1 \\
\grave{\bm r}[{\frac{i}{2},:}] & i\textit{\,mod\,}2=0
\label{eq:lemma:matrix-r}
\end{cases}.
\end{equation}
Based on the form of $\bm r$
we can have $g_K(\bm{A}, \bm r)\in\mathbb R^{2KN\times 2KN}$
where each element is formed by corresponding element in $g_K(\bm{A}, \acute{\bm{r}})$
or $g_K(\bm{A}, \grave{\bm{r}})$ as
\begin{align}
& \!\!\! \{g_K(\bm{A}, \bm r)\}_{i,j} 
\notag\\
\!\!\!= 
&\begin{cases}
\{g_K(\bm{A}, \acute{\bm{r}})\}_{\frac{i+1}{2},\frac{j+1}{2}}, & \!\!\!\!i\textit{\,mod\,}2=1 \text{ and } j\textit{\,mod\,}2=1 \\ 
\{g_K(\bm{A}, \grave{\bm{r}})\}_{\frac{i}{2},\frac{j}{2}}, & \!\!\!\!i\textit{\,mod\,}2=0 \text{ and } j\textit{\,mod\,}2=0 \\ 
0 & \!\!\!\!\text{otherwise}.
\end{cases}.
\label{eq:lemma:matrix-g}
\end{align}
Refer to the notation introduced at the beginning of Appendix~\ref{app:proof}, 
$\{g_K(\bm{A}, \bm r)\}_{i,j}$ represents the $(i,j)$-th element in the matrix,
while $[g_K(\bm{A}, \bm r)]_{i,j}$ represents the $(i,j)$-th block in previous parts.
Note that
such a construction of $g_K(\bm{A}, \bm r)$ will not violate the structure matrix $\bm{A}$.
The construction of \eqref{eq:lemma:matrix-E}, \eqref{eq:lemma:matrix-r} 
are graphically illustrated in the left part of Figure~\ref{fig:comp-matrix},
which leads to the construction of \eqref{eq:lemma:matrix-g} in the right part.

\begin{figure}[ht]
	\centering
	\includegraphics[height=3.6cm]{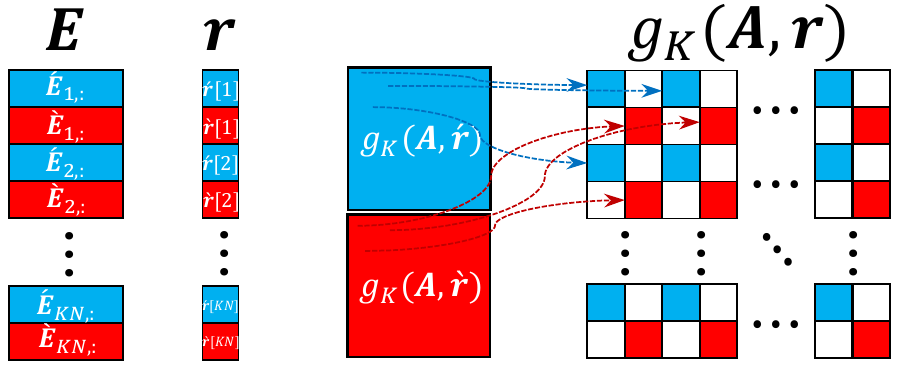}
		\vspace{-5px}
	\caption{Graphical illustration of the composed embeddings. The blue parts are from $\acute{\bm{E}}$, $\acute{\bm{r}}$ or $g_K(\bm{A}, \acute{\bm{r}})$ and the red parts are from $\grave{\bm{E}}$, $\grave{\bm{r}}$ or $g_K(\bm{A}, \grave{\bm{r}})$. The white spaces in $g_K(\bm{A}, \bm r)$ are zero values.}
	\label{fig:comp-matrix}
\end{figure}

Given any matrix $\bm G\in\mathbb R^{N\times N}$,
it can be split into a symmetric part $\hat{\bm{G}}$ and a skew-symmetric part $\breve{\bm{G}}$,
i.e.,
\begin{align}
\bm{G} 
= \underbrace{\frac{\bm G+\bm G^\top}{2}}_{\hat{\bm{G}}} 
+ \underbrace{\frac{\bm G-\bm G^\top}{2}}_{\breve{\bm{G}}}.
\label{eq:lemma:split}
\end{align}
Based on 
\eqref{eq:lemma:matrix-E},
\eqref{eq:lemma:matrix-r},
\eqref{eq:lemma:matrix-g},
and \eqref{eq:lemma:split},
$\forall h, t$,
we have 
\begin{align}
G_{ht} 
= &
[\hat{G}]_{ht} 
\! + \! [\breve{G}]_{ht}  
\! = \! {\acute{\bm{E}}_{:,h}}^\top g_K(\bm{A}, \acute{\bm{r}}) \acute{\bm{E}}_{:,t} 
\! + \! {\grave{\bm{E}}_{:,h}}^\top g_K(\bm{A}, \grave{\bm{r}}) \grave{\bm{E}}_{:,t},
\notag
\\
=& \sum\nolimits_{m,n}^{KN}\acute{E}_{m,h}\{g_K(\bm{A}, \acute{\bm{r}}) \}_{m,n} \acute{E}_{n,t} + 0 \label{eq:lemma:g5}
\\
&+ 0
+ \sum\nolimits_{m,n}^{KN}\grave{E}_{m,h}\{g_K(\bm{A}, \grave{\bm{r}}) \}_{m,n} \grave{E}_{n,t} 
\label{eq:lemma:g6} 
\\
=& \sum\nolimits_{(i,j) \textit{\,mod\,} 2=(1,1)}^{2KN,2KN}E_{i,h}\{g_K(\bm{A}, \bm r) \}_{i,j} E_{j,t} \label{eq:lemma:g1}  \\
&+ \sum\nolimits_{(i,j) \textit{\,mod\,} 2=(1,0)}^{2KN,2KN}E_{i,h}\{g_K(\bm{A}, \bm r) \}_{i,j} E_{j,t} \label{eq:lemma:g2} \\
&+ \sum\nolimits_{(i,j) \textit{\,mod\,} 2=(0,1)}^{2KN,2KN}E_{i,h}\{g_K(\bm{A}, \bm r) \}_{i,j} E_{j,t} \label{eq:lemma:g3} \\
&+ \sum\nolimits_{(i,j) \textit{\,mod\,} 2=(0,0)}^{2KN,2KN}E_{i,h}\{g_K(\bm{A}, \bm r) \}_{i,j} E_{j,t} \label{eq:lemma:g4} \\
=& \sum\nolimits_{i,j}^{2KN,2KN} E_{i,h}\{g_K(\bm{A}, \bm r) \}_{i,j} E_{j,t}  \notag
\\
=& {\bm E_{:,h}}^\top g_K(\bm{A}, \bm r) \bm E_{:,t}
= \bm h^\top g_K(\bm{A}, \bm r) \bm t
=
f(h,r,t), \notag
\end{align}
with $\bm h = \bm E_{:,h}, \bm t = \bm E_{:,t}$ and $\bm r$ in \eqref{eq:lemma:matrix-r}.
\eqref{eq:lemma:g2} and \eqref{eq:lemma:g3} are 0 since $\{ g_K(\bm{A}, \bm r) \}_{i,j}$ is zero when $i$ and $j$ are not simultaneously even or odd.
From \eqref{eq:lemma:g5} to \eqref{eq:lemma:g1}, we let $m=\nicefrac{i+1}{2}$ and $n=\nicefrac{j+1}{2}$
to get the odd part in \eqref{eq:lemma:matrix-E} and \eqref{eq:lemma:matrix-g}.
From \eqref{eq:lemma:g6} to \eqref{eq:lemma:g4}, we let $m=\nicefrac{i}{2}$ and $n=\nicefrac{j}{2}$
to obtain the even part in \eqref{eq:lemma:matrix-E} and \eqref{eq:lemma:matrix-g}.
\end{proof}

Finally,
we show a Lemma
which bridges
3 order tensor $\mathbf{G}$ with 
bilinear scoring function of form \eqref{eq:uni}.
Given any KG with tensor form $\mathbf G\in\mathbb R^{|\mathcal E|\times |\mathcal R|\times |\mathcal E|}$,
we denote $\bm G_r\in\mathbb R^{\mathcal E\times \mathcal E}$
as the $r$-th slice in the lateral of $\mathbf G$, i.e. $\bm G_r = \mathbf G_{\cdot,r, \cdot}$,
corresponding to relation $r$.

\begin{lemma}
\label{lm:tensor}
Given any KG with tensor form $\mathbf G\in\mathbb R^{|\mathcal E|\times |\mathcal R|\times |\mathcal E|}$,
and structure matrix $\bm{A}$.
If all the $\bm G_r$'s can be
independently expressed by a unique entity embedding matrices 
$\dot{\bm E}_r\in\mathbb R^{2KN\times N}$
and relation embedding $\dot{\bm r}\in\mathbb R^{2KN}$,
i.e. $\forall h,t =1\dots N, 
[G_r]_{h,t} = \dot{\bm h}_r^\top g_K(\bm{A}, \dot{\bm r}_r)\dot{\bm t}_r$,
with $\dot{\bm h}_r=[\dot{\bm E}_r]_{\cdot,h}, \dot{\bm t}_r=[\dot{\bm E}_r]_{\cdot, t}$,
then there exist entity embedding  $\bm E\in\mathbb R^{2KN|\mathcal R|\times N}$
and relation embedding  $\bm R\in\mathbb R^{2KN|\mathcal R|\times |\mathcal R|}$
such that
$\forall h,r,t, G_{hrt}= f_{\bm{A}}(h,r,t) = \bm h^\top g_K(\bm{A}, \bm r)\bm t$.
\end{lemma}

\begin{proof}
We show that computing $G_{hrt}$
can be independently expressed by $[G_r]_{ht}$ for each relation $r$.
Specifically,
we define the rows of entity embedding in $\bm E$ as
\begin{align}
\bm E_{i,:} = [\dot{\bm E}_{i \textit{\,mod\,} |\mathcal R|}]_{\lceil \frac{i}{|\mathcal R|} \rceil,:}
\label{eq:lemma:tensor-E}
\end{align}
And each element in the relation embedding $\bm r$ as 
\begin{align}
\bm r[i] = 
\begin{cases}
\dot{\bm r}[r]_{\lceil \frac{i}{|\mathcal R|} \rceil} & i \textit{\,mod\,} \mathcal R = r
\\
0 & \text{otherwise},
\end{cases},
\label{eq:lemma:tensor-r}
\end{align}
which leads to the element in $g_K(\bm{A}, \bm r)$  as
\begin{align}
&
\{g_K(\bm{A}, \bm r)\}_{i,j}
\notag  
\\ 
& =
\begin{cases}
\{g_K(\bm{A}, \dot{\bm r}_r)\}_{\lceil \frac{i}{|\mathcal R|} \rceil, \lceil \frac{j}{|\mathcal R|} \rceil} & 
i\textit{\,mod\,}|\mathcal R| = r
\textit{\;and\;}
j\textit{\,mod\,}|\mathcal R| = r
\\
0 & \text{otherwise}
\end{cases}.
\label{eq:lemma:tensor-g}
\end{align}

\begin{figure}[ht]
	\centering
	\includegraphics[height=3.5cm]{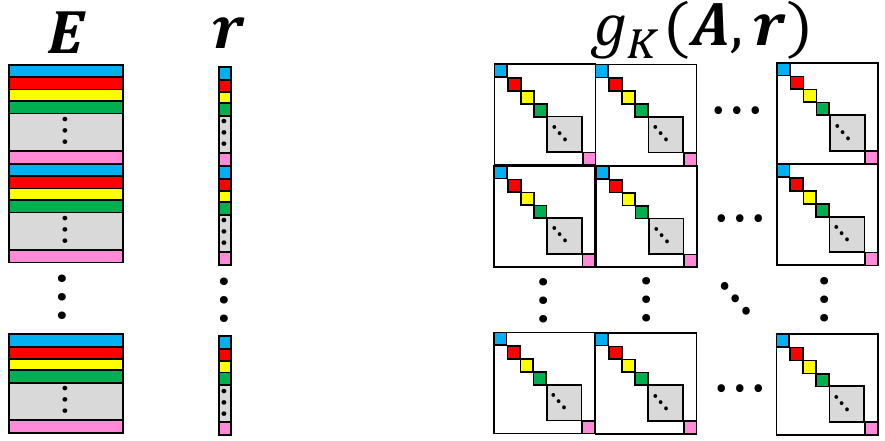}
	\vspace{-5px}
	\caption{Graphical illustration of the composed embeddings. Different colors represent components from different embeddings. White colors mean the zero values and gray colors represent spaces with mixed colors.}
	\label{fig:comp-tensor}
\end{figure}

The construction of \eqref{eq:lemma:tensor-E}, \eqref{eq:lemma:tensor-r} and \eqref{eq:lemma:tensor-g} can be graphically illustrated in Figure~\ref{fig:comp-tensor}.
Under these definitions,
we can get that each element $G_{hrt}$ can be expressed with
\begin{align}
G_{hrt} &=  [G_r]_{h,t} \nonumber \\
&= \sum_{i_r=1}^{2K|\mathcal E|}\sum_{j_r=1}^{2K|\mathcal E|} [E_{r}]_{i_r,h}\{g_K(\bm{A}, \bm r_r)\}_{i_r,j_r} [E_r]_{j_r,t} 
\label{eq:exp:stepsum}\\
&= \sum_{i=1}^{|\mathcal R|\cdot 2K|\mathcal E|}\sum_{j=1}^{|\mathcal R|\cdot 2K|\mathcal E|} E_{ih} \{g_K(\bm{A}, \bm r)\}_{ij} E_{jt}  
\label{eq:exp:allsum} \\
&=\bm h^{\top} g_K(\bm{A}, \bm r) \bm t  \nonumber
\end{align}
The step to get \eqref{eq:exp:stepsum} depends on Lemma~\ref{lm:matrix}.
Eq.~\eqref{eq:exp:stepsum} to \eqref{eq:exp:allsum} depends on \eqref{eq:lemma:tensor-E} and \eqref{eq:lemma:tensor-g}.
\end{proof}

\subsubsection{Proof of Proposition \ref{pr:expBLMs}}

\begin{proof}
Based on Lemma~\ref{lm:matrix} and Lemma~\ref{lm:tensor},
if
\begin{enumerate}[leftmargin=*]
	\item $\exists \hat{\bm{r}} \in \mathcal{C}$ 
	such that 
	$g_{K}(\bm{A},\hat{\bm{r}})$ is symmetric,
	and
	\item $\exists \breve{\bm{r}} \in \mathcal{C}$ 
	%\in\mathbb R^d$ 
	such that 
	$g_{K}(\bm{A},\breve{\bm{r}})$ is skew-symmetric,
\end{enumerate}
with 
$\mathcal{C} 
\equiv 
\{ \bm{r} \in\mathbb R^K \,|\, 
\bm r\neq \bm 0, 
r_i\in
\{0,\pm1,\dots, \pm K\},
i=1,\dots, K \},$,
then
given any KG with the tensor form $\mathbf G$,
there exist entity embedding $\bm E\in\mathbb R^{2K|\mathcal E||\mathcal R|\times |\mathcal E|}$
and relation embedding  $\bm R\in\mathbb R^{2K|\mathcal E||\mathcal R|\times |\mathcal R|}$
such that for all $h,r,t$, we have
\[G_{hrt}= f_{\bm{A}}(h,r,t) = \bm h^\top g_K(\bm{A}, \bm r)\bm t,\] 
with $\bm h = \bm E_{:,h}, \bm t = \bm E_{:,t}, \bm r=\bm E_{:,r}$.
Thus,
scoring function
\eqref{eq:uni} is fully expressive once condition 1) and 2) hold.
\end{proof}

\subsection{Proposition~\ref{pr:degenerate}}
\label{app:degenerate}

\subsubsection{Auxiliary Lemmas}

First,
we introduce two lemmas from matrix theory about the rank of Kronecker product
and the solution of equation group.

\begin{lemma}[\cite{horn2012matrix}]
\label{lem:kron}
Denote $\otimes$ as the Kronecker product,
given two matrices $\bm X, \bm Y$, we have
$\text{rank}(\bm X\otimes \bm Y) = \text{rank}(\bm X)\cdot \text{rank}(\bm Y)$.
\end{lemma}

\begin{lemma}[\cite{horn2012matrix}]
\label{lem:zero}
Given $\bm{A}\in\mathbb R^{d\times d}$,
there is no non-zero solution $\bm x\neq \bm0\in\mathbb R^d$ for
the equation group $\bm{A} \bm x = \bm 0$,
if and only if $\text{rank}(\bm{A}) = d$.
\end{lemma}

Note that the definition of degenerate structure 
$\bm{A}$ is that
	$\exists \bm h\neq 0, \bm h^\top g_K(\bm{A},\bm r)\bm t\neq0, \forall \bm r, \bm t$
	or 
	$\exists \bm r\neq 0, \bm h^\top g_K(\bm{A},\bm r)\bm t= 0, \forall \bm h, \bm t$.
To proof that $\bm{A}$ is not degenerate if and only if 
$\emph{rank}(\bm{A})=K$
and
$\{1,\dots, K\} \subset \{|A_{ij}|: i,j=1\dots K\}$,
we prove its converse-negative proposition
in Lemma~\ref{lem:con-neg}.

\begin{lemma}
	\label{lem:con-neg}
	$\exists \bm h\neq 0, \bm h^\top g_K(\bm{A},\bm r)\bm t\neq0, \forall \bm r, \bm t$
	or 
	$\exists \bm r\neq 0, \bm h^\top g_K(\bm{A},\bm r)\bm t= 0, \forall \bm h, \bm t$.
	if and only if
	$\text{rank}(\bm{A})<K$ or $\exists a\in\{1,\dots, K\}, a\notin \{|A_{ij}|:, i,j=1\dots K\}$.
\end{lemma}

This can be decomposed into two separate parts
in Lemma~\ref{lem:case1} and~\ref{lem:case2}.

\begin{lemma}
\label{lem:case1}
$\exists \bm h\neq 0, \bm h^\top g_K(\bm{A},\bm r)\bm t\neq0, \forall \bm r, \bm t$
if and only if $\text{rank}(\bm{A})<K$.
\end{lemma}

\begin{proof}
To begin with, 
we show the 
relationship between the rank of $g_{K}(\bm{A}, \bm r)$
and the rank of $\bm{A}$.
If we assign $\bm r=\bm 1\in\mathbb R^d$,
then the $(i,j)$-th block will be $[\bm g_{K}(\bm{A}, \bm 1)]_{ij} = \text{sign}(A_{ij})\cdot \bm{I}$
with the identity matrix $\bm{I}\in \mathbb R^{\frac{d}{K}\times \frac{d}{K}}$.
Using Kronecker product,
we can write $\bm g_{K}(\bm{A}, \bm 1)$ as a Kronecker product,
\begin{equation}
\bm g_{K}(\bm{A}, \bm 1) = \text{sign}(\bm{A})\otimes \bm{I},
\label{eq:grank1}
\end{equation}
where $\otimes$ here represents the Kronercker product 
and $\text{sign}(\bm{A})$ is a $K\times K$ matrix formed by the signs of elements in  $\bm{A}$.
Then, 
based on Lemma~\ref{lem:kron},
we have 
\begin{align}
\text{rank}\big(\bm g_{K}(\bm{A}, \bm 1)\big) 
= \nicefrac{d}{K} \cdot \text{rank}(\bm{A}),
\label{eq:grank2}
\end{align}
and
$\forall \bm r\in\mathbb R^d$,
$\text{rank}\big(\bm g_{K}(\bm{A}, \bm r)\big) 
\leq \nicefrac{d}{K} \cdot \text{rank}(\bm{A})$.

\begin{itemize}[leftmargin=*]
\item 	
First, we show the \textbf{sufficient condition},
i.e.,
if $\text{rank}(\bm{A})<K$, 
we have
$\exists \bm h\neq 0, \bm h^\top g_K(\bm{A},\bm r)\bm t\neq0, \forall \bm r, \bm t$.

Since $\bm{A}$ is not full rank,
we have $\text{rank}\big(\bm g_{K}(\bm{A}, \bm r)\big) < d$
based on \eqref{eq:grank2}.
Then based on Lemma~\ref{lem:zero},	
for all $\bm r\in\mathbb R^d$
there exists $\bm h\neq \bm 0$
that 
\[\bm g_{K}(\bm{A}, \bm r)^\top \bm h = \bm 0.\]
This leads to $\bm h^\top \bm g_{K}(\bm{A}, \bm r) = \bm 0$.
Thus, 
if $\text{rank}(\bm{A})<K$,
$\exists\bm h\neq \bm 0, \forall \bm r, \bm t, \bm h^\top \bm g_{K}(\bm{A}, \bm r)\bm t= 0$.

\item
Then we show the
\textbf{necessary condition},
i.e.,
if $\exists \bm h\neq 0, \bm h^\top g_K(\bm{A},\bm r)\bm t\neq0, \forall \bm r, \bm t$,
we have
$\text{rank}(\bm{A})<K$.

We assign
$\bm r=\bm 1$,
and a set of $\bm t$ with $(1,0,\dots, 0), (0,1,\dots,0), \dots, (0,0,\dots,1)$.
Then, 
this will lead to the following equation group
\begin{align*}
	\bm h^\top \{\bm g_{K}(\bm{A}, \bm 1)\}_{\cdot,1} &= 0,\\ 
	\bm h^\top \{\bm g_{K}(\bm{A}, \bm 1)\}_{\cdot,2} &= 0,\\
	&\cdots,\\
	\bm h^\top \{\bm g_{K}(\bm{A}, \bm 1)\}_{\cdot,d} &= 0.
\end{align*}

We proof the necessary  condition by contraction here.
Assume $\text{rank}(\bm{A})=K$, then $\text{rank}\big(\bm g_{K}(\bm{A}, \bm 1)^\top\big) =\text{rank}\big(\bm g_{K}(\bm{A}, \bm 1)\big)  =d$
based on \eqref{eq:grank1}.
Then, based on Lemma~\ref{lem:zero},
there is no $\bm h\neq \bm 0$ such that 
the above equation group is satisfied.
Thus, 
the assumption that $\text{rank}(\bm{A})=K$ is wrong.
Therefore,
if $\exists \bm h\neq 0, \bm h^\top g_K(\bm{A},\bm r)\bm t\neq0, \forall \bm r, \bm t$,
we have
$\text{rank}(\bm{A})<K$.
\end{itemize}

Based on the proof of sufficient and necessary conditions,
we have
$\exists \bm h\neq 0, \bm h^\top g_K(\bm{A},\bm r)\bm t\neq0, \forall \bm r, \bm t$
if and only if $\text{rank}(\bm{A})<K$.
\end{proof}

\begin{lemma}
	\label{lem:case2}
	$\exists \bm r\neq 0, \bm h^\top g_K(\bm{A},\bm r)\bm t= 0, \forall \bm h, \bm t$
	if and only if 
	$\exists a\in\{1,\dots, K\}, a\notin \{|A_{ij}|:, i,j=1\dots K\}$.
\end{lemma}

\begin{proof}
	\begin{itemize}[leftmargin=*]
		\item First, we show the sufficient condition, i.e.,
		if $\exists a\in\{1,\dots, K\}, a\notin \{|A_{ij}|:, i,j=1\dots K\}$,
		we have $\exists \bm r\neq 0, \bm h^\top g_K(\bm{A},\bm r)\bm t= 0, \forall \bm h, \bm t$.
		
		Given the $K$-chunk representation of $\bm r$,
		if $\exists$ $a$ $\in$ $\{1,\dots, K\}, a\notin \{|A_{ij}|:, i,j=1\dots K\}$,
		we assign 
			\begin{equation}
		\bm r_i = 
		\begin{cases}
		\bm 1 & i=a \\
		\bm 0 & i\neq a
		\end{cases}.
		\end{equation}
		Then,
		$\bm r_{|A_{ij}|}=\bm 0$ is always true 
		since $|A_{ij}|\neq a$.
		This leads to $g_K(\bm{A}, \bm r) = \bm 0$
		with $[g_K(\bm{A}, \bm r)]_{ij} = \text{sign}(A_{ij})\cdot \diag{(\bm r_{|A_{ij}|})}=\bm 0$.
		As a result,
		$\forall \bm h, \bm t,\bm h^\top g_K(\bm{A},\bm r)\bm t= 0$.
		Therefore,
		if $\exists a\in\{1,\dots, K\}, a\notin \{|A_{ij}|:, i,j=1\dots K\}$,
		we have $\exists \bm r\neq 0, \bm h^\top g_K(\bm{A},\bm r)\bm t= 0, \forall \bm h, \bm t$.
		
		\item Then, we show the \textbf{necessary condition}, i.e.,
		if $\exists \bm r\neq 0, \bm h^\top g_K(\bm{A},\bm r)\bm t= 0, \forall \bm h, \bm t$,
		we have $\exists a\in\{1,\dots, K\}, a\notin \{|A_{ij}|:, i,j=1\dots K\}$.
		
		We can enumerate $\bm h, \bm t$
		as the set of unit vectors with one dimension as $1$
		and the remaining to be $0$.
		Then from $\bm h^\top g_K(\bm{A},\bm r)\bm t= 0$
		we derive $g_K(\bm{A},\bm r)=\bm 0$
		since any element is $0$.
		Specially,
		we have that each block in $g_K(\bm{A},\bm r)$ is
		\[
		[g_K(\bm{A}, \bm r)]_{ij} = \text{sign}(A_{ij})\cdot \diag{(\bm r_{|A_{ij}|})} = \bm 0.
		\]
		If the number of unique values of set $\{ A_{ij} \}$ is $K$,
		we will have $\bm{r} = 0$.
		This is in contrary to the fact that
		$\bm r\neq \bm 0$.
		Thus there must 
		$\exists a\in\{1,\dots, K\}, a\notin \{|A_{ij}|:, i,j=1\dots K\}$.

	\end{itemize}
\noindent
Thus, 
we have $\exists \bm r\neq 0, \bm h^\top g_K(\bm{A},\bm r)\bm t= 0, \forall \bm h, \bm t$
if and only if 
$\exists a\in\{1,\dots, K\}, a\notin \{|A_{ij}|:, i,j=1\dots K\}$.
\end{proof}

By combining Lemma~\ref{lem:case1} and Lemma~\ref{lem:case2},
we can show Lemma~\ref{lem:con-neg} that
$\exists \bm h\neq 0, \bm h^\top g_K(\bm{A},\bm r)\bm t\neq0, \forall \bm r, \bm t$
or 
$\exists \bm r\neq 0, \bm h^\top g_K(\bm{A},\bm r)\bm t= 0, \forall \bm h, \bm t$.
if and only if
$\text{rank}(\bm{A})<K$ or $\exists a\in\{1,\dots, K\}, a\notin \{|A_{ij}|:, i,j=1\dots K\}$.
Since
the original statement is equal to the converse-negative proposition,
Proposition~\ref{pr:degenerate} is proved.

\subsubsection{Proof of Proposition~\ref{pr:degenerate}}

\begin{proof}
Proposition~\ref{pr:degenerate},
is equivalent to the statement that
$\bm{A}$ is degenerate if and only if 
$\text{rank}(\bm{A})<K$
and $\exists a\in\{1,\dots, K\}, a\notin \{|A_{ij}|:, i,j=1\dots K\}$.
From Definition~\ref{def:degenerate},
$\bm{A}$ is degenerate means
\begin{enumerate}[leftmargin=*]
	\item $\exists \bm h\neq 0, \bm h^\top g_K(\bm{A},\bm r)\bm t\neq0, \forall \bm r, \bm t$; or
	\item $\exists \bm r\neq 0, \bm h^\top g_K(\bm{A},\bm r)\bm t= 0, \forall \bm h, \bm t$.
\end{enumerate}
Here, 
Lemma~\ref{lem:case1} proves 1) and Lemma~\ref{lem:case2} proves 2).
Thus,
we can conclude that
$\bm{A}$ is non-degenerate,
if and only if $\text{rank}(\bm{A})=K$
and 
$\{1, \dots, K\} \subset \{|A_{ij}|: i,j=1\dots K\}$.
\end{proof}

\subsection{Proposition~\ref{pr:equiv}}
\label{app:equiv}

We denote the $K$-chunk representation of the embeddings as 
$\bm E = [\bm E_1; \dots; \bm E_K]$
with $\bm E_k\in\mathbb R^{\frac{d}{K}\times |\mathcal E|}, k=1\dots K$
and 
$\bm R = [\bm R_1; \dots; \bm R_K]$
with $\bm R_k\in\mathbb R^{\frac{d}{K}\times |\mathcal E|}, k=1\dots K$.
Besides,
given the permutation matrix $\bm \Pi$,
we denote 
$\pi(i) = j$ 
and $\pi^{-1}(j)=i$
if $\Pi_{ij}=1$ 
for $i,j=1\dots K$.

\subsubsection{Auxiliary Lemmas}

\begin{lemma}
\label{lem:optimal}
If below two conditions hold,
then $\bm{A} \equiv \bm{A}'$.
\begin{itemize}[leftmargin=*]
\item given the optimal embedding 
$\bm E^*$ and $\bm R^*$ for $f_{\bm{A}}(h,r,t)$
there exist $\bm E'$ and $\bm R'$ such that $f_{\bm{A}}(h,r,t) = f_{\bm{A}'}(h,r,t)$ always hold;

\item 
given the optimal embedding 
${\bm E'}^*$ and ${\bm R'}^*$ for $f_{\bm{A}'}(h,r,t)$,
there exist $\bm E$ and $\bm R$ such that $f_{\bm{A}}(h,r,t) = f_{\bm{A}'}(h,r,t)$ always hold.
\end{itemize}
\end{lemma}

\begin{proof}
	Denote $\bm P^*=\{\bm E^*, \bm R^*\}, \bm P=\{\bm E, \bm R\}$
	and ${\bm P'}^*=\{{\bm E'}^*, {\bm R'}^*\}, \bm P'=\{\bm E', \bm R'\}$	
	Then, 
	from Definition~\ref{def:equiv},
	we have 
	$\bm P^* = \arg\max_{\bm P} M(F(\bm P; \bm{A}), \mathcal S)$
	and 
	${\bm P'}^*$ $=$  $\arg\max_{\bm P'}$  $M(F(\bm P'; \bm{A}'), \mathcal S)$.
	
	If given the optimal embedding 
	$\bm E^*$ and $\bm R^*$ for the scoring function $f_{\bm{A}}(h,r,t)$
	there exist $\bm E'$ and $\bm R'$ such that $f_{\bm{A}}(h,r,t) = f_{\bm{A}'}(h,r,t)$,
	we will have
	\begin{align}
	 M\big(F(\bm P^*; \bm{A}), \mathcal S\big) 
	 & = M\big(F(\bm P'; \bm{A}'),\mathcal S\big),
	 \notag
	 \\ 
	 & \leq M\big(F({\bm P'}^*; \bm{A}'), \mathcal S\big).
	 \label{eq:p-opt}
	\end{align}
	
	Similarly,
	if given the optimal embedding 
	${\bm E'}^*$ and ${\bm R'}^*$ for the scoring function $f_{\bm{A}'}(h,r,t)$
	there exist $\bm E$ and $\bm R$ such that $f_{\bm{A}}(h,r,t) = f_{\bm{A}'}(h,r,t)$,
	we have
	\begin{align}
	M\big(F({\bm P'}^*; \bm{A}'),\mathcal S\big) 
	& = M\big(F(\bm P; \bm{A}),\mathcal S\big),
	\notag
	\\
	& \leq M\big(F({\bm P}^*; \bm{A}), \mathcal S\big).
	\label{eq:p'-opt}
	\end{align}
	Based on \eqref{eq:p-opt} and \eqref{eq:p'-opt},
	we have $M\big(F(\bm P^*; \bm{A}),\mathcal S\big) = M\big(F({\bm P'}^*; \bm{A}'),\mathcal S\big)$,
	namely $\bm{A} \equiv \bm{A}'$.
\end{proof}

\subsubsection{Proof of Proposition~\ref{pr:equiv}}

\begin{proof}
The key point of this proof is that,
there exist corresponding operations on the optimal embedding such that the score of equivalent structures
can always be the same,
i.e.,
$f_{\bm{A}'}(h,r,t) = f_{\bm{A}}(h,r,t)$.
\begin{enumerate}[label=(\roman*),leftmargin=15px]
\item 
We can permute the corresponding chunks in the entity embeddings 
to get the same scores.
If there exists a permutation matrix $\bm \Pi\in\mathbb \{0,1\}^{K\times K}$ that $\bm{A}' = \bm \Pi^{\top} \bm{A}\bm \Pi$,
we will have $A'_{ij} = A_{\pi(i),\pi(j)}$ and $A_{ij}= A'_{\pi^{-1}(i), \pi^{-1}(j)}$
for $i,j=1\dots K$.

Given $\bm E^*, \bm R^*$ as the optimal embeddings trained by $f_{\bm{A}}(h, r, t)$,
we set $\bm E', \bm R'$
with $\bm E'_i = \bm E_{\pi(i)}^*,i=1\dots K, \bm R'=\bm R^*$.
In this way,
we always have 
\begin{align*}
&f_{\bm{A}'}(h,r,t) 
\\
&= \sum\nolimits_{i'=1}^K\!\!\sum\nolimits_{j'=1}^K\!\!\!\! \text{sign}(A'_{i'j'})\langle \bm h_{i'}, \bm r_{|A_{i'j'}|},\bm t_{j'}  \rangle
\\
&=  \sum\nolimits_{i'=1}^K\!\!\sum\nolimits_{j'=1}^K \!\!\!\!\text{sign}(A_{\pi(i'),\pi(j')})\langle\bm h_{\pi(i')}^*, \bm r_{|A_{\pi(i'),\pi(j')}|}^*,\bm t_{\pi(j')}^*\rangle, 
\\
&= \sum\nolimits_{i=1}^K\!\sum\nolimits_{j=1}^K \!\!\text{sign}(A_{ij})\langle\bm h_{i}^*, \bm r_{|A_{ij}|}^*,\bm t_{j}^*\rangle, 
\\
&= f_{\bm{A}}(h,r,t).
\end{align*}
In the third to fourth line, we set $i=\pi(i')$ and $j=\pi(j')$.

Similarly,
given ${\bm E'}^*, {\bm R'}^*$ as the optimal embeddings trained by $f_{\bm{A}'}(h,  r,  t)$,
we can set $\bm E, \bm R$
with  $\bm E_i = {\bm E'}_{\pi^{-1}(i)}^*, i=1\dots K, \bm R={\bm R'}^*$.
Thus, we always have
\begin{align*}
&f_{\bm{A}}(h,r,t) 
\\
&= \sum\nolimits_{i=1}^K\sum\nolimits_{j=1}^K \text{sign}(A_{ij})\langle\bm h_{i}, \bm r_{|A_{ij}|},\bm t_{j}\rangle,
\\
&= \sum\nolimits_{i=1}^K\sum\nolimits_{j=1}^K \text{sign}(A'_{\pi^{-1}(i),\pi^{-1}(j)})
\\
&~~~~~~~~~~~~~~~~~~~~~~~~~~~\cdot\langle{\bm h'}_{\pi^{-1}(i)}^*, {\bm r'}_{|{A'}_{\pi^{-1}(i),\pi^{-1}(j)}|}^*,{\bm t'}_{\pi^{-1}(j)}^*\rangle, 
\\
&= \sum\nolimits_{i'=1}^K\!\sum\nolimits_{j'=1}^K\! \text{sign}(A'_{i'j'})\langle {\bm h'}_{i'}^*, {\bm r'}_{|A_{i'j'}|}^*,{\bm t'}_{j'}^*  \rangle
\\
&= f_{\bm{A}'}(h,r,t).
\end{align*}
In the third to fourth line, we set $i'=\pi^{-1}(i)$ and $j'=\pi^{-1}(j)$.
Finally,
based on Lemma~\ref{lem:optimal},
we have $\bm{A}\equiv \bm{A}'$.

\item 
We can permute the corresponding chunks in relation embedding
to get the same scores.

If there exists a permutation matrix $\bm \Pi\in\{0,1\}^{K\times K}$ that 
$\Phi_{\bm{A}'} = \bm \Pi\Phi_{\bm{A}}$,
we will have 
$|A_{i,j}| = \pi(|A'_{ij}|), |A'_{ij}| = \pi^{-1}(|A_{ij}|)$
and 
$\text{sign}(A'_{ij}) = \text{sign}(A_{ij})$.

Given $\bm E^*, \bm R^*$
as the optimal embeddings trained by $f_{\bm{A}}(h,r,t)$,
we can set $\bm E', \bm R'$ with
$\bm E'=\bm E^*, \bm R'_i = \bm R_{\pi^{-1}(i)}^*, i=1\dots K$.
In this way,
we alway have 	
\begin{align*}
&f_{\bm{A}'}(h,r,t)
\\
&= \sum\nolimits_{i'=1}^K\sum\nolimits_{j'=1}^K \text{sign}(A'_{i'j'})\langle\bm h'_{i'}, \bm r'_{|A'_{i'j'}|},\bm t'_{j'} \rangle, 
\\
&= \sum\nolimits_{i'=1}^K\sum\nolimits_{j'=1}^K \text{sign}(A_{i'j'})\langle\bm h_{i'}^*, \bm r_{\pi^{-1}(|A'_{i'j'}|)}^*,\bm t_{j'}^* \rangle,
\\
&= \sum\nolimits_{i'=1}^K\sum\nolimits_{j'=1}^K \text{sign}(A_{i'j'})\langle\bm h_{i'}^*, \bm r_{|A_{i'j'}|}^*,\bm t_{j'}^* \rangle,
\\
&= f_{\bm{A}}(h,r,t),
\end{align*}
with $|A_{i'j'}| = \pi^{-1}(|A_{i'j'}|)$

Similarly,
given ${\bm E'}^*, {\bm R'}^*$ as the optimal embeddings trained by $f_{\bm{A}'}(h,  r,  t)$,
we can set $\bm E, \bm R$
with  $\bm E = {\bm E'}^*, \bm R_i={\bm R'}_{\pi(i)}^*, i=1\dots K$.
In this way,
we always have
\begin{align*}
&f_{\bm{A}}(h,r,t) 
\\
&= \sum\nolimits_{i=1}^K\sum\nolimits_{j=1}^K \text{sign}(A_{ij})\langle\bm h_{i}, \bm r_{|A_{ij}|},\bm t_{j}\rangle,
\\
&= \sum\nolimits_{i=1}^K\sum\nolimits_{j=1}^K \text{sign}(A'_{ij})
\cdot\langle{\bm h'}_{i}^*, {\bm r}_{\pi(|A'_{ij}|)}^*,{\bm t'}_{j}^*\rangle, 
\\
&= \sum\nolimits_{i=1}^K\sum\nolimits_{j=1}^K \text{sign}(A'_{ij})
\cdot\langle{\bm h'}_{i}^*, {\bm r'}_{|A_{ij}|}^*,{\bm t'}_{j}^*\rangle, 
\\
&= f_{\bm{A}'}(h,r,t),
\end{align*}
with $|A_{i,j}| = \pi(|A'_{ij}|)$.
Finally,
based on Lemma~\ref{lem:optimal},
we have $\bm{A}\equiv \bm{A}'$.

\item 
We can flip the signs of corresponding chunks in relation embedding
to get the same scores.

If there exists a sign vector $\bm s\in \{\pm 1\}^K$
that
$[\bm \Phi_{\bm{A}'}]_{k,\cdot} = s_k \cdot [\bm \Phi_{\bm{A}}]_{k,\cdot}, \forall k=1\dots K$,
we will have
$A'_{ij} = s_k\cdot A_{ij}$ and $A_{ij} = s_k\cdot A'_{ij}$ with $k=|A_{ij}|=|A'_{ij}|$
and $s_k\in\{\pm1\}$.

Given $\bm E^*, \bm R^*$
as the optimal embedding trained by 
$f_{\bm{A}}(h,r,t)$,
we can set $\bm E', \bm R'$ with
$\bm E'=\bm E^*, \bm R'_k = s_k\cdot\bm R_{k}^*, k=1\dots K$.
In this way,
we always have

\begin{figure*}[ht]
	\centering
	\subfigure[ogbl-biokg (AutoBLM)]{\includegraphics[width=0.31\columnwidth]{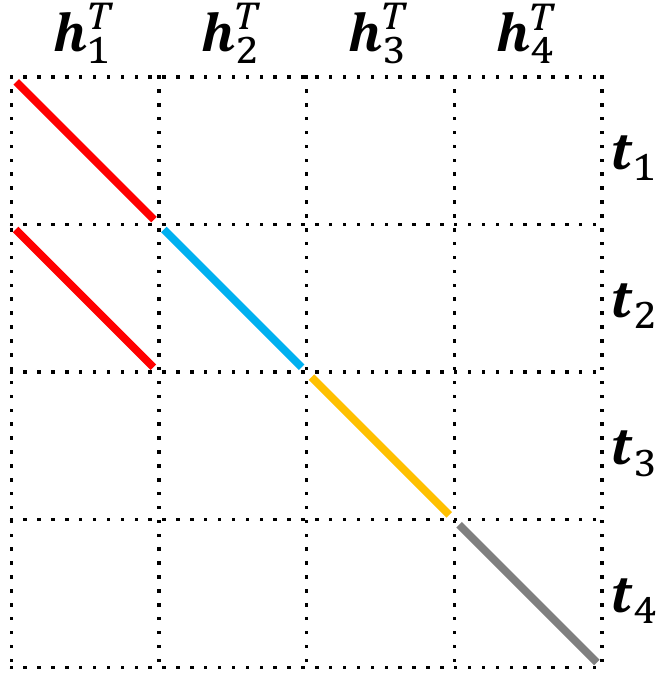}}
	\qquad
	\subfigure[ogbl-wikikg2 (AutoBLM)]{\includegraphics[width=0.31\columnwidth]{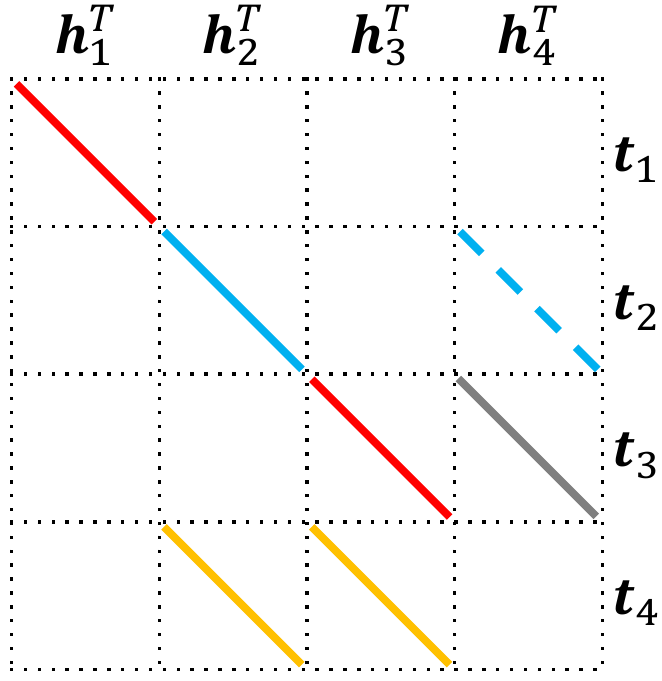}}
	\qquad\qquad
	\subfigure[ogbl-biokg (AutoBLM+)]{\includegraphics[width=0.31\columnwidth]{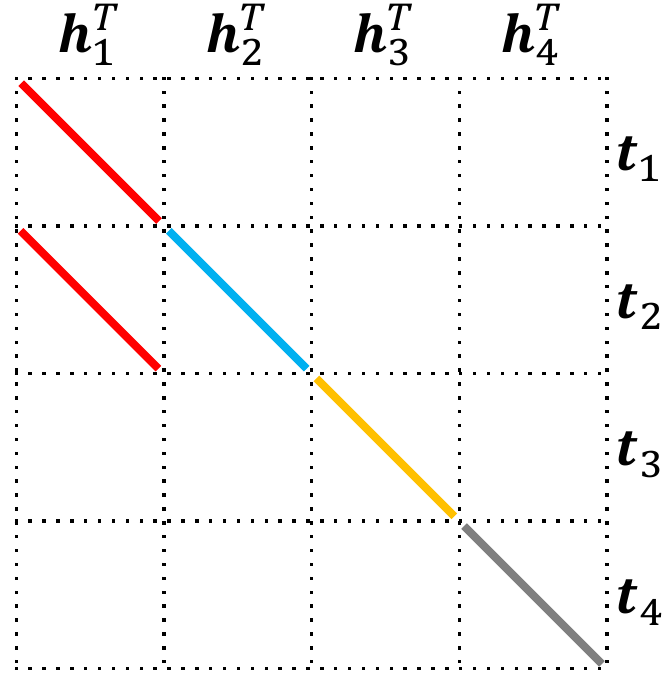}}
	\qquad
	\subfigure[ogbl-wikikg2 (AutoBLM+)]{\includegraphics[width=0.31\columnwidth]{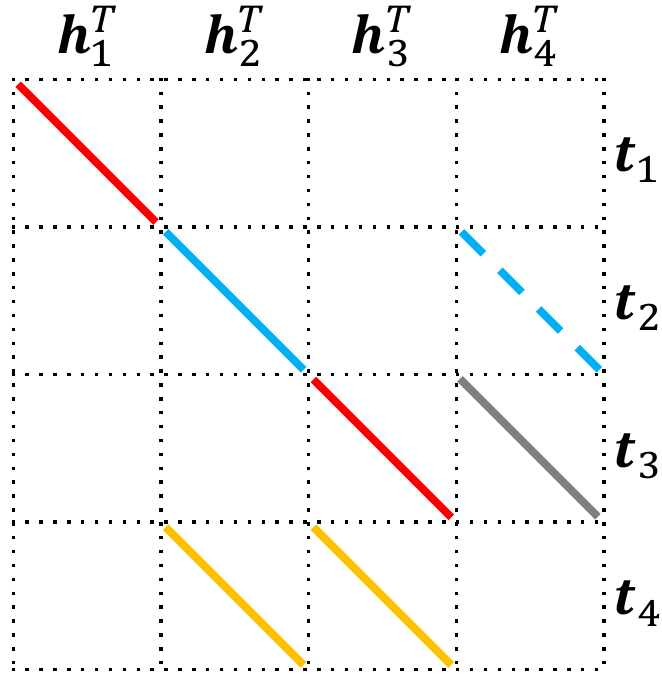}}
	\vspace{-10px}
	\caption{A graphical illustration of $g_4(\bm{A}, \bm r)$ identified by AutoBLM and AutoBLM+ 
		on the large-scale KG completion task with ogbl-biokg and ogbl-wikikg2 datasets.}
	\label{fig:structure:ogb}
\end{figure*}

\begin{figure*}[ht]
	\begin{center}
		%\subfigure[FB15k-2p.]
		{\includegraphics[width=0.31\columnwidth]{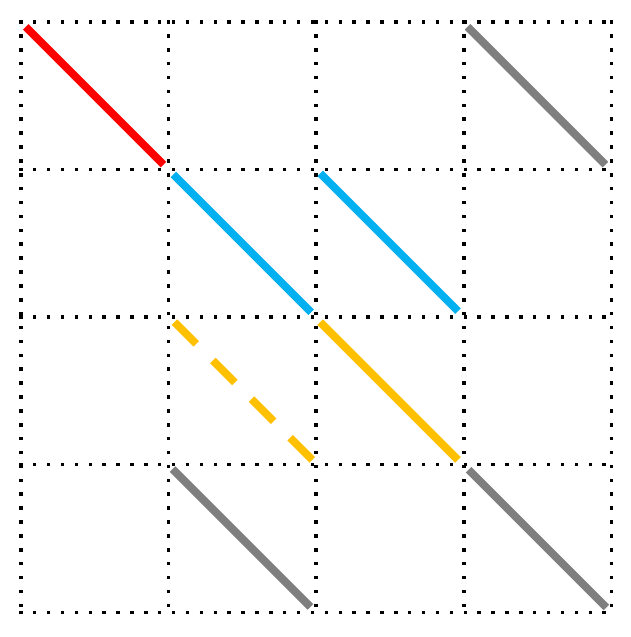}}
		\qquad
		%\subfigure[FB15k-3p.]
		{\includegraphics[width=0.31\columnwidth]{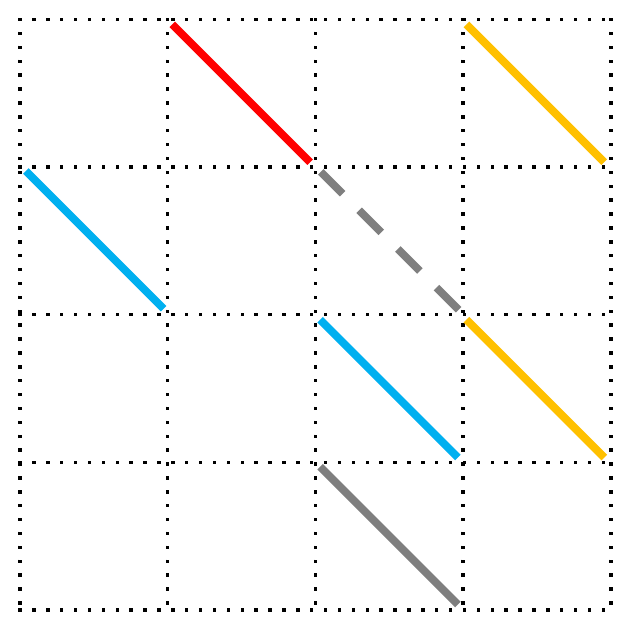}}
		\qquad\qquad
		%\subfigure[FB15k237-2p.]
		{\includegraphics[width=0.31\columnwidth]{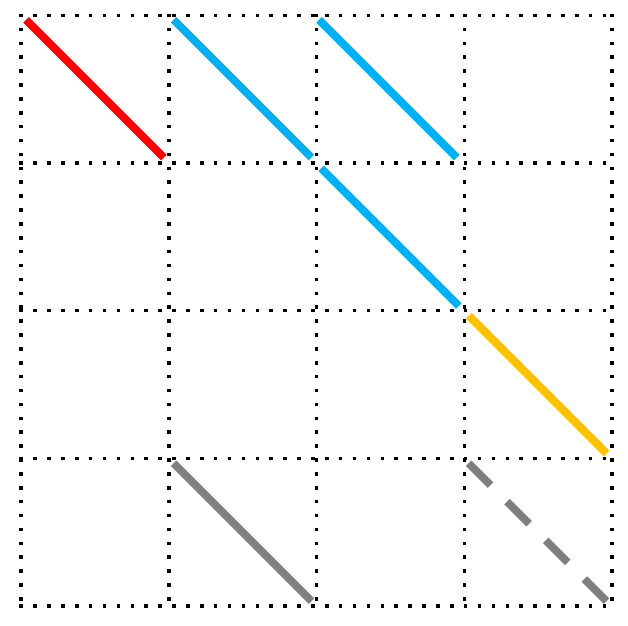}}
		\qquad
		%\subfigure[FB15k237-3p.]
		{\includegraphics[width=0.31\columnwidth]{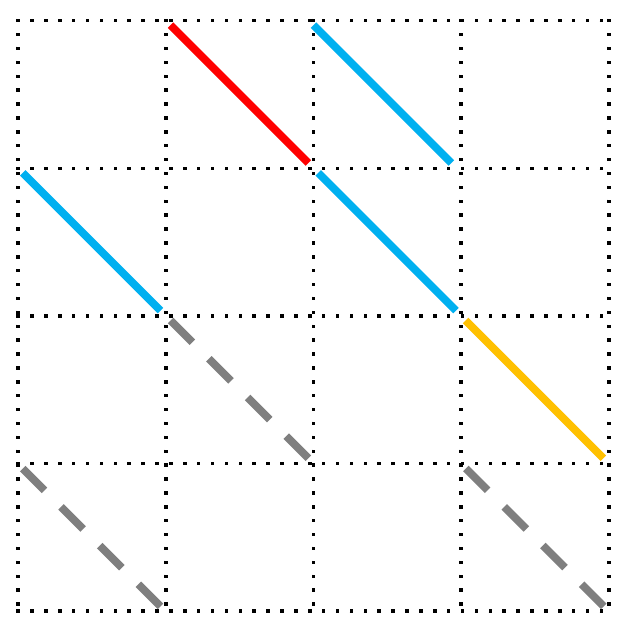}}

		\subfigure[FB15k-2p.]
		{\includegraphics[width=0.31\columnwidth]{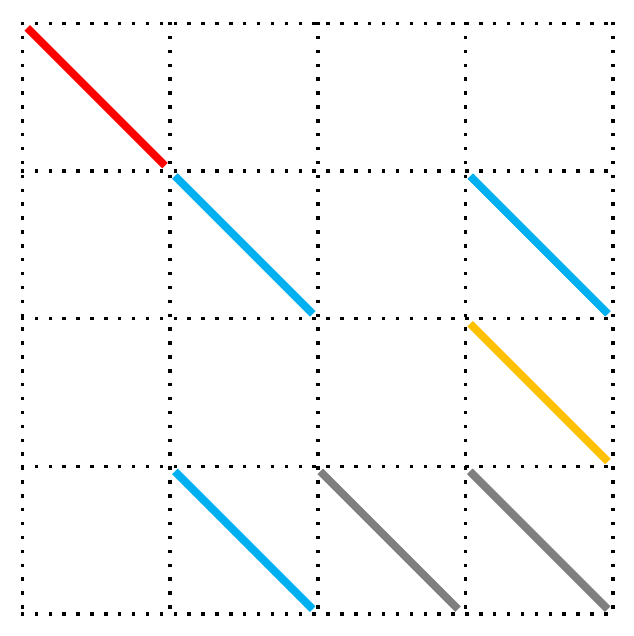}}
		\qquad
		\subfigure[FB15k-3p.]
		{\includegraphics[width=0.31\columnwidth]{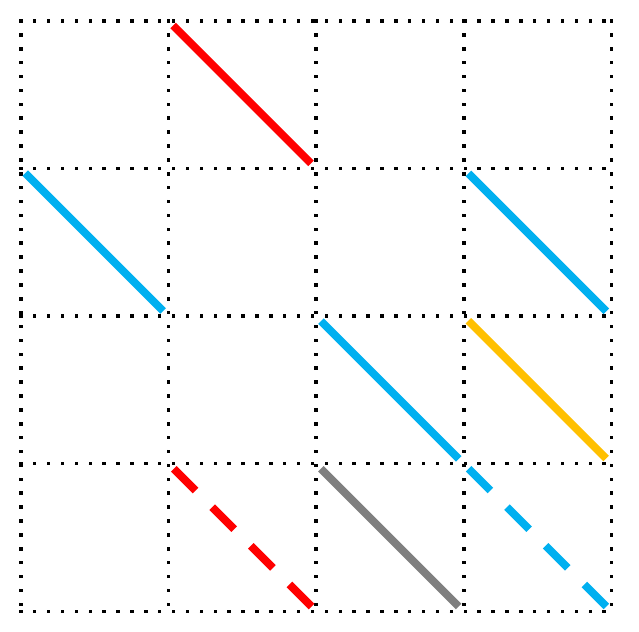}}
		\qquad\qquad
		\subfigure[FB15k237-2p.]
		{\includegraphics[width=0.31\columnwidth]{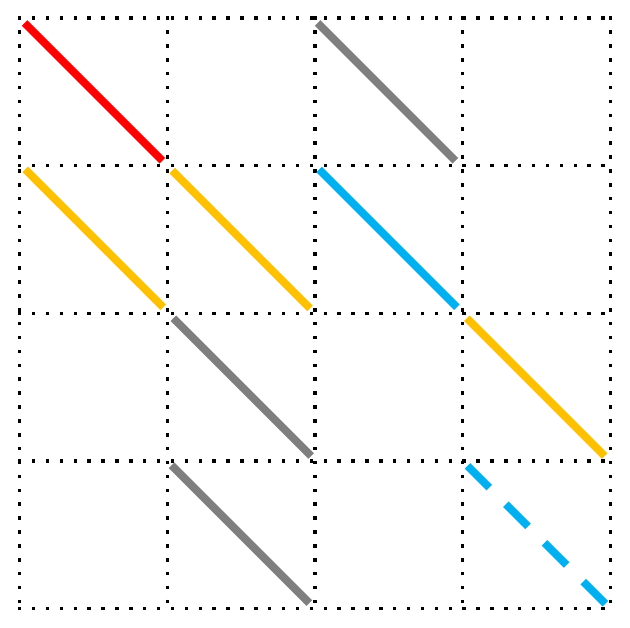}}
		\qquad
		\subfigure[FB15k237-3p.]
		{\includegraphics[width=0.31\columnwidth]{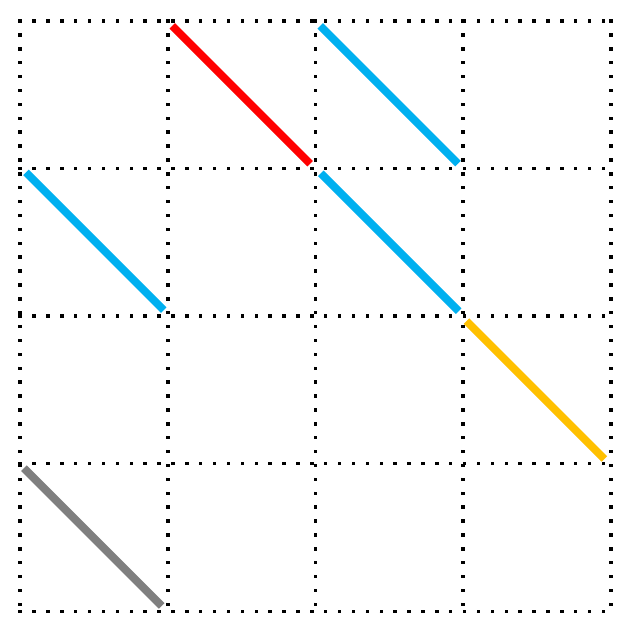}}
		\vspace{-10px}
		\caption{A graphical illustration of $g_4(\bm{A}, \bm r)$ identified by AutoBLM
			(top) and AutoBLM+ (bottom) on the multi-hop query task.}
		\label{fig:s1}
	\end{center}
\end{figure*}

\begin{figure*}[ht]
	\centering
	%\subfigure[AIFB.]
	{\includegraphics[width=0.31\columnwidth]{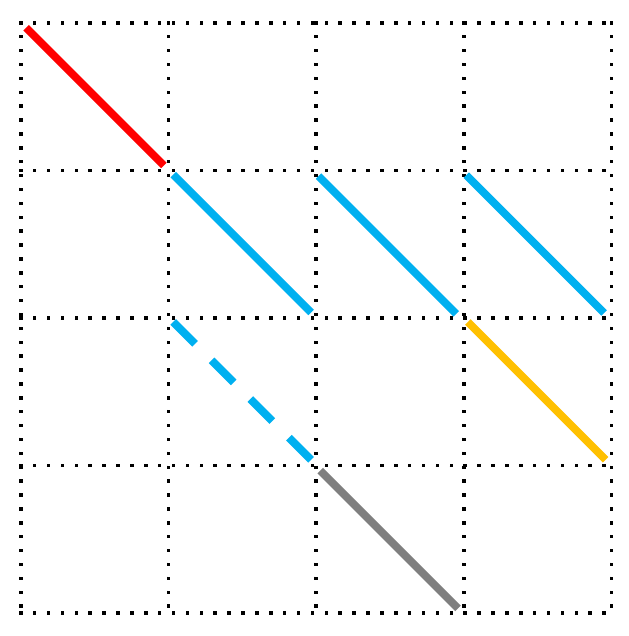}}
	\qquad\qquad
	%\subfigure[MUTAG.]
	{\includegraphics[width=0.31\columnwidth]{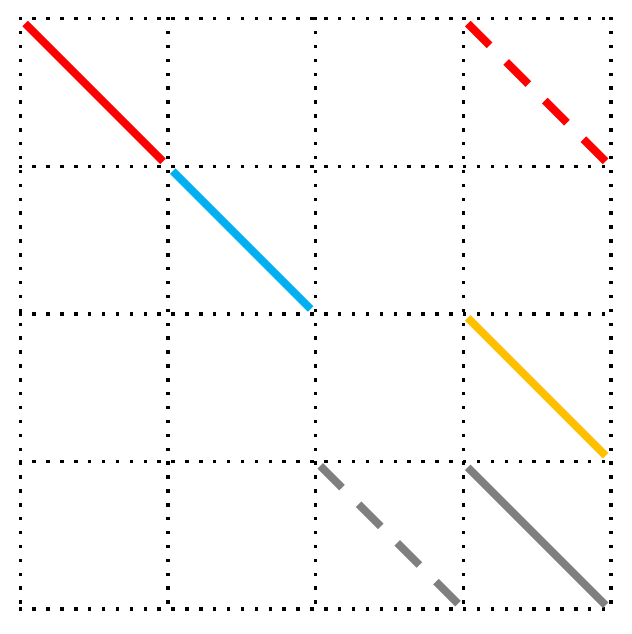}}
	\qquad\qquad
	%\subfigure[BGS.]
	{\includegraphics[width=0.31\columnwidth]{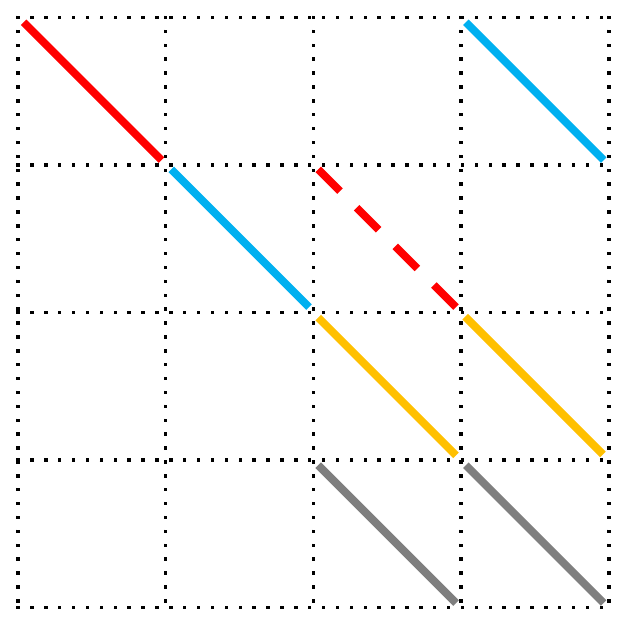}}

	\subfigure[AIFB.]
	{\includegraphics[width=0.31\columnwidth]{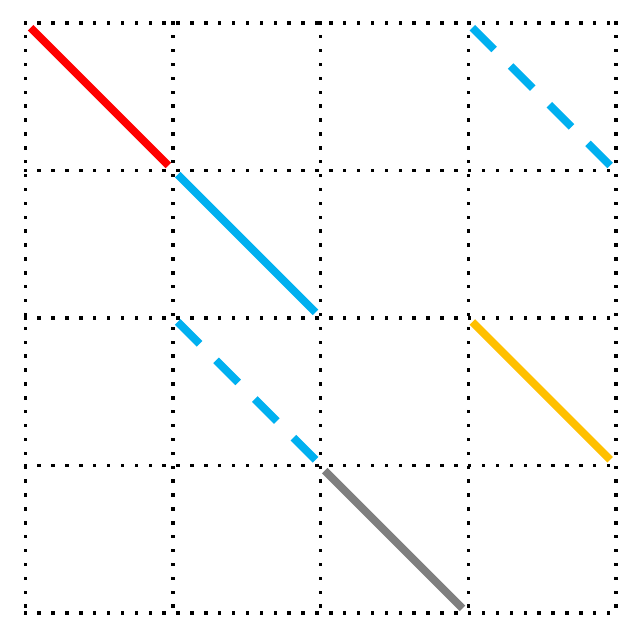}}
	\qquad\qquad
	\subfigure[MUTAG.]
	{\includegraphics[width=0.31\columnwidth]{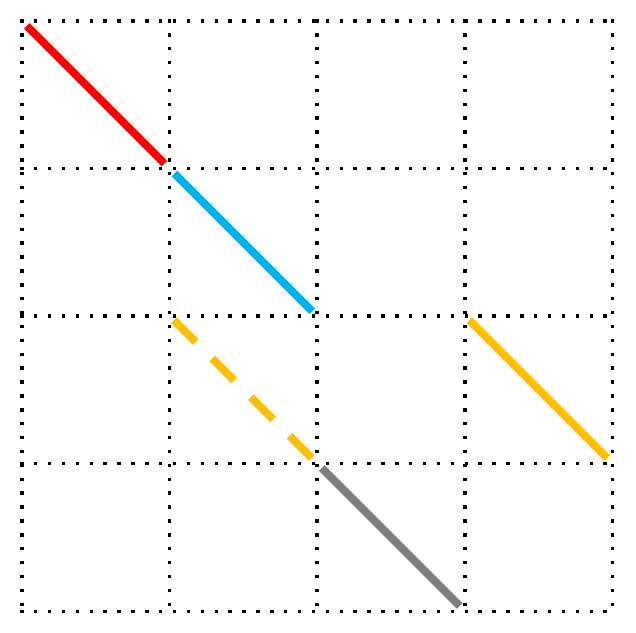}}
	\qquad\qquad
	\subfigure[BGS.]
	{\includegraphics[width=0.31\columnwidth]{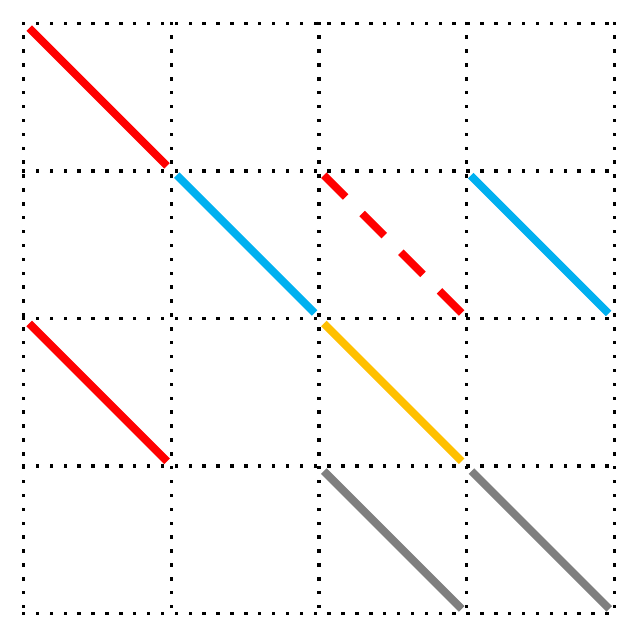}}
	\vspace{-10px}
	\caption{A graphical illustration of $g_4(\bm{A}, \bm r)$ identified by
		AutoBLM (top) and AutoBLM+ (bottom) on the entity classification task.}
	\label{fig:s2}
	\vspace{-10px}
\end{figure*}

\begin{align*}
&f_{\bm{A}'}(h,r,t) 
\\
&= \sum\nolimits_{i=1}^K\sum\nolimits_{j=1}^K \text{sign}(A'_{ij})\langle\bm h'_{i}, \bm r'_{|A'_{ij}|},\bm t'_{j}\rangle,
\\
&= \sum\nolimits_{i=1}^K\sum\nolimits_{j=1}^K s_{|A_{ij}|}\cdot\text{sign}(A_{ij})
\cdot\langle{\bm h}_{i}^*, s_{|A_{ij}|}\cdot{\bm r}_{|A_{ij}|}^*,{\bm t}_{j}^*\rangle, 
\\
&= \sum\nolimits_{i=1}^K\sum\nolimits_{j=1}^K s_{|A_{ij}|}^2\cdot\text{sign}(A_{ij})
\cdot\langle{\bm h}_{i}^*, {\bm r}_{|A_{ij}|}^*,{\bm t}_{j}^*\rangle, 
\\
&= \sum\nolimits_{i=1}^K\sum\nolimits_{j=1}^K \text{sign}(A_{ij})
\cdot\langle{\bm h}_{i}^*, {\bm r}_{|A_{ij}|}^*,{\bm t}_{j}^*\rangle, 
\\
&= f_{\bm{A}}(h,r,t),
\end{align*}
with $s_{|A_{ij}|}^2=1$.

Similarly,
given ${\bm E'}^*, {\bm R'}^*$ as the optimal embeddings trained by $f_{\bm{A}'}(h,  r,  t)$,
we can set $\bm E, \bm R$
with  $\bm E = {\bm E'}^*, \bm R_k=s_k\cdot {\bm R'}_{k}^*, k=1\dots K$.
In this way,
we always have
\begin{align*}
&f_{\bm{A}}(h,r,t) 
\\
&= \sum\nolimits_{i=1}^K\sum\nolimits_{j=1}^K \text{sign}(A_{ij})\langle\bm h_{i}, \bm r_{|A_{ij}|},\bm t_{j}\rangle,
\\
&= \sum\nolimits_{i=1}^K\sum\nolimits_{j=1}^K s_{|A'_{ij}|}\cdot \text{sign}(A'_{ij})
\cdot\langle{\bm h'}_{i}^*,  s_{|A'_{ij}|}\cdot{\bm r'}_{|A'_{ij}|}^*,{\bm t'}_{j}^*\rangle, 
\\
&= \sum\nolimits_{i=1}^K\sum\nolimits_{j=1}^K s_{|A'_{ij}|}^2\cdot \text{sign}(A'_{ij})
\cdot\langle{\bm h'}_{i}^*, {\bm r'}_{|A'_{ij}|}^*,{\bm t'}_{j}^*\rangle, 
\\
&= \sum\nolimits_{i=1}^K\sum\nolimits_{j=1}^K \text{sign}(A'_{ij})
\cdot\langle{\bm h'}_{i}^*, {\bm r'}_{|A'_{ij}|}^*,{\bm t'}_{j}^*\rangle, 
\\
&= f_{\bm{A}'}(h,r,t).
\end{align*}
Finally,
based on Lemma~\ref{lem:optimal},
we have $\bm{A}\equiv \bm{A}'$.
\end{enumerate}
\end{proof}

\section{Relation distribution in different datasets}
\label{app:rel:distribution}

Following \cite{rossi2020knowledge},
if more than half of the training triples 
of a relation $r$
have
inverse triples 
(i.e.,
$|\{(t,r,h)\in\mathcal S_{\text{tra}}: (h,r,t)\in \mathcal S_{\text{tra}}\}| >
\frac{1}{2}|\{(h,r,t)\in \mathcal S_{\text{tra}}\}|$),
$r$ is considered as symmetric.
If there exists no inverse triplet
(i.e., $|\{(t,r,h)\in\mathcal S_{\text{tra}}: (h,r,t)\in \mathcal
S_{\text{tra}}\}| =0$),
$r$ is anti-symmetric.
Relations that are neither 
symmetric nor
anti-symmetric
are general asymmetric.
Finally, a relation $r$ belongs to the inverse type
if $\exists r'\in\mathcal R: |\{(t,r',h)\in\mathcal S_{\text{tra}}:(h,r,t)\in\mathcal S_{\text{tra}}\}| >  \frac{1}{2}|\{(h,r,t)\in\mathcal S_{\text{tra}}\}|$.

As can be seen, the four datasets have very 
different distributions and thus properties.
As demonstrated in neural architecture search
\cite{elsken2019neural,liu2018darts,zoph2017neural},
different datasets need different neural architectures.
The architectures discovered have better performance than those
designed by
humans.
Hence,
the scoring functions should also be data-dependent, 
as demonstrated empirically
in Section \ref{exp:kgc:performance}.

\begin{table}[ht]
	\centering
	\caption{Distribution of relation types in the testing set.}
	\label{tb:rel-distribution}
	\setlength\tabcolsep{3.1pt}
	\vspace{-10px}
	\begin{tabular}{c|ccc|c}
		\toprule
		& symmetry & anti-symmetry & general asymmetry & inverse \\  \midrule
		WN18     & 23.4\%    & 72.1\%         & 4.5\%      & 4.5\%   \\
		FB15k    & 9.3\%     & 5.2\%          & 85.5\%     & 74.9\%  \\
		WN18RR   & 37.4\%    & 59.0\%         & 3.6\%      & 0.0\%   \\
		FB15k237 & 3.0\%     & 8.5\%          & 88.5\%     & 10.5\%  \\
		YAGO3-10 & 3.4\%     & 0.7\%          & 95.9\%     & 8.2\%  \\
		\bottomrule
	\end{tabular}
\end{table}

\section{Consistent performance under different dimensions}
\label{sssec:transfer}

In this section, we perform the KG completion experiment in Section~\ref{ssec:KGC}.
First, we
collect the first
100 structures $\bm{A}$'s 
(with $d=64$)
of AutoBLM+ in Section~\ref{ssec:KGC}
and measure the corresponding 
validation MRR
performance 
in step~\ref{step:evo-train} of Algorithm~\ref{alg:evolution}.
We then
increase the embedding dimensionality to $1024$, 
retrain and re-evaluate these structures.
Figure~\ref{fig:transfer}
compares the validation MRRs obtained with
$d=64$ and $d=1024$
on the WN18RR and FB15k-237 data sets.
As can be seen,
the two sets of MRRs
are correlated,
especially for the top performed ones.
The Spearman's rank correlation coefficient
on WN18RR is
 $0.4255$ and on FB15k-237 is $0.7054$.
\begin{figure}[H]
	\centering
	\includegraphics[width=0.48\columnwidth]{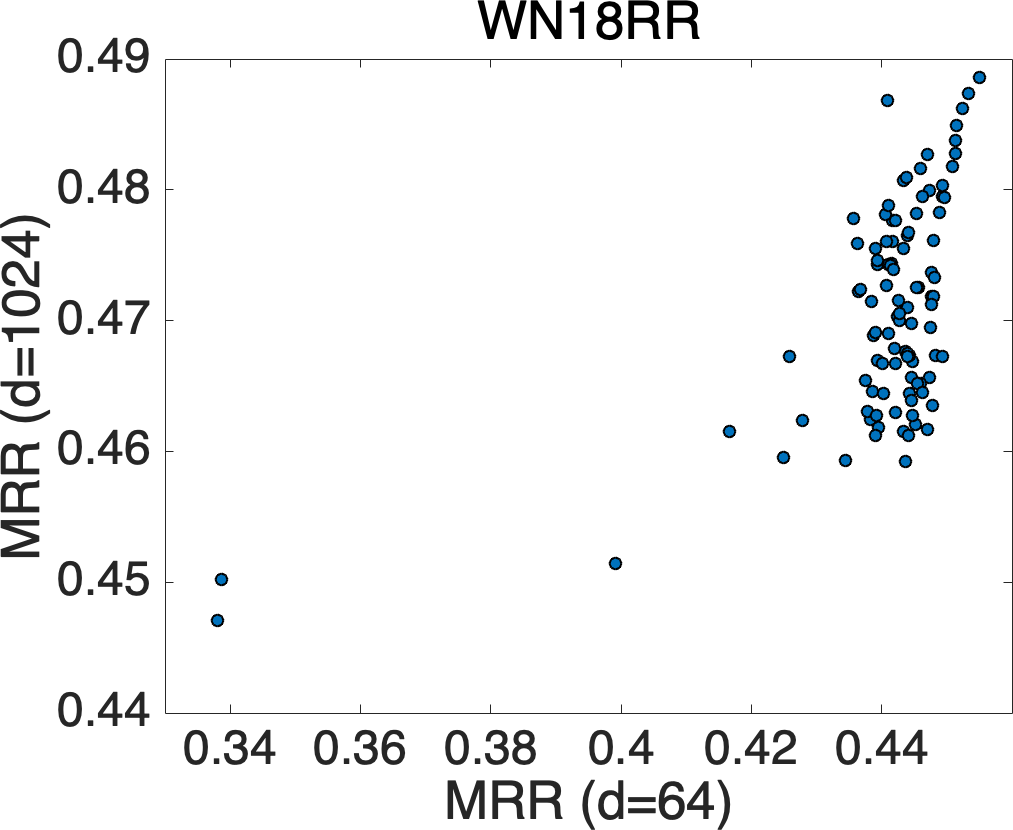}
	\hfill
	\includegraphics[width=0.50\columnwidth]{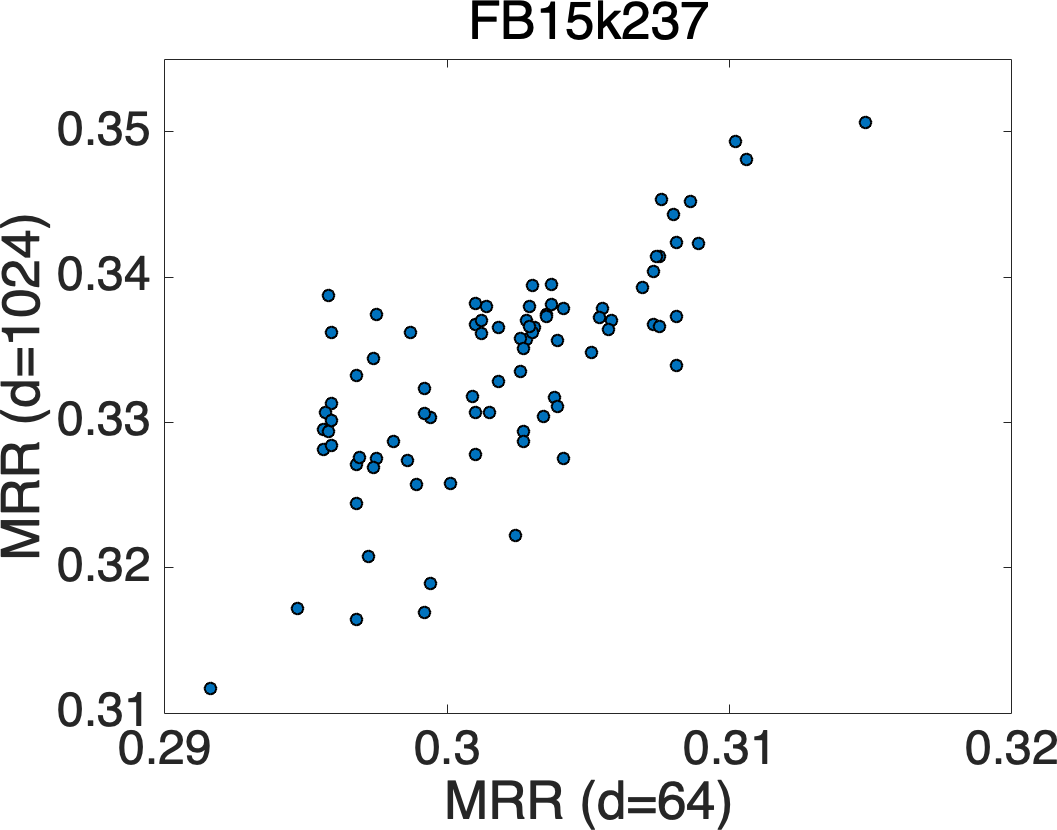}
	
	\vspace{-10px}
\caption{Validation MRRs of the structures with $d=64$ and $d=1024$.}
	\label{fig:transfer}
\end{figure}

\section{Models Obtained by AutoBLM and AutoBLM+}
\label{app:figures}

Figures~\ref{fig:structure:ogb},
\ref{fig:s1} and
\ref{fig:s2} 
show the structures obtained by AutoBLM and AutoBLM+
on the tasks of 
KG completion (Section~\ref{ssec:KGC}) for the  ogbl-biokg and ogbl-wikikg2
datasets, 
multi-hop query (Section~\ref{sec:exp:hop}) and
entity classification (Section~\ref{sec:exp:cls}), respectively.

\end{document}